%% file: main.tex
\documentclass{elsarticle}
\usepackage{natbib}
\usepackage{hyperref}

\makeatletter
\def\ps@pprintTitle{%
 \let\@oddhead\@empty
 \let\@evenhead\@empty
 \def\@oddfoot{}%
 \let\@evenfoot\@oddfoot}
\makeatother

\usepackage{graphicx}
\usepackage[obeyFinal]{todonotes}
\usepackage[defblank]{paralist}
\usepackage{amssymb,amsmath}
\usepackage{wrapfig}
\usepackage[obeyFinal]{todonotes} % use 'final' option to drop the notes
\usepackage{xspace}
\usepackage{amsthm}
\usepackage{subfigure}
\usepackage{multicol}
\usepackage{multirow}
\usepackage{color,colortbl}
\usepackage{booktabs}
\usepackage{enumitem}
\graphicspath{{./imgs/}}

\newtheorem{definition}{Definition}
\newtheorem{corollary}{Corollary}
\newtheorem{theorem}{Theorem}
\newtheorem{lemma}{Lemma}

\newtheorem{remark}{Remark}

\usepackage{etex,etoolbox}
\usepackage{thmtools}
\usepackage{environ}
\newcommand\rtext[1]{{\color{red}{#1}}}

\newcommand{\rawsys}{RAW-SYS\xspace}
\newcommand{\full}{\texttt{full data}\xspace}
\newcommand{\restricted}{\texttt{restricted domain}\xspace}
\newcommand{\nodata}{\texttt{no data}\xspace}

\input{macros}

\newcommand{\onlytechrep}[1]{}

\usetikzlibrary{fit,calc}

\sloppypar

\journal{arXiv}

\begin{document}

\begin{frontmatter}

%\title{Working with partial traces in workflow-nets with data: a planning based comparative approach}
\title{
  %%\btext{Verification of data-aware workflows via reachability: formalization and experiments}
  Solving reachability problems on data-aware workflows \\
  {\small (Currently under submission)}
}

\author[FBK]{Riccardo De Masellis}
\ead{demasellis@gmail.com}
\author[FBK]{Chiara Di Francescomarino\corref{cor2}}
\ead{dfmchiara@fbk.eu}
\author[FBK]{Chiara Ghidini}
\ead{ghidini@fbk.eu}
\author[FUB]{Sergio~Tessaris}
\ead{tessaris@inf.unibz.it}
\cortext[cor2]{Corresponding author}

\address[FBK]{Fondazione Bruno Kessler}
\address[FUB]{Free University of Bozen-Bolzano}

% For research notes, remove the comment character in the line below.
% \researchnote

\begin{abstract}
	Recent advances in the field of Business Process Management have brought about several suites able to model complex data objects along with the traditional control flow perspective. Nonetheless, when it comes to formal verification there is still the lack of effective verification tools on imperative data-aware process models and executions: the data perspective is often abstracted away and verification tools are often missing. 

	In this paper we provide a concrete framework for formal verification of reachability properties on imperative data-aware business processes.
        We start with an expressive, yet empirically tractable class of data-aware process models, an extension of Workflow Nets, and we provide a rigorous mapping between the semantics of such models and that of three important paradigms for reasoning about dynamic systems: Action Languages, Classical Planning, and Model-Checking.
%%
        % The mapping is based on a common interpretation of all such frameworks in terms of transition systems.
%%
        Then we perform a comprehensive assessment of the performance of three popular tools supporting the above paradigms in solving reachability problems for imperative data-aware business processes,
        which paves the way for a theoretically well founded and practically viable exploitation of formal verification techniques on data-aware business processes.
    \end{abstract}
% \tableofcontents
%\listoflabels
\end{frontmatter}

\input{introduction}

\input{running-example}

\input{framework}

\input{dawnet-fts}
\input{encoding-BC}

\input{encoding-PDDL}

\input{encoding-nuSMV}
\input{repair}
\input{encoding-traces}
\input{evaluation}
\input{relatedworks}
\input{conclusions}

% \newpage
\appendix
\input{additional-encoding-bc}

\input{additional-encoding-pddl}
\input{additional-encoding-traces}

\bibliographystyle{plainnat}
\bibliography{biblio}

\end{document}

%% file: macros.tex
\usepackage{xspace}
\usepackage{ifthen}

\usepackage{listings}
\newcommand{\lstmath}[1]{\text{\lstinline{#1}}}

\newcommand{\pres}[1]{\ensuremath{{}^{\bullet}#1}\xspace}
\newcommand{\posts}[1]{\ensuremath{#1^{\bullet}}\xspace}

\newcommand{\tuple}[1]{\ensuremath{\langle{#1}\rangle}}

\newcommand{\fire}[1]{\ensuremath{\stackrel{{#1}}{\rightarrow}}}

\newcommand{\tstep}[2]{\ensuremath{{#1}^{#2}}\xspace}

\newcommand{\dwnetbc}{\ensuremath{\Phi}}
\newcommand{\bcdwnet}{\ensuremath{\dwnetbc^-}}

\newcommand{\ournet}{DAW-net\xspace}
\newcommand{\ournets}{DAW-nets\xspace}

\newcommand{\pidRel}[2][]{\ensuremath{#2^{\wp\ifthenelse{\equal{#1}{}}{}{(#1)}}}\xspace}

\newcommand{\tup}[1]{\langle#1\rangle}            % tuple

\newcommand{\A}{\ensuremath{\mathcal{A}}\xspace} 
 \newcommand{\D}{\mathcal{D}}
 \newcommand{\F}{\ensuremath{\mathcal{F}}\xspace}
\newcommand{\G}{\mathcal{G}} 
 
\newcommand{\K}{\mathcal{K}}

 \newcommand{\V}{\mathcal{V}}

\newcommand{\set}[1]{\{ #1 \}}

\newcommand{\dlvk}{\texttt{DLV}$^\K$\xspace}

\newcommand{\writef}{\ensuremath{\textsf{wr}}}

\newcommand{\guardf}{\ensuremath{\textsf{gd}}}
\newcommand{\varset}{\ensuremath{\mathcal{V}}}
\newcommand{\domf}{\ensuremath{\textsf{dm}}}
\newcommand{\guardlang}[1][\mathcal{D}]{\ensuremath{\mathcal{L}(#1)}}

\newcommand{\ordf}{\ensuremath{\textsf{ord}}}
\newcommand{\img}{\ensuremath{img}}

\newcommand{\assign}{\eta}
\newcommand{\deff}{\textsf{def}}
\newcommand{\dmodel}{\ensuremath{\mathcal{D}}}
\newcommand{\nmodel}{\ensuremath{\mathcal{N}}}

\newcommand{\mrk}{M}
\newcommand{\adom}{\ensuremath{\textsf{adm}}}
\newcommand{\finite}[1]{\ensuremath{\overline{#1}}\xspace}

\newcommand{\wkfProj}{\ensuremath{\Pi}\xspace}
\newcommand{\dwnmodel}{\ensuremath{W}\xspace}

\newcommand{\kvar}[1]{\ensuremath{\text{var}_{#1}}\xspace}
\newcommand{\kvardom}[1]{\ensuremath{\text{dom}_{#1}}\xspace}
\newcommand{\kguard}[1]{\ensuremath{\text{grd}_{#1}}\xspace}
\newcommand{\kvardef}[1]{\ensuremath{\text{def}_{#1}}\xspace}
\newcommand{\kvarchange}[1]{\ensuremath{\text{chng}_{#1}}\xspace}

\newcommand{\cvalue}[1]{\texttt{#1}}
\newcommand{\Nat}{\mathbb{N}}

\newcommand{\clingo}{\textsc{clingo}\xspace}

\newcommand{\fastdw}{\textsc{fast-downward}\xspace}
\newcommand{\nuxmv}{\textsc{nuXmv}\xspace}

\newcommand{\lra}{\leftrightarrow}
\newcommand{\ra}{\rightarrow}

\newcommand{\fine}{\textsf{ended}}
\newcommand{\last}{\textsf{last}}
\newcommand{\ol}{\overline}

\newcommand{\enc}{\textsf{enc}}

\newcommand{\nextpre}[1]{\nxt{\textsf{pre}(#1)}}
\newcommand{\currpre}[1]{\textsf{pre}(#1)}

\newcommand{\smvts}{\textsc{smv}}
\newcommand{\smv}{\textsc{smv}\xspace}

\newcommand{\stvar}{\ensuremath{\text{stVar}}}

\newcommand{\udef}{\ensuremath{\text{undef}}}
\newcommand{\true}{\ensuremath{\textsc{true}}}
\newcommand{\false}{\ensuremath{\textsc{false}}}
\newcommand{\nxt}[1]{\ensuremath{\text{next($#1$)}}}

\newcommand{\clts}{\bclng}
\newcommand{\bclng}{\ensuremath{\textsc{bc}}\xspace}
\newcommand{\bctrans}{\ensuremath{\G_{\bclng}}\xspace}
\newcommand{\bcassign}[1]{\ensuremath{\nu_{\bclng}^{#1}}\xspace}

\newcommand{\naf}{\ensuremath{{\sim}}\xspace}
\newcommand{\nullv}{\text{\lstinline{null}}\xspace}
\newcommand{\flit}[2]{\ensuremath{{#1}{:}{#2}}\xspace}
\newcommand{\feqlit}[3]{\flit{#1}{{#2}{=}{#3}}\xspace}
\newcommand{\bcqlang}[1]{\ensuremath{{#1}^{\bclng}}\xspace}
\newcommand{\bdescr}{\ensuremath{\textsc{b}}\xspace}

\newcommand{\fdts}{\pddllng}

\newcommand{\pddllng}{\ensuremath{\textsc{pddl}}\xspace}
\DeclareMathOperator{\pddleff}{eff}
\DeclareMathOperator{\pddlpre}{pre}
\DeclareMathOperator{\pddlhead}{head}
\newcommand{\pddlqlang}[1]{\ensuremath{{#1}^{\pddllng}}\xspace}

\newcommand{\dwnetpddl}{\ensuremath{\Psi}}
\newcommand{\pddldwnet}{\ensuremath{\dwnetpddl^-}}

\newcommand{\smvlng}{\ensuremath{\textsc{smv}}\xspace}
\newcommand{\TS}{\ensuremath{\mathit{TS}}}
\newcommand{\RG}{\ensuremath{\mathit{RG}}}

%% file: introduction.tex
%!TEX root = ./main.tex

\section{Introduction}

% Current business analysis monitoring instruments

Recent advances in the field of Business Process Management have brought about several suites able to model complex data objects along with the traditional control flow perspective. Nonetheless, when it comes to formal verification, there is still a lack of effective tools on imperative data-aware process models and executions. Indeed, the data perspective is often either abstracted away due to the intrinsic difficulty of handling unbounded data, or investigated only on the theoretical side, providing decidability results for very expressive scenarios without actual verification tools (see \cite{De-Masellis:2017:aaai} for an in depth analysis). 
Automated Planning is one of the core areas of AI where theoretical investigations and concrete and robust tools have made possible the reasoning about dynamic systems and domains. 
In the last few years, links between Automated Planning and Business Process Management have started to emerge, either to exploit Automated Planning techniques to address specific problems of Business Process Management (BPM), such as the one of process alignment~\cite{DBLP:conf/aips/LeoniLM18,DBLP:conf/aaai/XuLZ17a,DBLP:conf/aips/GiacomoMMS16}, trace completion~\cite{DBLP:conf/bpm/MasellisFGT17} and process verification~\cite{De-Masellis:2017:aaai}, or to argue for a wider relationship between the two fields~\cite{DBLP:journals/jodsn/Marrella19}. 

Despite this growing interest, a systematic investigation of the actual usefulness of Automated Planing techniques, and of its variant components, is still lacking in the BPM context. In fact, all the work above relies on specific and ad hoc encodings of a given %specific
task at hand, therefore leveraging specific Planning formalisms, mostly referring either to classical or to action based planning. 
As a consequence, which variants among classical and non-classical planning models~\cite{Geffner:2013aa} may be better suited for which complex real-world BPM challenge still remains unstudied, as also noted in the final discussion of~\cite{DBLP:journals/jodsn/Marrella19}. This is not negligible, considering the articulated and vast area of different variants and approaches that constitute Automated Planning. 

In this paper we aim at 
%contributing
 filling this gap by providing a first comprehensive evaluation of different Automated Planning formalisms on a complex and still open challenge in Business Process Management: the provision of a theoretically sound and practically supported data-aware business process verification technique. 
We fulfill this objective by focusing on reachability problems for imperative data-aware business processes.
In particular we \begin{inparaenum}[\it(i)] 
\item exploit an expressive, yet empirically tractable class of data-aware process models, built by extending the well established Workflow Nets formalism~\cite{aalst_soundness:2010};
\item we establish a rigorous mapping between our data-aware process models and three important paradigms for reasoning about dynamic systems, namely Action Languages, Classical Planning, and Model Checking\footnote{While one may argue that Model Checking does not typically figure among Automated Planning techniques, its usage as an approach to planning is well known and supported by decision procedures and tools. See e.g., \cite{bddwork1}}, based on a common interpretation of the three dynamic systems in terms of transition systems which allows us to
\item perform a first comprehensive assessment of the performance of three popular tools supporting the above paradigms in computing reachability for imperative data-aware business processes\end{inparaenum}. The reasoning formalisms considered in this work exhibit different characteristics, whose impact in terms of reasoning with data-aware workflow net languages is not easily predictable without a robust evaluation. Therefore our work provides not only a solid contribution to a theoretically well founded and practically viable exploitation of Automated Planning to support formal verification on data-aware business processes, but it also paves the way to more general and rigorous investigations on the use of different planning formalisms for BPM. 

The paper is structured as follows. Section~\ref{sec:motivations} highlights the general motivations behind our work and introduces a running example; Section~\ref{sec:ourframework} provides some background on Workflow Nets~\cite{aalst_soundness:2010} and  describes the language of \ournets originally introduced in~\cite{DBLP:conf/bpm/MasellisFGT17}; Section~\ref{sec:dawnet-fts} describes the approach and the main steps for exploiting the notion of transition system as a bridge between \ournet and the chosen Action Languages, Classical Planning, and Model-Checking languages, while details of the specific encodings of the \ournets reachability problem in the three specific formalisms are provided in  Sections~\ref{sec:bc:encoding}, \ref{sect:pddl:enc}, and \ref{sect:nuxmv:enc}. Section~\ref{sec:trace_repair} introduces the problem of trace repair we use for the the empirical evaluation together with its recast in terms of reachability and lastly Section~\ref{sec:evaluation} presents an extensive evaluation of three well known solvers (i.e., \clingo, \fastdw and \nuxmv) available for the different formalisms both on synthetic and real life logs. While the evaluation does not indicate a clear ``winner'', it empirically shows that adding the data dimension dramatically affects the solvers' ability to return a result as well as their performance. Moreover, characteristics of the trace, such as its dimension, also have an important impact on which solver performs best. The code used and all datasets are publicly available (see Section~\ref{sec:evaluation} for details). The paper ends with related work and concluding remarks.

% Part of this work \todo{Shall we move this to the related work session?} has been published in~\cite{DBLP:conf/bpm/MasellisFGT17}. Compared to the early version, where the focus was on an ad hoc encoding of the problem of trace repair using a specific \emph{Action Language}, this paper introduces a more general encoding technique based on finite transition systems and an extensive evaluation of different solvers. The new technique highlights the core dynamic properties of Workflow Nets, and enables a seamless encoding of the decision problems under investigation in a variety of formal frameworks.

%% file: running-example.tex
%!TEX root = ./main.tex

\section{Verification of Data-Aware Business Processes}
\label{sec:motivations}

Commercial and non-commercial BPM suites such as Bonita, Bizagi, YAWL and Camunda support nowadays the modelling of both control and data flow and provide some form of verification support. Following the analysis in~\cite{De-Masellis:2017:aaai}, they nonetheless offer no or very limited formal verification support when it comes to detect the interdependencies between data and control flow: for instance, they fail to report the critical issue in the process in Figure~\ref{fig:imgs_sampleData}. Although this process never terminates, due to the writing of variable $y$ by activity \textbf{T2}, when verified by the tools above this is not revealed, and the process is labeled as potentially able to terminate. Indeed YAWL \cite{VanDerAalst2005} offers verification features limited to the control flow  and thus it wrongly reports that such a process can always reach the termination state. All other tools instead only offer a simulation environment (i.e., no formal verification) that checks whether the process passes through all the sequence flows, without taking into account data.

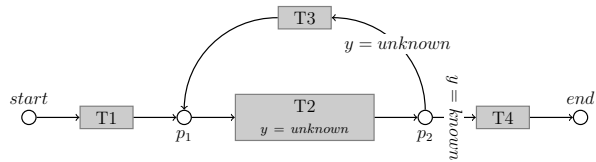
\begin{figure*}[t]
  \centering
    \tikzstyle{place}=[circle,draw=black,thick]
    \tikzstyle{transition}=[text width=1cm,text centered,rectangle,draw=black!50,fill=black!20,thick]
    \scalebox{0.57}{\begin{tikzpicture}[font=\large]
      \node (start) [place,label=above:$start$] {};
      \node[transition] (t1)  [right=1cm of start, anchor=west] {T1};
      \node[place,label=below:$p_1$] (p1) [right=1cm of t1, anchor=west] {};
      \node[transition,text width=3cm] (t2)  [right=1cm of p1, anchor=west] {T2\\ {\small $y=unknown$}};
      \node[transition] (t3)  [above=1.5cm of t2, anchor=south] {T3};
      \node[place,label=below:$p_2$] (p2) [right=1cm of t2] {};
      \node[transition] (t4)  [right=1cm of p2, anchor=west] {T4};
      \node[place,label=above:$end$] (end)  [right=1cm of t4] {};
      \draw [->,thick] (start.east) -- (t1.west);
      \draw [->,thick] (t1.east) -- (p1.west);
      \draw [->,thick] (p1.east) -- (t2.west);
      \draw [->,thick] (t2.east) to (p2.west);
	  \draw [->,thick] (p2.east) to  node[fill=white,inner sep=1pt, rotate=-90] {{$y=known$}} (t4.west);
	  \draw [->,thick] (p2.north) to [bend right=45] node[fill=white,inner sep=1pt, midway] {{$y=unknown$}} (t3.east);
	  \draw [->,thick] (t3.west) to [bend right=45] (p1.north);
	  \draw [->,thick] (t4.east) to  (end.west);
    \end{tikzpicture}}
  \caption{A Simple Business Process Model.}
  \label{fig:imgs_sampleData}
\end{figure*}

If we move from actual tools to theoretical frameworks the situation improves. Indeed we may encode the example above in, e.g., Colored Petri Nets (CPN) as it offers verification support that takes into account conditions on labels and some interaction between activities and data. Nonetheless, CPN suffers of usability issues as by encoding data directly within the Net, it is difficult to couple it with tools that are based on common data models such as the relational model or attribute-value pairs.
In order to exploit the verification features of CPN the user would have to manually encode the data flow within the Net and if this may not prove convoluted for examples such as the one in Figure~\ref{fig:imgs_sampleData}, the complexity can escalate with more elaborated examples such as the one illustrated in Figure~\ref{fig:imgs_SampleWFnet}, which we use as running example throughout the paper.  

\begin{figure*}[h]
  \centering
    \tikzstyle{place}=[circle,draw=black,thick]
    \tikzstyle{transition}=[text width=2cm,text centered,rectangle,draw=black!50,fill=black!20,thick]
    \scalebox{0.5}{\begin{tikzpicture}[font=\large]
      \node (start) [place,label=above:$start$] {};
      % \node at ( 0,1) [place] {};
%       \node at ( 0,0) [place] {};
      \node[transition,label=above:$T1$, text width=2cm] (t1)  [right=.5cm of start, anchor=west] {ask application documents};
      \node[place,label=right:$p_1$] (p1) [right=1cm of t1, anchor=west] {};
      \node[transition,label=above:$T2$] (t2)  [above right=1.5cm of p1, anchor=west] {send student application};
      \node[transition,label=above:$T3$] (t3)  [below right=1.5cm of p1, anchor=west] {send worker application};
      \node[place,label=above:$p_2$] (p2) [right=.5cm of t2] {};
      \node[place,label=below:$p_3$] (p3) [right=.5cm of t3] {};
      \node[transition,text width=1.5cm,label=above:$T4$] (t4)  [right=.5cm of p2, anchor=west] {fill student request};
      \node[transition,text width=1.5cm,label=above:$T5$] (t5)  [right=.5cm of p3, anchor=west] {fill worker request};
      \node[place,label=left:$p_4$] (p4) [right=6.5cm of p1, anchor=west] {};
      \node[transition,label=above:$T7$] (t7)  [right=1.5cm of p4] {senior officer approval};
	  \node[transition,label=above:$T6$] (t6)  [above=1cm of t7] {local officer approval};
      \node[transition,label=above:$T8$] (t8)  [below=1cm of t7] {bank committee approval};
      \node[place,label=above right:$p_5$] (p5) [right=.5cm of t7, anchor=west] {};
      \node[transition,label=above:$T9$,text width=.5cm] (t9)  [right=.5cm of p5] {\quad};
      \node[place,label=above:$p_6$] (p6) [above right=.8cm of t9] {};
      \node[place,label=below:$p_7$] (p7) [below right=.8cm of t9] {};
      \node[transition,label=above:$T10$,text width=2.1cm] (t10)  [right=.5cm of p6] {send approval to customer};
      \node[transition,label=below:$T11$,text width=2.1cm] (t11)  [right=.5cm of p7] {store approval in branch};
      \node[place,label=above:$p_8$] (p8)  [right=.5cm of t10] {};
      \node[place,label=below:$p_9$] (p9)  [right=.5cm of t11] {};
      \node[transition,text width=1cm,label=above:$T12$] (t12)  [right=5cm of t9] {issue loan};
      \node[place,label=above:$end$] (end)  [right=.5cm of t12] {};
      \draw [->,thick] (start.east) -- (t1.west);
      \draw [->,thick] (t1.east) -- (p1.west);
      \draw [->,thick] (p1.north) to [bend left=45] node[fill=white,inner sep=1pt, near start] {{$loanType=s$}} (t2.west);
      \draw [->,thick] (p1.south) to [bend right=45] node[fill=white,inner sep=1pt, near start, swap] {{$loanType=w$}} (t3.west);
      \draw [->,thick] (t2.east) to (p2.west);
      \draw [->,thick] (t3.east) to (p3.west);
      \draw [->,thick] (p2.east) to (t4.west);
      \draw [->,thick] (p3.east) to (t5.west);
      \draw [->,thick] (t4.east) to [bend left=20] (p4.north);
      \draw [->,thick] (t5.east) to [bend right=20] (p4.south);
      \draw [->,thick] (p4.north) to [bend left=30] node[fill=white,inner sep=1pt, midway,text width=1cm] {{$request\leq 5k$}} (t6.west);
      \draw [->,thick] (p4.south) to [bend right=30] node[fill=white,inner sep=1pt,midway,swap,text width=1cm] {{$request\geq 100k$}} (t8.west);
      \draw [->,thick] (p4.east) to node[fill=white,inner sep=1pt, midway,swap] {else} (t7.west);
      \draw [->,thick] (t6.east) to [bend left=20] (p5.north);
      \draw [->,thick] (t8.east) to [bend right=20] (p5.south);
      \draw [->,thick] (t7.east) to (p5.west);
      \draw [->,thick] (p5.east) to (t9.west);
      \draw [->,thick] (t9.east) to (p6.west);
      \draw [->,thick] (t9.east) to (p7.west);
      \draw [->,thick] (p6.east) to (t10.west);
      \draw [->,thick] (p7.east) to (t11.west);
      \draw [->,thick] (t10.east) to (p8.west);
      \draw [->,thick] (t11.east) to (p9.west);
      \draw [->,thick] (p8.east) to  (t12.west);
      \draw [->,thick] (p9.east) to  (t12.west);
      \draw [->,thick] (t12.east) to  (end.west);
    \end{tikzpicture}}
  \caption{The LoanRequest Process Model.}
  \label{fig:imgs_SampleWFnet}
\end{figure*}
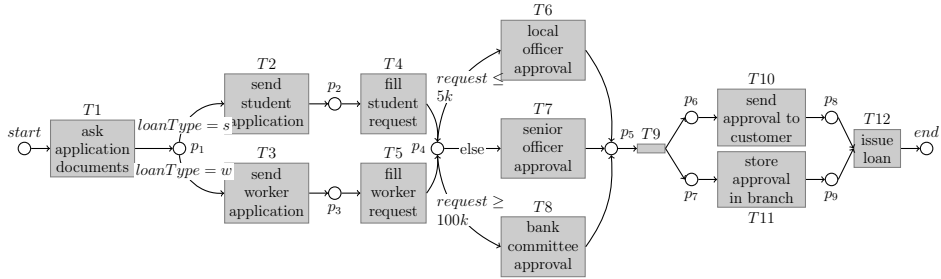
 
This (data-aware) workflow, modelling a simple Loan Application process, is composed of 12 tasks (\emph{T1-T12}). The application starts with task \emph{T1} where the customer asks for the application forms. Then, depending on whether the customer is asking for a student loan ($loanType=s$) or a worker loan ($loanType=w$) the process continues with a mutually exclusive choice between the sequence of activities \emph{T2-T4} and \emph{T3-T5}. Once the appropriate (student or worker) request is filled, another mutually exclusive choice is presented for the execution: if the amount requested in lower than $5k$, then task \emph{T6} is executed and only a local approval is needed; if the amount requested is greater than $100k$, then the request needs to be approved by the bank committee (task \emph{T8}); in all the other cases the process continues with task \emph{T7}. Once the request is approved by means of the appropriate approval task, two parallel activities are executed: the approval is sent to the customer (task \emph{T10}), and the approval documents are stored in the branch (task \emph{T11}). Finally the loan is issued (task \emph{T12}). 
Concerning the interaction between activities and data, task \emph{T1} assigns variable $loanType$ to $s$ or $w$; furthermore task \emph{T4} is empowered with the ability to write the variable $request$ with a value smaller or equal than $30k$ (being this the maximum amount of a student loan). Similarly, task \emph{T5} is in charge of writing the variable $request$ up to $500k$, which is the maximum amount for a worker.

Given the example above one may check different properties which can be reduced to verifying a \emph{reachability} problem \cite{Aalst97}: whether the process admits at least one execution (that is, it can always reach the end from start), or whether termination is guaranteed from start if we also impose the execution of certain tasks. In our example it is easy to see that the interplay of the control and data flow would prevent any execution from start to end that includes both \emph{T2} and \emph{T8} for instance, as \emph{T2} can write the variable $request$ with a value smaller or equal than $30k$ while \emph{T8} is executed only if $request$ has a value greater that $100k$. While these properties may be verifiable by exploiting the power of CPN or of extensions of Workflow-Nets (WF-Nets) \cite{sidorovastahletal:2011}, the ad hoc encoding of the relational data model that is nowadays widely used to express data objects (in BPM suites and beyond) in the these formalisms is among the reasons for their failure to have a significant impact on practical tools. 

In the remaining of the paper we provide a first theoretically sound and empirically extensive investigation on the exploitation of automated planning techniques to compute reachability in data-aware business processes. The reason we focus on reachability is that,
as extensively illustrated in~\cite{Aalst97}, a number of crucial properties that define a \emph{sound} Business Process can be formulated as reachability properties. These include, e.g., whether, no matter how the process evolves from its beginning, it is always possible to reach its end, or if there are no dead tasks, namely all of them could eventually fire.
Consequently, planning is a natural paradigm for solving reachability problems. Besides, it comes with solid tools that can be directly exploited to support practical applications and lastly it treats data objects as first class citizens, thus enabling a natural encoding of data-aware business processes based on common data models such as the relational data model or attribute-value pairs.

%Consequently, planning is analogously simple but many fold: (i) it is a formalism that easily fits to the computing of reachability; (ii) it comes with solid tools that can be directly exploited to support practical applications; and (iii) it treats data objects as first class citizens: thus it enables a natural encoding of data-aware business processes based on common data models such as the relational data model or attribute-value pairs.

%The reason for adding \emph{T9} will become clearer in Section XX \todo{sergio: I don't get this} and we can simply disregard this task here.

% Process tasks are modeled in PNs as transitions while arcs and places constraint their ordering.
% %%
% For instance, the process in Figure~\ref{fig:imgs_SampleWFnet}\footnote{For the sake of simplicity we only focus here on the, so-called, happy path, that is the successful granting of the loan.}
% % Consider, for instance, the process of a \emph{Loan Request} depicted in Figure~\ref{fig:imgs_SampleWFnet}\footnote{For the sake of simplicity we only focus here on the, so-called, happy path, that is the successful granting of the loan}. The 10 tasks that compose the process are modeled by transitions. This process
% exemplifies how PNs can be used to model parallel and mutually exclusive choices, typical of business processes: sequences \emph{T2;T4}-\emph{T3;T5} and transitions \emph{T6-T7-T8} are indeed placed on mutually exclusive paths. Transitions \emph{T10} and \emph{T11} are instead placed on parallel paths. Finally, \emph{T9} is needed to prevent connections between nodes of the same type.

%% file: framework.tex
\section{The Framework: \ournet}
\label{sec:ourframework}
In this section we first provide the necessary background on Petri Nets and WF-nets (Section~\ref{sec:WF-net}) and then suitably extend WF-nets to represent data and their evolution as transitions are performed (Section~\ref{sec:ournet}).

\subsection{The Workflow Nets modeling language} % (fold)
\label{sec:WF-net}

Petri Nets~\cite{Petri:1962aa} (PN) is a modeling language for the description of distributed systems that has widely been applied to the description and analysis of business processes~\cite{vanderaalst:1998}. 
The classical PN is a directed bipartite graph with two node types, called \emph{places} and \emph{transitions}, connected via directed arcs. Connections between two nodes of the same type are not allowed. 

\begin{definition}[Petri Net]
	\label{def:PN}
  A \emph{Petri Net} is a triple $\tuple{P,T,F}$ where $P$ is a set of \emph{places}; $T$ is a set of \emph{transitions}; $F\subseteq (P \times T) \cup (T \times P)$ is the flow relation describing the arcs between places and transitions (and between transitions and places).
\end{definition}
The \emph{preset} of a transition $t$ is the set of its input places: $\pres{t} = \{p \in P \mid (p,t) \in F\}$. The \emph{postset} of $t$ is the set of its output places: $\posts{t} = \{p \in P \mid (t,p) \in F\}$. Definitions of pre- and postsets of places are analogous.
Places in a PN may contain a discrete number of marks called tokens. Any distribution of tokens over the places, formally represented by a total mapping  $M : P\mapsto \mathbb{N}$, represents a configuration of the net called a \emph{marking}.  

\begin{figure}
  \centering
    \fbox{\includegraphics[width=.25\linewidth]{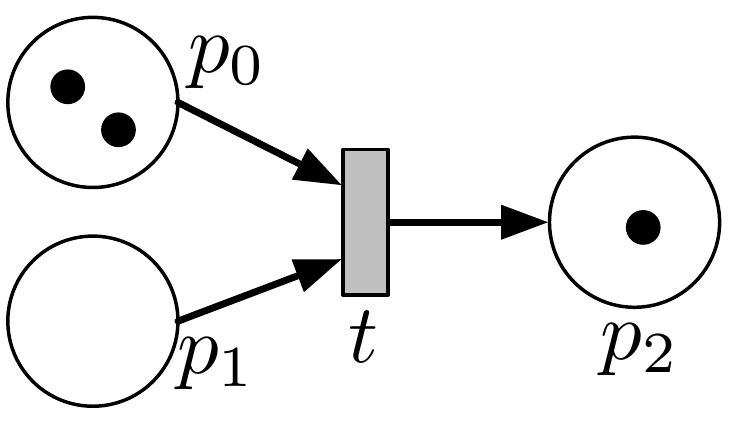}}
  \caption{A Petri Net.}
  \label{fig:imgs_PNsample}
\end{figure}
PNs come with a graphical notation where places are represented by means of circles, transitions by means of rectangles and tokens by means of full dots within places. Figure~\ref{fig:imgs_PNsample} depicts a PN with a marking $M(p_0)=2$, $M(p_1)=0$, $M(p_2)=1$. The preset and postset of $t$ are $\{p_0, p_1\}$ and $\{p_2\}$, respectively.  

Process tasks are modeled in PNs as transitions while arcs and places constrain their ordering.
For instance, the process in Figure~\ref{fig:imgs_SampleWFnet}
% Consider, for instance, the process of a \emph{Loan Request} depicted in Figure~\ref{fig:imgs_SampleWFnet}\footnote{For the sake of simplicity we only focus here on the, so-called, happy path, that is the successful granting of the loan}. The 10 tasks that compose the process are modeled by transitions. This process
exemplifies how PNs can be used to model parallel and mutually exclusive choices, typical of business processes: sequences \emph{T2;T4} and \emph{T3;T5} are placed on mutually exclusive paths; the same holds for transitions \emph{T6}, \emph{T7}, and \emph{T8}. Transitions \emph{T10} and \emph{T11} are instead placed on parallel paths. Finally, \emph{T9} is needed to prevent connections between nodes of the same type.

The expressivity of PNs exceeds, in the general case, what is needed to model business processes, which typically have a well-defined starting point and a well-defined ending point. This imposes syntactic restrictions on PNs, that result in the following definition of a workflow net (WF-net)~\cite{vanderaalst:1998}.
%DECISO CON SERGIO DI LASCIARE PERDERE FREE-CHOICE. COMPLICATO DA CAPIRE E NON CI SERVE.
% 
\begin{definition}[WF-net]
A PN $\tuple{P,T,F}$ is a WF-net if it has a single source place start, a single sink place $end$, and every place and every transition is on a path from $start$ to $end$, i.e., for all $n\in P\cup T$, $(start,n)\in F^*$ and $(n,end)\in F^*$, where $F^*$ is the reflexive transitive closure of $F$.
\end{definition}

A marking in a WF-net represents the \emph{workflow state} of a process execution.
We distinguish two special markings: the \emph{initial marking} $M_s$, which has one token in the $start$ place ($M_s(start)=1$) and all others places are empty ($\forall p \in P\setminus\set{start}.M_s(p)=0$), and the \emph{final marking} $M_e$, which has one token in the $end$ place ($M_e(end)=1$) and all others places are empty ($\forall p \in P\setminus\set{end}.M_e(p)=0$).

The semantics of a PN/WF-net, and in particular the notion of \emph{valid firing}, defines how transitions route tokens through the net so that they correspond to a process execution. 

\begin{definition}[Valid Firing]
	\label{def:Firing}
A firing of a transition $t\in T$ from $M$ to $M'$ is \emph{valid}, in symbols $M \fire{t} M'$, iff
\begin{enumerate}
  \item $t$ is enabled in $M$, i.e., $\{ p\in P\mid M(p)>0\}\supseteq \pres{t}$; and
  \item the marking $M'$ is such that for every $p\in P$:
  \begin{displaymath}
	  \small
    M'(p) =
    \begin{cases}
      M(p)-1 & \text{if $p\in \pres{t}\setminus\posts{t}$}\\
      M(p)+1  & \text{if $p\in \posts{t}\setminus\pres{t}$}\\
      M(p) & \text{otherwise}
    \end{cases}
  \end{displaymath}
\end{enumerate}
\end{definition}
Condition 1.~states that a transition is enabled if all its input places contain at least one token; condition 2.~states that when $t$ fires it consumes one token from each of its input places and produces one token in each of its output places. 

%\todordminline{Attenzione! Ora un case termina sempre nel marking finale!!! Se ci serve riferirci ad un case parziale, allora parlare di sequence of valid firing.}

Notationally, $t_1t_2t_3 . . . t_{k}$ is a sequence of valid firings that leads from $M_0$ to $M_k$ iff $M_0 \fire{t_1} M_1 \fire{t_2} M_2 \fire{t_3} \ldots \fire{k-1} M_{k-1} \fire{t_k} M_k$. In this paper we use the term \emph{case} of a WF-Net to denote a sequence of valid firings 
$M_0 \fire{t_1} M_1 \fire{t_2} M_2 \fire{t_3} \ldots \fire{k-1} M_{k-1} \fire{t_k} M_k$ 
such that $M_0=M_s$ is the initial marking and $M_k=M_e$ is the final marking. A case is thus a sequence of valid firings that connect the initial to the final marking through the PN.

From now on we concentrate on \emph{safe} nets, which generalise the class of \emph{structured workflows} and are the basis for best practices in process modeling~\cite{kiepuszewskihofstedeetal:2013}. The notion of safeness is defined in terms of $k$-boundedness (see also~\cite{aalst_soundness:2010}).

\begin{definition}[$k$-boundedness and safeness]
A marking of a PN is $k$-bound if the number of tokens in all places is at most $k$. A PN is $k$-bound if the initial marking is $k$-bound and the marking of all cases is $k$-bound. If $k=1$, the PN is safe.  
\end{definition}
It is important to notice that our approach can be seamlessly generalized to other classes of PNs, as long as they are $k$-bound. 

\paragraph{Reachability on Petri Nets}
  Given a PN and a set of ``goal'' markings $G$ for its places, the reachability problem amounts to check whether there is a sequence of valid firings from the initial marking $M_s$ of the PN to any of the markings in $G$.

\subsection{The \ournet modeling language} % (fold)
\label{sec:ournet}

\ournet \cite{DBLP:conf/bpm/MasellisFGT17} is a data-aware extension of WF-nets which enriches them with the capability of reasoning on data. \ournet extends WF-nets by providing: (i) a model for representing data; (ii) a way to make decisions on actual data values; and (iii) a mechanism to express modifications to data.
It does it by introducing: 
\begin{itemize}
\item a set of variables taking values from possibly different domains (addressing (i));
\item queries on such variables used as transitions preconditions (addressing (ii))
\item variables updates and deletion in the specification of net transitions (addressing (iii)).
\end{itemize}
\ournet follows the approach of state-of-the-art WF-nets with data~\cite{sidorovastahletal:2011,de_leoni:2013}, from which it borrows the above concepts, extending them by allowing reasoning on actual data values as better explained in Section~\ref{sec:related_works}. 

Throughout the section we use the WF-net in Figure~\ref{fig:imgs_SampleWFnet}  extended with data as a running example.

\subsubsection{Data Model}
\label{sec:datamodel}
The data model of WF-nets take inspiration from the data model of the IEEE XES standard for describing event logs, which represents data as a set of variables. Variables take values from specific sets on which a partial order can be defined. As customary, we distinguish between the data model, namely the intensional level, from a specific instance of data, i.e., the extensional level.

\begin{definition}[Data model]\label{def:dmodel}
A \emph{data model} is a tuple $\D = (\V, \Delta, \domf, \ordf)$ where:
\begin{compactitem}
\item $\V$ is a possibly infinite set of variables;
\item $\Delta = \set{\Delta_1, \Delta_2, \ldots}$ is a possibly infinite set of domains (not necessarily disjoint);
\item $\domf: \V \rightarrow \Delta$ is a total and surjective function which associates to each variable $v$ its domain $\Delta_i$;
\item $\ordf$ is a partial function that, given a domain $\Delta_i$, if $\ordf(\Delta_i)$ is defined, then it returns a \emph{partial order} (reflexive, antisymmetric and transitive) $\leq_{\Delta_i} \subseteq \Delta_i \times \Delta_i$.\footnote{To simplify the model we assume that if $\{(o,o'), (o', o)\}\subseteq \ordf(\Delta_i)$ then $o = o'$.}
\end{compactitem}
\end{definition}

A data model for the loan example is $\V = \{loanType$, $request$, $loan\}$, $\domf(loanType)=\set{\cvalue{w}, \cvalue{s}}$, $\domf(request) = \Nat$, $\domf(loan) = \Nat$, with $\domf(loan)$ and $\domf(loanType)$ being total ordered by the natural ordering $\leq$ in $\Nat$.

An actual instance of a data model is simply a partial function associating values to variables.
\begin{definition}[Assignment]
Let $\D = \tuple{\V, \Delta, \domf, \ordf}$ be a data model. An \emph{assignment} for variables in $\V$ is a \emph{partial} function $\assign: \V \rightarrow \bigcup_i\Delta_i$ such that for each $v \in \V$, if $\assign(v)$ is defined, i.e., $v \in \img(\eta)$ where $\img$ is the image of $\eta$, then we have $\assign(v) \in \domf(v)$.
\end{definition}

Instances of the data model are accessed through a boolean query language, which notably allows for equality and comparison. As will become clearer in Section~\ref{sec:our-net}, queries are used as \emph{guards}, i.e., preconditions for the execution of transitions.

\begin{definition}[Query language - syntax]\label{def:dawnet:qlang}
Given a data model, the language $\guardlang$ is the set of formulas $\Phi$ inductively defined according to the following grammar:
$$\Phi \quad := \quad true \mid \deff(v) \mid v = t_2 \mid t_1 \leq t_2 \mid \neg \Phi_1 \mid \Phi_1 \land \Phi_2$$
where $v \in \V$ and $t_1, t_2 \in \V \cup \bigcup_i \Delta_i$.
\end{definition}

Examples of queries of the loan scenarios are $request \leq \cvalue{5k}$ or $loanType=\cvalue{w}$.
Given a formula $\Phi$ or a term $t$, and an assignment $\assign$, we write $\Phi[\assign]$ (respectively $t[\assign]$) for the formula $\Phi'$ (or the term $t'$) where each occurrence of variable $v$ for which $\assign$ is defined is replaced by $\assign(v)$ (i.e.\ unassigned variables are not substituted).

\begin{definition}[Query language - semantics]
Given a data model $\dmodel$, an assignment $\assign$ and a query $\Phi \in \guardlang$ we say that $\dmodel, \assign$ satisfies $\Phi$, written $D, \assign \models \Phi$ inductively on the structure of $\Phi$ as follows:
\begin{compactitem}
\item $\dmodel, \assign \models true$;
\item $\dmodel, \assign \models \deff(v)$ iff $v[\assign] \neq v$;
\item $\dmodel, \assign \models t_1 = t_2$ iff $t_1[\assign], t_2[\assign] \not \in \V$ and $t_1[\assign] \equiv t_2[\assign]$;
\item $\dmodel, \assign \models t_1 \leq t_2$ iff $t_1[\assign], t_2[\assign] \in \Delta_i$ for some $i$ and $\ordf(\Delta_i)$ is defined and $t_1[\assign] \leq_{\Delta_i} t_2[\assign]$;
\item $\dmodel, \assign \models \neg \Phi$ iff it is not the case that $\dmodel, \assign \models \Phi$;
\item $\dmodel, \assign \models \Phi_1 \land \Phi_2$ iff $\dmodel, \assign \models \Phi_1$ and $\dmodel, \assign \models \Phi_2$.
\end{compactitem} 
\end{definition}

Intuitively, $\deff$ can be used to check if a variable has an associated value or not (recall that assignment $\eta$ is a partial function); equality has the intended meaning and $t_1 \leq t_2$ evaluates to true iff $t_1$ and $t_2$ are values belonging to the same domain $\Delta_i$, such a domain is ordered by a partial order $\leq_{\Delta_i}$ and $t_1$ is actually less or equal than $t_2$ according to $\leq_{\Delta_i}$.

\begin{lemma}\label{lemma:query:adom}
    Let $\Phi \in \guardlang$, and $\adom(\Phi)$ be its active domain -- i.e.\ the set of constants appearing in $\Phi$. Given an assignment $\assign$ we consider the data model $\dmodel'$ as the restriction of $\dmodel$ w.r.t.\ $\adom(\Phi)\cup \img(\assign)$, then
    \begin{equation*}
        \dmodel, \assign \models \Phi \text{ iff } \dmodel', \assign \models \Phi 
    \end{equation*}
\end{lemma}

The last lemma ca be easily proved by structural induction on the formula, and it makes the queries independent of the actual ``abstract'' domains enabling the use of finite subsets. Intuitively, the property holds because the query language doesn't include quantification over the variables.

\subsubsection{Data-aware net}\label{sec:our-net}

Data-AWare nets (\ournet) are obtained by combining the data model just introduced with a WF-net, and by formally defining how transitions are guarded by queries and how they update/delete data. 

\begin{definition}[\ournet]\label{def:our-net}
  A \emph{\ournet} is a tuple $\langle$$\dmodel,$ $\nmodel,$  $\writef,$~$\guardf$$\rangle$ where:
\begin{compactitem}
    \item $\nmodel=\tuple{P,T,F}$ is a WF-net;
    \item $\dmodel=\tuple{\V, \Delta, \domf, \ordf}$ is a data model;
    \item $\writef: T \mapsto \bigcup_{\V'\subseteq \V} (\V' \mapsto 2^{\domf(\varset)})$, where $\domf(\varset)=\bigcup_{v \in \varset} \domf(v)$, is a function that associates each transition to a (\emph{partial}) function mapping variables to a finite subset of their domain; that is satisfying the property that for each $t$, $\writef(t)(v)\subseteq\domf(v)$ for each $v$ in the domain of $\writef(t)$.
    \item $\guardf: T \mapsto \guardlang$ is a function that associates a guard to each transition.
  \end{compactitem}
\end{definition}

Function $\guardf$ associates a guard, namely a query, to each transition. The intuitive semantics is that a transition $t$ can fire if its guard $\guardf(t)$ evaluates to true (given the current assignment of values to data). Examples are $\guardf(T6)=request \leq \cvalue{5k}$ and $\guardf(T8)= \neg(request \leq \cvalue{99999})$.

Function $\writef$ is used to express how a transition $t$ modifies data: after the firing of $t$, each variable $v$ in the domain of $\writef(t)$ can take any value among a specific finite subset of $\domf(v)$. We have three different cases:
\begin{itemize}
 \item $\emptyset\subset \writef(t)(v) \subseteq \domf(v)$: $t$ nondeterministically assigns a value from $\writef(t)(v)$ to $v$;
  \item $\writef(t)(v) = \emptyset$: $t$ deletes the value of $v$ (hence making $v$ undefined);
  \item $v \not\in dom(\writef(t))$: value of $v$ is not modified by $t$.
\end{itemize}
Notice that by allowing $\writef(t)(v) \subseteq \domf(v)$ in the first bullet above we enable the specification of restrictions for specific tasks. E.g., $\writef(T4):\set{request} \mapsto \set{\cvalue{0} \ldots \cvalue{30k}}$ says that $T4$ writes the $request$ variable and intuitively that students can request a maximum loan of $\cvalue{30k}$, while $\writef(T5):\set{request} \mapsto \set{\cvalue{0} \ldots \cvalue{500k}}$ says that workers can request up to $\cvalue{500k}$.

The intuitive semantics of $\guardf$ and $\writef$ is formalized in the notion of \ournet valid firing, which is based on the definition of \ournet state, and state transition. 
A \ournet state includes both the state of the WF-net, namely its marking, and the state of data, namely the assignment. 

\begin{definition}[\ournet state]
A \emph{state} of a \ournet $\tuple{\dmodel, \nmodel, \writef, \guardf}$ is a pair $(\mrk,\assign)$ where $\mrk$ is a marking for $\tuple{P,T,F}$ and $\assign$ is an assignment for $\dmodel$.
\end{definition}

Analogously to traditional WF-nets, we distinguish two special states: the \emph{initial state} $(\mrk_s, \eta_s)$, which is such that $M_s$ is the initial marking of $\nmodel$ and $\eta_s$ is empty, i.e., $dom(\assign_0)=\emptyset$; and the \emph{final states} $(\mrk_e, \eta_e)$, such that $\mrk_e$ is the final marking of~$\nmodel$ (no conditions on assignment $\eta_e$).

\begin{definition}[\ournet Valid Firing]\label{def:dwnet:firing}
Given a \ournet $\tuple{\dmodel, \nmodel, \writef, \guardf}$, a firing of a transition $t\in T$ is a \emph{valid firing} from $(\mrk,\assign)$ to $(\mrk',\assign')$, written as $(\mrk,\assign)\fire{t}(\mrk',\assign')$, iff 
% conditions 1.\ and 2.\ of Def.~\ref{def:Firing} holds for $\mrk$ and $\mrk'$ (i.e., if $\mrk$ and $\mrk'$ it is a WF-Net valid firing)  and
  \begin{enumerate}
	  \item $\mrk\fire{t}\mrk'$ is a WF-Net valid firing
  % \item $t$ is enabled in $\mrk$, i.e., $\{ p\in P\mid \mrk(p)>0\}\supseteq \pres{t}$ and
  \item $\dmodel, \assign \models \guardf(t)$, 
  % \item $\mrk'$ is such that for every $p\in P$:
  % \begin{displaymath}
  %   \mrk'(p) =
  %   \begin{cases}
  %     \mrk(p)-1 & \text{if $p\in \pres{t}\setminus\posts{t}$}\\
  %     \mrk (p)+1  & \text{if $p\in \posts{t}\setminus\pres{t}$}\\
  %     \mrk (p) & \text{otherwise}
  %   \end{cases}
  % \end{displaymath}
  \item assignment $\assign'$ is such that, if $\textsc{wr} = \set{v \mid \writef(t)(v)\neq\emptyset}$, $\textsc{del}  = \set{ v \mid \writef(t)(v)=\emptyset}$:
  \begin{itemize}
  	\item its domain $dom(\assign') = dom(\assign)\cup \textsc{wr} \setminus \textsc{del}$;
        \item for each $v\in dom(\assign')$:
        \begin{displaymath}
          \assign'(v) =
          \begin{cases}
            d \in \writef(t)(v) & \text{if $v\in \textsc{wr}$}\\
            \assign(v)  & \text{otherwise.}
          \end{cases}
        \end{displaymath}
  \end{itemize}
\end{enumerate}
\end{definition}

Conditions 2.~and 3.\ extend the notion of valid firing of WF-nets imposing additional pre- and postconditions on data, and in particular preconditions on $\assign$ and postconditions on $\assign'$. Specifically, 2.\ says that for a transition $t$ to be fired its guard $\guardf(t)$ must be satisfied by the current assignment $\assign$. Condition 3.\ constrains the new state of data: the domain of $\assign'$ is defined as the union of the domain of $\eta$ with variables that are written ($\textsc{wr}$), minus the set of variables that must be deleted ($\textsc{del}$). Variables in $dom(\assign')$ can indeed be grouped in three sets depending on the effects of $t$: (i) $\textsc{old} = dom(\eta) \setminus \textsc{wr}$: variables whose value is unchanged after $t$; (ii) $\textsc{new} = \textsc{wr} \setminus dom(\eta)$: variables that were undefined but have a value after $t$; and (iii) $\textsc{overwr} = \textsc{wr} \cap dom(\eta)$: variables that did have a value and are updated with a new one after $t$.
The final part of condition 3.\ says that each variable in $\textsc{new} \cup \textsc{overwr}$ takes a value in $\writef(t)(v)$, while variables in $\textsc{old}$ maintain the old value $\eta(v)$.

Similarly to regular PNs, a \emph{case} of a \ournet is a sequence of \ournet valid firings 
$(\mrk_0,\assign_0)\fire{t_0} (\mrk_1,\assign_1) \fire{t_1} \ldots \fire{t_{k-1}} (\mrk_{k-1},\assign_{k-1})\fire{t_k}(\mrk_k,\assign_k) $ 
such that $(\mrk_0,\assign_0)$ is the initial state and $(\mrk_k,\assign_k)$ is a final state.

Note that, since we assumed a finite set of transitions and $\writef(t)$ is finite for any transition $t$, for any model $W$ we can consider its active domain $\adom(W)$ as the set of all constants in $\writef(t)$ and $\guardf(t)$ for all the transitions $t$ in $W$. Moreover, given the Lemma~\ref{lemma:query:adom}, we can consider the restriction of the model to the finite set of states defined by the active domain.

\begin{definition}[Finite \ournet models]\label{def:finite:dawnet}
   Let $W = \tuple{\dmodel, \nmodel, \writef, \guardf}$ be a \ournet with $\dmodel=\tuple{\V, \Delta, \domf, \ordf}$ and $\nmodel=\tuple{P,T,F}$, $\adom(W)$ be its active domain. We consider the data model $\dmodel_W$ as the restriction of $\dmodel$ w.r.t.\ $\adom(W)$.
   
   The finite version of $W$, denoted as \finite{W}, is defined as $\tuple{\dmodel_W, \nmodel, \writef, \guardf}$.
\end{definition}

By looking at the definition of valid firing and using Lemma~\ref{lemma:query:adom} it's easy to realise that the set of states of \finite{W} is closed w.r.t.\ the relation induced by the valid firing. As we'll see later on, this enables us to focus on finite \ournet models.

\paragraph{Reachability on \ournets}
  Analogously to PN, given a \ournet $W$ and a set of goal states $G$, the reachability problem in \ournets amounts to check whether there is a path from the initial state $(\mrk_s, \assign_s)$ of $W$ to any of the goal states in $G$.

  Although such a definition is quite general, in practical applications one may be interested in the reachability of a specific marking $M_g$ of $W$: in this case, the goal is the set of states of $W$ with $M_g$ as the first component and any assignment in the second (by virtue of Definition~\ref{def:finite:dawnet} and Lemma~\ref{lemma:query:adom}, goal sets defined in this way are always finite).
In fact, in this paper we focus on what we call \emph{clean termination}: the reachability, from the initial state, of any of the final states $(\mrk_e, \assign_e)$, namely no tokens in any place except in the sink, which should instead contain a (single) one, and any assignment $\assign_e$.
We remark, however, that our technique is general and allows for solving reachability of any set of goal states.

% subsection _ournet (end)
%%

%% file: dawnet-fts.tex
%!TEX root = ./main.tex

\section{Reachability: from theory to practice}
\label{sec:dawnet-fts}

In this section we illustrate our approach to verify reachability properties of \ournets by exploiting state-of-the-art techniques and tools. 
% Indeed, one of the contributions of this work is to test the performances of different paradigms and solutions able to solve planning problems in our setting. %% which is characterized to be data-intensive.
%%
The paradigms we selected are: action languages, classical planning, and model checking. The specific tools\footnote{Hereafter we will often refer to these tools as solvers.} we used are \clingo, \fastdw, and \nuxmv respectively, which in turn are paired with the representation languages \bclng, \pddllng, and \smv briefly summarised later in the paper.

In the following sections, given a \ournet $W$, we show how to encode $W$ in the specific language used by each paradigm/tools and formulate the reachability problem as a termination condition.
We then prove that such encodings are sound and complete, i.e., if a property is verified within one of the encodings, then it is indeed valid for the original \ournet, otherwise it is not.
Notably, we do so by providing the semantics of the above tools in terms of transition systems, which provide an intermediate representation that could support the extension of this work to new paradigms/tools whose semantics can be expressed likewise.
This also allows us to use the same structure for the sound and completeness proofs for the encodings, whose main conceptual steps are now described.

In order to exploit the notion of transition systems as a bridge between \ournet and the three chosen paradigms, we first need to 
connect transition systems and \ournets. To do so, we observe that the behaviour of a \ournet is possibly nondeterministic, as in each state, in general, more than one transition can fire: thus, while cases define \emph{a} possible \ournet evolution, a transition system, called \emph{reachability graph}, captures \emph{all} possible evolutions.

\begin{definition}[\ournet reachability graph] \label{def:fin-reach-graph}
Let $W = \tuple{\dmodel, \nmodel, \writef, \guardf}$ be a \ournet with $\dmodel=\tuple{\V, \Delta, \domf, \ordf}$ and $\nmodel=\tuple{P,T,F}$, then its reachability graph is a transition system $\RG_{W}=(T, \ol{S}, \ol{s}_0, \ol{\delta})$ where:
\begin{itemize}
\item $\ol{S}$ is the set of states, each representing a \ournet $(\mrk, \assign)$ state;
\item $\ol{s}_0 = (\mrk_s, \assign_s)$ is the initial state;
\item $\ol{S}$ and $\ol{\delta} \subseteq \ol{S} \times T \times \ol{S}$ are defined by induction as follows:
\begin{compactitem}
\item $\ol{s}_0 \in \ol{S}$;
\item if $(\mrk, \assign) \in \ol{S}$ and $(\mrk, \assign) \fire{t} (\mrk',\assign')$ is a valid firing then $s'=(\mrk',\assign') \in \ol{S}$ and $(s, t, s') \in \ol{\delta}$ (and we write $s \fire{t} s' \in \ol{\delta}$).
\end{compactitem}
\end{itemize}
\end{definition}

Let us consider the finite version of $W$ as defined in Definition~\ref{def:finite:dawnet}; clearly the initial state is among its states and, since the set of its states is closed w.r.t.\ the relation induced by the valid firings, we get the following: 

\begin{lemma}\label{lemma:dawnet:adom}
  Let $W = \tuple{\dmodel, \nmodel, \writef, \guardf}$ be a \ournet model and \finite{W} its finite version, then their reachability graphs are identical.\end{lemma}

In virtue of the above lemma, in the rest of the paper when we mention the states of a \ournet $W$ we intend the states of  $\finite{W}$, that is, we restrict to the finite set of states defined by the active domain.

As we are going to show in the following sections, the semantics of the chosen paradigms can be captured by means of transition systems as well. Therefore, given a \ournet $W$, let us denote with $\bclng(W)$, $\pddllng(W)$ and $\smvlng(W)$ the encodings of $W$ in the respective languages and with $\TS_{\bclng(W)}$, $\TS_{\pddllng(W)}$ and $\TS_{\smvlng(W)}$ the transition systems generated by the above encodings.
A \emph{path} of a transition system $\RG_{W}$ is a sequence of states obtained by starting from the initial state of $\RG$ and traversing its transitions, thus it represents a possible evolution or behaviour (indeed, each path of $\RG$ is a trace of $W$ by definition). If we can show that for each path in $\TS_\ast$, with $\ast \in \set{\bclng(W), \pddllng(W), \smvlng(W)}$, there is an \emph{equivalent} (see later) path in $\RG$, then the encoding of the reachability problem for $W$ in the $\ast$ language is sound, as the tool only generates \ournet behaviours. Likewise, if for each path in $\RG$ there is a path in $\TS_\ast$, then the encoding is complete as all \ournet behaviours are captured by $\TS_\ast$. This property is called \emph{trace equivalence}~\cite{mc-book}, and is of utmost importance for us as trace-equivalent transition systems are indistinguishable by linear properties, which include reachability properties. In other words, if $\RG$ is trace equivalent with $\TS_{\ast}$, a reachability property in $\TS_{\ast}$ is true if and only if it is true in $\RG$, which allows us to use the above tools for solving our trace repair problem.

In order to formalise trace equivalence we shall first define what does it mean for two paths, generated by different transition system, to be equivalent.

Let $\pi : s_0 \fire{t_1} s_1 \fire{t_2} \ldots \fire{t_{n}} s_n$ and $\pi' : s_0' \fire{t_1'} s_1' \fire{t_2'} \ldots \fire{t_{n}'} s_n'$ be paths generated by $RG$ and $TS_{\ast}$ respectively. We say that $\pi$ and $\pi'$ are \emph{equivalent} if $s_0 = s_0'$ and for each $i \in \set{1, \ldots, n}$ we have that $s_i = s_i'$ and $t_i = t_i'$.
Technically however, states of $\TS_{\ast}$ encode information on the state of the net $(\mrk, \assign)$ in different ways, and possibly also use some additional information for technical reasons.
Let us then call $\enc_{\ast}(\mrk, \assign)$ the encoding of $(\mrk, \assign)$, and $\enc_{\ast}(t)$ the encoding of $t$ in the language $\ast \in \set{\clts, \fdts, \smvts}$. Then $\pi$ is equivalent to $\pi'$ (w.r.t. $\enc_{\ast}$) if $\enc_{\ast}(s_0) = s_0'$ and for each $i \in \set{1, \ldots, n}$ we have that $\enc_{\ast}(s_i) = s_i'$ and $\enc_{\ast}(t_i) = t_i'$.
%\todordminline{With queries?}

\begin{definition}[Trace equivalence]\label{def:trace-equivalence}
	Let $\RG$ be a \ournet reachability graph and $TS_{\ast}$ a transition system, then $\RG$ and $\TS_{\ast}$ are \emph{trace equivalent} iff there is an injective function $\enc_{\ast}$ from the states and transitions of $\RG$ to those of $\TS_{\ast}$ s.t.
	\begin{compactenum}
	\item for each path $s_0 \fire{t_1} s_1 \fire{t_2} \ldots \fire{t_{n}} s_n$ of $\RG$, $$\enc_{\ast}(s_0) \fire{\enc_{\ast}(t_1)} \enc_{\ast}(s_1) \fire{\enc_{\ast}(t_2)} \ldots \fire{\enc_{\ast}(t_{n})} \enc_{\ast}(s_n)$$ is a path of $\TS_{\ast}$; and
	\item for each path $s_0' \fire{t_1'} s_1' \fire{t_2'} \ldots \fire{t_{n}'} s_n'$ of $\TS_{\ast}$ there is a path $s_0 \fire{t_1} s_1 \fire{t_2} \ldots \fire{t_{n}} s_n$ of $\RG$ s.t.\ $\enc_{\ast}(s_0) = s'_0$, and $\enc_{\ast}(s_i) = s'_i$ $\enc_{\ast}(t_i) = t'_i$ for $1\leq i\leq n$.
	\end{compactenum}

\end{definition}

Recall that the reachability problem expresses whether, given a set of goal states $G$ of $W$, at least one of those can be reached from the initial one $(M_{s}, \eta_{s})$. Given Definition~\ref{def:fin-reach-graph}, it immediately follows that if $\RG$ and $\TS_{\ast}$ are trace equivalent, then, if a state is reachable in $\RG$, it is also reachable in $\TS_{\ast}$ by performing the very same transitions and vice-versa, thus they satisfy the same reachability properties.

% Although it would be possible to verify different reachability properties, i.e.\ different sets of final states, in this paper we focus on the so called \emph{clean termination}\todo{is that the proper name?}. That is, any of the states in which there are no tokens in all the places except the sink, which should instead contain a (single) one. \btext{Formally, the set of goal states $G$ will be the states of $W$ with $M_e$ as first component and any assignment in the second.}

Note that, in order to exploit the specific algorithms to verify reachability in the ``target'' system, we need also to show that the corresponding language for specifying the termination property is expressive enough to capture all and only the states corresponding to the final states.

% However, proving trace equivalence is quite convoluted. For this reason is customary to prove a stronger property, called bisimilarity~\cite{mc-book}, well-known to imply the former. 
% %%
% States of two transition systems are in \emph{bisimulation relation} if: \begin{compactenum}[\it(i)] \item they satisfy the same \emph{local} properties and \item every step of one can be matched by a step in the other and vice-versa.
%   \end{compactenum}
% %%
% We thus prove trace equivalence by showing that there exists such a relation between $RG$ and $TS_{\ast}$ and the initial states of $RG$ and $TS_\ast$ are there included.

% More concretely, the main steps to prove soundness and completeness of the three encodings with $RG$ are thus similar. To make the presentation more modular, we therefore give here the shared details and leave the specifics in the subsections, one for each encoding. We start from the definition of \emph{local bisimilarity}, which capture the fact that two states 

In the next sections, for each language $\ast \in \set{\clts, \fdts, \smvts}$ we 
prove that $\RG$ and $\TS_{\ast}$ are trace equivalent, and therefore they can be used to solve the above mentioned reachability problem. These sections are organised in three parts: in the first one the specific formalism is introduced, then the following ones provide the encoding and the proofs.

%%% Local Variables:
%%% mode: latex
%%% TeX-master: "main.tex"
%%% save-place: t
%%% End:

%% file: encoding-BC.tex
%!TEX root = ./main.tex

\section{The encoding in Action Languages}\label{sec:bc:encoding}

In this section we will show how the transition system underlying \ournet can be represented by means of the so called \emph{action languages}~\cite{lif99}, in order to use the ASP based solver \clingo to verify reachability properties. In our work we will be focusing on the \bclng language as introduced in~\cite{lee_action_2013}, but the same idea can be applied to different frameworks (e.g.\ \dlvk~\cite{eiter_dlvk:2003}).

\subsection{The \mbox{\normalfont \textsc{bc}} action language}

An action description \bdescr in the language \bclng includes a finite set of symbols: \emph{fluent}, and \emph{action constants}. Each fluent constant is associated to a finite set, of cardinality greater or equal than 2, called the domain. Boolean fluents have the domain equal to $\{\true, \false\}$. An atom is an expression of the form $f = v$, where $f$ is a fluent constant, and $v$ is an element of its domain.

A \bclng description \bdescr contains \emph{static} and \emph{dynamic laws}; the former are 
expressions of the type
\begin{equation}
	\label{eq:staticrule}
	% \begin{lstlisting}
 	  \text{\lstinline|$A_0\ $ \ if \ $A_1,\ldots, A_n\ $ ifcons \ $A_{n+1},\ldots, A_m$|}
% 	\end{lstlisting}
\end{equation}
while the second are of the form
\begin{equation}
	\label{eq:dynamicrule}
	% \begin{lstlisting}
  	\text{\lstinline|$A_0\ $ after \ $A'_1,\ldots, A'_n\ $  ifcons \ $A_{n+1},\ldots, A_m$|}
	%\end{lstlisting}
\end{equation}
with $0 \leq n\leq m$, where $A_0, A_{1},\ldots, A_m$ are atoms, and $A'_1,\ldots, A'_n$ can be atoms or action constants.

Intuitively, static laws assert that each state satisfying $A_1,\ldots, A_n$ and where $A_{n+1},\ldots, A_m$ can be assumed satisfies $A_0$ as well. In dynamic laws, instead, $A'_1,\ldots, A'_n$ must be satisfied in the preceding state. An action constant is satisfied if the action is being ``executed'' in a specific state.

Initial states can be specified with statements of the form
\begin{equation}
	\label{eq:initial}
	% \begin{lstlisting}
	\text{\lstinline| initially \ $A$|}
	%\end{lstlisting}
\end{equation}
where $A$ is an atom, that is, an expression of the form $f=v$.

Semantics of a \bclng description $\bdescr$ is given in terms of a translation into a sequence $P_\ell(\bdescr)$ of ASP programs where atoms are of the form $i : A$ with $0\leq i\leq \ell$, and $A$ being a \bclng fluent or action constant. Each static law \eqref{eq:staticrule} is translated as the set of ASP rules with strong negation
\begin{displaymath}
  i : A_0 \leftarrow i: A_1,\ldots, i : A_n, \naf \neg (i:A_{n+1}),\ldots, \naf \neg (i:A_m)
\end{displaymath}
where $0\leq i \leq\ell$. Each dynamic law \eqref{eq:dynamicrule} is translated as
\begin{displaymath}
  i+1 : A_0 \leftarrow i: A_1,\ldots, i : A_n, \naf \neg (i+1:A_{n+1}),\ldots, \naf \neg (i+1:A_m)
\end{displaymath}
where $0\leq i< \ell$. If $A_0$ is \lstinline|false| the corresponding rules are constraints.

Initial state constraints of the form \lstinline{initially f=v} force the value of fluent \lstinline{f} in the first state and are therefore translated as
\begin{displaymath}
  0:f=v
\end{displaymath}

In addition to the rules corresponding to static / dynamic laws and initial constraints, the ASP program contains rules to: \begin{inparaenum}[\it(i)] \item set the initial state by nondeterministically assigning the values of fluents
\[0:f=v \lor \neg (0:f=v)\]
for each fluent $f$ and elements $v$ of its domain; \item force total knowledge on action execution
\[i:a \lor \neg (i:a)\]
for each action constant $a$ and $i<\ell$; and finally \item impose existence and uniqueness of values for the fluents\end{inparaenum}
\begin{align*}
  &\leftarrow \naf (i:f=v_1), \ldots, \naf (i:f=v_k) \\
  \neg i:f=v &\leftarrow i:f=w
\end{align*}
for any fluent $f$, $v, v_1, \ldots, v_k, w$ elements of its domain, $v\neq w$, and $i\leq\ell$.

Given a stable model $M_\ell$ of $P_\ell(\bdescr)$, we indicate as $\bcassign{i}(M_\ell)$ the variable assignments in the corresponding state $0\leq i\leq\ell$:
\begin{displaymath}
    \bcassign{i}(M_\ell) = \{ f\mapsto o \mid f\in\F, \feqlit{i}{f}{o}\in M_\ell \}
\end{displaymath}
Note that the above rules on fluents guarantee that $\bcassign{i}(M_\ell)$ is well defined: total and functional.

Although not originally introduced in~\cite{lee_action_2013}, we introduce the implicit transition system induced by a \bclng description in order to relate it with the \ournet reachability graph as explained in Section~\ref{sec:dawnet-fts}.

\begin{definition}[\bclng transition system]\label{def:bc:trans}
    Let \bdescr be a \bclng description over the set \F of fluent, and \A of action constants. Let $S_\bdescr$ be the set of all \emph{total} assignments for fluents in \F satisfying the domain restrictions. Then its transition system $\TS_{\bclng(\bdescr)}=(\A, S, S_0, \delta)$ is defined as:
    \begin{itemize}
        \item $S_0\subseteq S_\bdescr$ is the set of initial states
        \[ S_0 = \{ \bcassign{0}(M_0) \mid \text{$M_0$ is a stable model of $P_0(\bdescr)$}\}\]
        \item $S\subseteq S_\bdescr$ and $\delta\subseteq S_\bdescr\times 2^\A \times S_\bdescr$ are the minimum sets s.t.:
        \begin{compactitem}
        \item $S_0 \subseteq S$;
        \item if $M_{\ell+1}$ is a stable model of $P_{\ell+1}(\bdescr)$ for some $\ell\geq 0$, and $\bcassign{\ell}(M_{\ell+1})\in S$; then $\bcassign{\ell+1}(M_{\ell+1})\in S$ and $(\bcassign{\ell}(M_{\ell+1}),\{a\in\A \mid \flit{\ell}{a}\in M_{\ell+1} \},\bcassign{\ell+1}(M_{\ell+1}))\in \delta$
        \end{compactitem}
     \end{itemize}
\end{definition}

While states of a \bclng transition system can be easily compared with \ournet reachability graph states, \bclng enables concurrent and ``empty'' transitions ($\delta\subseteq S_\bdescr\times 2^\A \times S_\bdescr$); therefore the encoding should take that into account (see Section~\ref{sec:bcencoding} for details).

The solver
%tool
 we use in Section \ref{sec:evaluation} to compute reachability on problems expressed with the \bclng language is the Coala compiler and \clingo ASP solver~\cite{gebserCoalaCompilerAction2010}. Within the tool, the final states to define the reachability problem can be specified using a set of (possibly negated) atoms using \lstinline|finally| statements analogous to the \lstinline|initially| previously described in Equation~\eqref{eq:initial}.

\subsection{Encoding \ournets in \mbox{\normalfont \textsc{bc}}}
\label{sec:bcencoding}

Given a \ournet $W$, its encoding in the \bclng, denoted as $\bclng(W)$, introduces:\begin{inparaenum}[\it(i)] \item a fluent for each variable in $W$, with the domain corresponding to the variable domain plus a special constant \nullv (representing the fact that a variable is not assigned); \item a boolean fluent for each place; and \item an action constant for each transition\end{inparaenum}. To simplify the notation we will use the same constant names. Moreover, without loss of generality we assume that guards are in conjunctive normal form.\footnote{We do not expect huge guards, therefore we ignore the potential exponential blowup due to the conversion to CNF of arbitrary queries.}

Let $\langle\dmodel,\nmodel,\writef,\guardf\rangle$ be a \ournet as defined in Definition~\ref{def:our-net}, where $\nmodel=\tuple{P,T,F}$ and $\dmodel=\tuple{\V, \Delta, \domf, \ordf}$. Let $\V'\subseteq\V$ be the (finite) set of variables appearing in the model, and $\adom(v)$ be the \emph{active domain} of $v\in\V'$, that is, $\adom(v) = \bigcup_{t\in T} \writef(t)(v)$.
The encoding $\bclng(W)$ starts by representing in \bclng key properties of transactions in a \ournet:

\begin{itemize}
	\item All fluents are ``inertial'', that is their value propagates through states unless their value is changed by the transitions:
\begin{align}
&\text{\lstinline|v=o after v=o ifcons v=o|} & & \text{for $\lstmath{v}\in\V'$ and $\lstmath{o}\in \adom(v)\cup\{\nullv\}$}\label{eq:bcEnc:var:inertia}\\
&\text{\lstinline|p=o after p=o ifcons p=o|} & & \text{for $\lstmath{p}\in P$ and $\lstmath{o}\in \{\lstmath{true}, \lstmath{false}\}$} \label{eq:bcEnc:inertiap}
\end{align}

\item Transitions modify the value of places and variables; for each transition $\text{\lstinline|t|}\in T$:
{\small
\begin{align}
&\text{\lstinline|p=false after t|} & & \text{for all \lstinline|p|}\in \pres{\lstmath{t}}\setminus\posts{\lstmath{t}} \label{eq:bcEnc:inp_update} \\
&\text{\lstinline|p=true after t|} & & \text{for all \lstinline|p|}\in\posts{\lstmath{t}}\setminus\pres{\lstmath{t}} \label{eq:bcEnc:outp_update}\\
&\text{\lstinline|v=d after t ifcons v=d|} & & \text{for all (\lstinline|v|,\lstinline|d|)}\in \writef(\lstmath{t}) \\
&\text{\lstinline|v=null after t|} & & \text{for all \lstmath{v} s.t.\ } \writef(\lstmath{t})(\lstmath{v}) = \emptyset \label{eq:bcEnc:var:del} \\
&\text{\lstinline|false after t ifcons v=d|} & & 
\begin{minipage}{5cm}
\text{for all $\lstmath{v}\in\V'$ s.t.\ $\writef(\lstmath{t})(\lstmath{v})\neq\emptyset$,}\\
\text{\qquad $\lstmath{d}\in \{\nullv\}\cup\adom(\lstmath{v})\setminus\writef(\lstmath{t})(\lstmath{v})$}
\end{minipage}
 \label{eq:bcEnc:var:comp}
\end{align}
}
Note that equation \eqref{eq:bcEnc:var:comp} ensures that after the execution of a transition the value of a variable modified by it must be among the legal values.

\item Transitions cannot be executed in parallel:
\begin{align}
&\text{\lstinline|false after t, s|} & & \text{for (\lstmath{t}, \lstinline|s|)}\in T\times T \text{ and \lstinline|t|${}\neq{}$\lstinline|s|}
\end{align}

\item Transitions cannot be executed unless input places are active:
\begin{align}
&\text{\lstmath{false after t, p=false}} & & \text{for \lstinline|t|}\in T \text{ and \lstinline|p|}\in \pres{\lstmath{t}}\label{eq:bcEnc:inplaces}
\end{align}

\item Initially, all the places but the start one are false and variables are unassigned
\begin{align}
&\lstmath{initially start=true}\\
&\lstmath{initially p=false} & & \text{for \lstinline|p|}\in P \text{ and \lstinline|p|}\neq \lstmath{start}\\
&\lstmath{initially v=null} & & \text{for \lstinline|v|}\in \V'
\end{align}

\item In \bclng a state progression does not require an actual action. To disallow this scenario we need to track states that have been reached without the execution of an action. This is achieved by including an additional boolean fluent \lstmath{trans} that is true if a state has been reached after a transition:
\begin{align}
&\lstmath{trans=true after t} & & \text{for $\lstmath{t}\in T$}\label{law:bcenc:trans} \\
&\lstmath{initially trans=true}
\end{align}
\end{itemize}

We now describe the encoding in \bclng of \ournet queries. Due to the rule-based essence of \bclng we consider a Disjunctive Normal Form characterisation of the \ournet queries defined in Definition~\ref{def:dawnet:qlang}): for any \ournet query $\Phi$ there is an equivalent normal form $\ol{\Phi}$\footnote{$\ol{\Phi}$ is not unique.} s.t.\ $\ol{\Phi}=\bigvee_{i=1}^k t_1^i\land\ldots \land t_{\ell_i}^i$ where each term $t_j^i$ is either of the form $v = o$ or $\neg\deff(v)$.
To prove the existence of such formula $\ol{\Phi}$ we consider the disjunctive normal form of $\Phi$, and to each term not in the form $v = o$ or $\neg\deff(v)$ we apply the following equivalences that are valid when considering the finite domain data model $\dmodel_W$ (which we can consider without loss of generality because of Lemma~\ref{lemma:dawnet:adom}):
\begin{align*}
  v_1 = v_2 &\equiv \bigvee_{o\in \domf(v_1)\cap\domf(v_2)} v_1 = o \land v_2 = o \\
  \neg (v_1 = v_2) &\equiv \bigvee_{(o_1,o_2)\in \domf(v_1)\times\domf(v_2), o_1\neq o_2} v_1 = o_1 \land v_2 = o_2 \\
  \neg (v = o ) &\equiv \bigvee_{o\in \domf(v)\setminus\{o\}} v = o \\
   v \leq o &\equiv \bigvee_{(o',o)\in \ordf(v)} v = o \\ 
   o \leq v &\equiv \bigvee_{(o,o')\in \ordf(v)} v = o \\ 
   v_1 \leq v_2 &\equiv \bigvee_{(o_1,o_2)\in \ordf(v_1)\cap\ordf(v_2)} v_1 = o_1 \land v_2 = o_2 \\ 
   \deff(v) &\equiv \bigvee_{o\in \domf(v)} v = o
\end{align*}

Given a \ournet query $\Phi$, and its Disjunctive Normal Form characterisation  $\ol{\Phi}=\bigvee_{i=1}^k t_1^i\land\ldots \land t_{\ell_i}^i$, we define the translation $\bcqlang{\Phi}$ of $\Phi$ as $\bigvee_{i=1}^k \bcqlang{t_1^i}\land\ldots \land \bcqlang{t_{\ell_i}^i}$; where $\bcqlang{v = o} \mapsto \lstmath{v = o}$ and $\bcqlang{\neg\deff(v)} \mapsto \lstmath{v = null}$.\footnote{This rather inefficient explicit representation is convenient for the proofs, and it can be optimised in the actual ASP implementation via grounding techniques and explicit representation of the orderings by means of ad hoc predicates.}

Having encoded the queries we can now encode the fact that transitions cannot be executed unless the guards are satisfied. This condition is encoded as a set of dynamic constraints preventing the execution unless at least one of the clauses of the guard is not satisfied.
For each transition $t$ s.t.\ $\guardf(t) \not\equiv true$ we consider $\bcqlang{\neg\guardf(t)}=\bigvee_{i=1}^k \bcqlang{t_1^i}\land\ldots \land \bcqlang{t_{\ell_i}^i}$ and include the following dynamic constraints:
\begin{equation}
    \begin{aligned}
        &\text{\lstinline!false after t, !} \bcqlang{t_1^1}, \ldots, \bcqlang{t_{\ell_1}^1}\\
        &\ldots \\
        &\text{\lstinline!false after t, !} \bcqlang{t_1^k}, \ldots, \bcqlang{t_{\ell_k}^k}
    \end{aligned}
\end{equation}
The set of final states for the reachability problem can be specified as the ones in which all place fluents are \false\ but the one corresponding to the sink:
\begin{equation}\label{eq:bcenc:final}	
\begin{aligned}
&\lstmath{finally sink=true}\\
&\lstmath{finally p=false} & & \text{for \lstinline|p|}\in P \text{ and \lstinline|p|}\neq \lstmath{sink}
\end{aligned}
\end{equation}

\subsection{Correctness and completeness of the {\normalfont\bclng} encoding}

Let $\bclng(\dwnmodel)$ be the \bclng encoding of a \ournet \dwnmodel. In this section we are going to show that $\TS_{\bclng(\dwnmodel)}$ and the reachability graph of \dwnmodel are trace equivalent as per Definition~\ref{def:trace-equivalence}. This section reports all the main  steps involved, while further technical details and proofs are presented in~\ref{apx:bc:encoding}.

To simplify the proofs we first relate sequences of valid firings in \dwnmodel\ of arbitrary length $\ell$ to stable models of $P_\ell(\bclng(\dwnmodel))$. In fact, the definition of $\TS_{\bclng(\dwnmodel)}$ (Definition~\ref{def:bc:trans}) is based on the notion of stable models.

To map sequences of valid firings into stable models we observe that stable models $P_\ell(\bclng(\dwnmodel))$ include both the fluent assignments and the information regarding the action constants that ``cause'' a specific state transition; therefore we split the mapping in two parts to separate these two distinct aspects:
\begin{definition}\label{def:dwnet:bc:map}
    Let $\rho = (\mrk_0,\assign_0)\fire{t_0}(\mrk_1,\assign_1)\fire{t_2}\ldots\fire{t_{\ell-1}}(\mrk_\ell,\assign_\ell)$ be a sequence of valid firings in \dwnmodel (with $\ell\geq 0$).\footnote{We consider $(\mrk_0,\assign_0)$ a sequence of size $0$.} The mapping $\dwnetbc_i(\rho)$ (for $0\leq i\leq \ell$) is defined as following:
    \begin{align*}
        \dwnetbc_i^\nu(\rho) = {}&\{ i:p=\true, \neg(i:p=\false)\mid p\in P, \mrk_i(p) > 0 \} \cup {} \\
        &\{ i:p=\false, \neg(i:p=\true)\mid p\in P, \mrk_i(p) = 0 \} \cup {} \\
        &\{ i:v=o, \neg(i:v=\nullv) \mid v\in\V', \assign_i(v) = o \} \cup {} \\
        &\{ i:v=\nullv \mid v\in\V', \assign_i(v) \text{ is undefined}\} \cup {} \\
        &\{ \neg(i:v=o) \mid v\in\V', o\in\adom(v), \assign_i(v) \neq o \text{ or } \assign_i(v) \text{ is undefined}\} \cup {} \\
        &\{ i:\lstmath{trans}=\true, \neg(i:\lstmath{trans}=\false)\} \\
        \dwnetbc_i^\tau(\rho) = {}&
        \begin{cases}
            \emptyset & \text{for $i\geq\ell$}\\
            \{ i:t_{i}\} \cup \{ \neg(i:t)\mid t\in T, t\neq t_{i} \} & \text{for $i<\ell$}
        \end{cases}\\
        \dwnetbc_i(\rho) = {}& \dwnetbc_i^\nu(\rho) \cup \dwnetbc_i^\tau(\rho)
    \end{align*}
\end{definition}

Note that Definition~\ref{def:dwnet:bc:map}, together with Definition~\ref{def:bc:trans}, enable us to say that the firing $(\mrk_i,\assign_i)\fire{t_i}(\mrk_{i+1},\assign_{i+1})$ in $\rho$ corresponds to the $\TS_{\bclng(\dwnmodel)}$ transition $\bcassign{i}(\dwnetbc_i^\nu(\rho))\fire{\{t\mid i:t \in \dwnetbc_i^\tau(\rho)\}}\bcassign{{i+1}}(\dwnetbc_{i+1}^\nu(\rho))$. Moreover, the definition of $\dwnetbc_i^\nu(\cdot)$ depends only on the state $(\mrk_i,\assign_i)$ and not on the selected path $\rho$, which provides just the ``witness'' for being a state of the reachability graph.

Before proving completeness and correctness we show that the states of the two transition systems preserve the satisfiability of $\guardlang[\dmodel_{\dwnmodel}]$ queries:

\begin{restatable}{lemma}{lemmaBCquery}\label{lemma:bc:query}
  Let  $\dwnmodel$ be a \ournet model, and $\rho = (\mrk_0,\assign_0)\fire{t_0}(\mrk_1,\assign_1)\fire{t_2}\ldots\fire{t_{\ell-1}}(\mrk_\ell,\assign_\ell)$ be a sequence of valid firings of \dwnmodel (with $\ell\geq 0$).
  For every state $(\mrk_i, \assign_i)$ ($0\leq i\leq \ell$) and query $\Phi\in\guardlang[\dmodel_{\dwnmodel}]$
    \begin{equation*}
        \dmodel_i, \assign_i \models \Phi \text{ iff }  \dwnetbc_i(\rho)\models i:\bcqlang{\Phi}
    \end{equation*}
  where $i:\bcqlang{\Phi}$ is the formula where $i:\cdot$ is placed in front of each term $\bcqlang{t}$, and $\dwnetbc_i(\rho)\models i:\bcqlang{t}$ iff $i:\bcqlang{t}\in\dwnetbc_i(\rho)$.
\end{restatable}
The proof is by induction on \guardlang[\dmodel_{\dwnmodel}].

We are now ready to state the completeness of \bclng encoding.

\begin{restatable}[Completeness of \bclng encoding]{lemma}{lemmaBCencComplete}\label{lemma:bcenc:complete}
    Let \dwnmodel be a \ournet and $\bclng(\dwnmodel)$ its encoding, then for any $\ell\geq 0$ if $\rho$ is a sequence of valid firings of \dwnmodel of length $\ell$, then $\bigcup_{i=0}^\ell\dwnetbc_i(\rho)$ is a stable model of $P_\ell(\bclng(\dwnmodel))$.
\end{restatable}

To prove correctness we first prove few properties about the reachability graph $\TS_{\bclng(\dwnmodel)}$.

\begin{restatable}{lemma}{lemmaDWNetBCmap}\label{lemma:dwnet:bc:map}
    Let $\TS_{\bclng(\dwnmodel)}=(\A, S, S_0, \delta)$, then
    \begin{enumerate}
        \item $S_0 = \{ s_0 \}$;
        \item if $(s,A,s')\in \delta$ then $|A| = 1$.
   \end{enumerate}   
 \end{restatable}
By using this lemma, given a \ournet model $\dwnmodel$, we can compare the transition system corresponding to its reachability graph $\RG_W=(T, \ol{S}, \ol{s}_0, \ol{\delta})$ to $\TS_{\bclng(\dwnmodel)}$ (Definition~\ref{def:bc:trans}). This because it guarantees that there is a single initial state, and the transitions (the $\delta$ of Definition~\ref{def:bc:trans}) contain exactly one action each.\footnote{\bclng allows transitions without or with multiple parallel actions.} The following definition introduces the mapping from stable models to sequences of valid firings.

\begin{definition}
    For \emph{any} stable model $M_\ell$ of $P_\ell(\bclng(\dwnmodel))$ with $\ell\geq 0$ we also define the ``inverse'' relation $\bcdwnet_i(M_\ell) = (\mrk_i, \assign_i)$ for $0\leq i\leq \ell$:
    \begin{align*}
        \mrk_i = & \{ (p, 1)\mid p\in P, \bcassign{i}(M_\ell)(p)=\true \}\cup {} \\
        & \{ (p, 0)\mid p\in P, \bcassign{i}(M_\ell)(p)=\false  \} \\
        \assign_i = & \{ (v, o)\mid v\in V', \bcassign{i}(M_\ell)(v)=o, o\neq\nullv \}
    \end{align*}
    and the transition $\tau_i(M_\ell)$ for $0\leq i< \ell$:
    \begin{align*}
      \tau_i(M_\ell) = a && \text{s.t.\ $(i:a)\in M_\ell$}
    \end{align*}
    where $a$ is an action constant in $P_\ell(\bclng(\dwnmodel))$.
\end{definition}

\begin{restatable}[Correctness of \bclng encoding]{lemma}{lemmaBCencCorrect}\label{lemma:bcenc:correct}
    Let \dwnmodel be a \ournet and $\bclng(\dwnmodel)$ its encoding, then for any $\ell\geq 0$ if $M_\ell$ is a stable model of $P_\ell(\bclng(\dwnmodel))$, then 
        \begin{equation*}
          \rho = \bcdwnet_0(M_\ell) \fire{\tau_0(M_\ell)}\bcdwnet_1(M_\ell) \fire{\tau_1(M_\ell)}\ldots\fire{\tau_{\ell-1}(M_\ell)}\bcdwnet_\ell(M_\ell)  
        \end{equation*}
        is a sequence of valid firings of \dwnmodel of length $\ell$.
\end{restatable}

We are now ready to state the main result of this section, namely trace equivalence between $\bctrans(\bclng(\dwnmodel))$ and $RG=(T, \ol{S}, \ol{s}_0, \ol{\delta})$. We do that by instantiating the abstract notion of $\enc_{\ast}$ introduced in Section~\ref{sec:dawnet-fts} to \bclng, and by using the completeness and correctness lemmata~\ref{lemma:bcenc:complete} and \ref{lemma:bcenc:correct}.

\begin{theorem}\label{th:bc-sound-comp}
  Let $\dwnmodel$ be a \ournet model and $\RG=(T, \ol{S}, \ol{s}_0, \ol{\delta})$ its reachability graph, then $\TS_{\bclng(\dwnmodel)}$ and $\RG$ are trace equivalent.
\end{theorem}

\begin{proof}
  To demonstrate trace equivalence we define the mapping $\enc_{\clts}(\cdot)$ as follows: for each $(\mrk,\assign)$
  \begin{align*}
    \enc_{\clts}((\mrk,\assign)) = {}&\{ (p,\true)\mid (p,1)\in\mrk \}\cup {} \\
    & \{(p,\false)\mid (p,0)\in\mrk\}\cup {} \\
    & \{(v,\nullv)\mid v \text{ not in the domain of }\assign\}\cup {} \\
    & \assign
  \end{align*}
   and $\enc_{\clts}(t) = \{t\}$ for each transition in $T$.
   
  Note that, by construction, for each sequence of valid firings $(\mrk_0,\assign_0)\fire{t_0}(\mrk_1,\assign_1)\fire{t_2}\ldots\fire{t_{\ell-1}}(\mrk_\ell,\assign_\ell)$,   $\enc_{\clts}((\mrk_i,\assign_i))=\bcassign{i}(\dwnetbc_i^\nu(\rho))$, and for any stable model $M_\ell$ of $P_\ell(\bclng(\dwnmodel))$ $\enc_{\clts}(\bcdwnet_i(M_\ell)) = \bcassign{i}(M_\ell)$.
  Therefore the proof is a consequence of the  lemmata \ref{lemma:bcenc:complete} and \ref{lemma:bcenc:correct}.
\end{proof}

The trace equivalence demonstrated by Theorem~\ref{th:bc-sound-comp} enables the use of ASP based planning to verify the existence of a plan satisfying the termination conditions expressed in Equation~\eqref{eq:bcenc:final}. 
Moreover, it is immediate to verify that the set of final states as characterised in Equation \eqref{eq:bcenc:final} corresponds to all and only the final states of the \ournet model mapped via $\enc_{\clts}(\cdot)$.

%In the case that the model is a trace workflow $\Wtau$, then the plan would exhibit a valid completion of the trace $\tau$; while the failure of finding such a plan would imply that the trace do not admit any completion.

%% file: encoding-PDDL.tex
%!TEX root = ./main.tex

\section{The encoding in Classical Planning}\label{sect:pddl:enc}

Automated planning has been widely employed as a reasoning framework to verify reachability properties of finite transition systems. It is therefore natural to include it among the verification tools that we considered for our evaluation.

In order to be as general as possible without committing ourselves to a specific planner we based our encoding on the de facto standard, namely PDDL~\cite{ghallab_pddl_1998}, however to simplify the formal discussion we will adopt the \emph{state-variable} representation~\cite{ghallab_automated_2016} instead of the more common \emph{classical}. There is no loss of generality, since it is known that state-variable planning domains can be translated to classical planning domains with at most a linear increase in size.

\subsection{State-Variable Representation of Planning Problems}

A detailed description of the planning formalism is outside the scope of this paper and the reader is referred to the relevant literature -- e.g.~\cite{ghallab_automated_2016} -- however in this section we briefly outline the main concepts for our setting.

Given a set of objects $B$, a \emph{state variable} is a symbol $v$ and an associated domain $Range(v)\subseteq B$.\footnote{In this context we restrict to state variables as $0$-ary functions.} The (finite) set of all the variables is denoted by $X$, a \emph{variable assignment} over $X$ maps each variable $v$ to one of the objects in $Range(v)$. In addition, we will consider a set of \emph{rigid relations} $R$ over $B$ that represent non-temporal relations among elements of the domain.

An \emph{atom} is a term of the form \true, $r(z_1,\ldots, z_n)$, or $v = z_1$; where $r\in R$, $v\in X$, and $z_i$ are either constants or parameters.\footnote{We use the term \emph{parameter} to denote mathematical variables, to avoid confusion with state variables.} A \emph{literal} is a positive or negated atom.

An \emph{action template} is a tuple $\alpha = (\pddlhead(\alpha), \pddlpre(\alpha), \pddleff(\alpha))$: $\pddlhead(\alpha)$ is an expression $act(z_1,\ldots,z_k)$ where $act$ is an (unique) action name and $z_i$ are parameters with an associated finite $Range(z_i)\subseteq B$, $\pddlpre(\alpha)$ is a set of literals (the preconditions), and $\pddleff(\alpha)$ is a set of assignments $v\leftarrow z$ where $z$ is a constant or parameter. All parameters in $\alpha$ must be included in $\pddlhead(\alpha)$. A \emph{state-variable action}  (or just action) is a ground template action where the parameters are substituted by constants from their corresponding ranges.

\begin{definition}\label{def:pddl:domain}
    A \emph{classical planning domain} is a triple $\Sigma = (S, A, \gamma)$ where
    \begin{itemize}
        \item $S$ the set of variable assignments over $X$;
        \item $A$ is the set of all actions;
        \item $\gamma : S'\times A' \rightarrow S$ with $S'\times A' = \{ (s,a)\in S\times A \mid \pddlpre(a) \text{ is satisfied in } S\}$ and
        \begin{align*}
            \gamma(s,a) = {}& \{ (v,w) \mid v \leftarrow w \in \pddleff(a) \}\\
            & \{ (v,w)\in s \mid \forall w'\in B\; v \leftarrow w' \not\in \pddleff(a) \}.
        \end{align*}
    \end{itemize}
\end{definition}

\begin{definition}
    A \emph{plan} is a finite sequence of actions $\pi = \langle a_1, \ldots, a_n \rangle$, or the empty plan $\langle \rangle$. A plan is \emph{applicable} to a state $s_0$ if there are states $s_1,\ldots, s_n$ s.t.\ $\gamma(s_{i-1}, a_i) = s_i$ (empty plans are always applicable).
    
    If $\pi$ is applicable to $s$, we extend the notation for $\gamma$ to include plans $\gamma(s, \pi)$: where $\gamma(s, \langle\rangle) = s$ and $\gamma(s, \pi.a) = \gamma(\gamma(s, \pi), a)$.\footnote{$\langle a_1,\ldots, a_n\rangle.a$ is defined as $\langle a_1,\ldots, a_n, a\rangle$.}
\end{definition}

\begin{definition}\label{def:pddl:plan}
    A \emph{state-variable planning problem} is a triple $P = (\Sigma, s_o, g)$ where $\Sigma$ is a planning domain, $s_0$ the initial state, and $g$ a set of ground literals called the goal.
    
    A \emph{solution} for $P$ is a plan $\pi$ s.t.\ $\gamma(s_0, \pi)$ satisfies $g$.
\end{definition}

To simplify the encoding introduced in Section~\ref{sec:pddl-encoding} we will consider an extended language for the precondition of action templates. This extended language enables the use of boolean combinations as well as atoms in which state variables can be used in place of constants, that is, $v = v'$, or $r(v_1,\ldots, v_n)$. Planning domains using the extended language can be encoded into classical planning domains by considering the disjunctive normal form of the original preconditions and introducing a new template for each conjunct.\footnote{This transformation might introduce an exponential blow-up of the domain due to the DNF transformation. However, this can be prevented in the overall conversion from \ournet by imposing the DNF restriction in the original model as well.}
The extended syntax concerning state variables can be dealt with by introducing a (new) parameter for each variable in the precondition appearing within a rigid relation or in the left hand side of an equality. This new parameter $z_v$ for the variable $v$ will be substituted to the original occurrence of $v$ and the new term $v = z_v$ is included in the precondition.

The above definitions can be re-casted for the extended language, and plans for the extended and corresponding classical planning domains can be related via a surjection projecting the ``duplicated'' template (introduced by the DNF conversion) into the original ones.

\subsection{Encoding \ournets in {\normalfont\pddllng}}
\label{sec:pddl-encoding}

The main idea behind the encoding of a \ournet $W$ in $\pddllng$, denoted with $\pddllng(W)$, is similar to the one presented in Section~\ref{sec:bcencoding}, where state variables are used to represent both places and \ournet variables, and actions represent transitions. The main differences lay in the fact that: (i) nondeterminism in the assignments must be encoded by means of grounding of action schemata; and (ii) the language for preconditions does not include disjunction, and thus more than a schema may correspond to a single transition.

Let $\dwnmodel = \langle\dmodel,\nmodel,\writef,\guardf\rangle$ be a \ournet as defined in Definition~\ref{def:our-net}, where $\nmodel=\tuple{P,T,F}$ and $\dmodel=\tuple{\V, \Delta, \domf, \ordf}$. $\V'\subseteq\V$ be the (finite) set of variables appearing in the model, and $\adom(v)$ be the \emph{active domain} of $v\in\V'$, that is, $\adom(v) = \bigcup_{t\in T} \writef(t)(v)$.
The encoding $\pddllng(\dwnmodel)$ is defined by introducing: a state variable for each $v\in \V'$, with domain corresponding to the domain of $v$ plus a special constant \nullv; and a boolean state variable for each place $p\in P$. To simplify the notation we will use the same constant names.
In addition to the state variables we introduce a rigid binary relation $\ordf$ to encode the partial ordering of Definition~\ref{def:dmodel}:
\begin{align}
    \ordf &= \bigcup_{\Delta_i \in \Delta} \{ (o, o')\in \Delta_i^2\mid o\leq_{\Delta_i} o'\}
\end{align}
and unary relations $wr_{t,v}$ to encode the range $\writef(t)(v)$ of Definition~\ref{def:our-net}:
\begin{equation}\label{eq:pddl:parameter:domain}
 \begin{aligned}
    wr_{t,v} &= \{ o \mid o\in\writef(t)(v) \} && \text{for all $t\in T$ and $v\in\V$ s.t.\ $\writef(t)(v)\neq\emptyset$} \\
    wr_{t,v} &= \{ \nullv \} && \text{for all $t\in T$ and $v\in\V$ s.t.\ $\writef(t)(v)=\emptyset$}
\end{aligned}
\end{equation}
   
Note that Lemma~\ref{lemma:dawnet:adom} guarantees the finiteness of $\ordf$, while $\writef(t)(v)$ are finite by definition.

Given a \ournet query $\Phi$  (Definition~\ref{def:dawnet:qlang}) we denote with $\pddlqlang{\Phi}$ its translation in the \pddllng preconditions, defined as:
\begin{equation}\label{def:enc:pddlqlang}
\begin{aligned}
    \pddlqlang{true} &= \true \\
    \pddlqlang{\deff(v)} &= \neg (v = \nullv ) \\
    \pddlqlang{v = t_2} &=  \neg (v = \nullv ) \land (v = t_2) \\
    \pddlqlang{t_1 \leq t_2} &=  \ordf(t_1, t_2) \\
    \pddlqlang{\neg \Phi_1} &=  \neg \pddlqlang{\Phi_1}\\
    \pddlqlang{\Phi_1 \land \Phi_2} &=  \pddlqlang{\Phi_1}\land \pddlqlang{\Phi_2} \\
    \end{aligned}  
\end{equation}

Analogously to the previous section we introduce an action template for each transition in $\dwnmodel$ where preconditions, in addition to the guard, include the marking of the network (token in the places). The selection of the actual value to be assigned to the updated variables is selected by means of action parameters (one for each variable).
Formally, 
	% \label{def:pddl:transition:template}
 let $t$ be a transition in $\dwnmodel$, and $\{{v_1}, \ldots, {v_k}\}$ the variables in the domain of $\writef(t)$; then $\pddllng(\dwnmodel)$ includes the action template $\alpha_t = (\pddlhead(\alpha_t), \pddlpre(\alpha_t), \pddleff(\alpha_t))$ defined as follows:
\begin{equation}\label{eq:pddl:transition:template}
\begin{aligned}
    \pddlhead(\alpha_t) &= t(z_{v_1}, \ldots, z_{v_k}) && \\
    \pddlpre(\alpha_t) &= \pddlqlang{\guardf(t)} \land \bigwedge_{v\in\{{v_1}, \ldots, {v_k}\}} wr_{t,v}(z_v) \land \bigwedge_{p\in \pres{t}} (p = \true)\\
    \pddleff(\alpha_t) &= \bigwedge_{v\in\{{v_1}, \ldots, {v_k}\}} (v = z_v )\land\bigwedge_{p\in \pres{t}\setminus\posts{t}} (p = \false) \land \bigwedge_{p\in\posts{t}} (p = \true) 
\end{aligned}
\end{equation}

Analogously to the \bclng case, the goal of the planning problem is described by the set of ground literals corresponding to the state in which all place fluents are \false\ but the one corresponding to the sink:
\begin{equation}\label{eq:pddl:final}	
\begin{aligned}
&sink = \true\\
&p=\false & & \text{for }p\in P \setminus \{sink\}
\end{aligned}
\end{equation}

\subsection{Correctness and completeness of the {\normalfont\pddllng} encoding}

To show that the transition systems underlying a \ournet $W$ and its encoding $\pddllng(W)$ are equivalent we define a bijective mapping between the states of the two systems and we show that for each valid firing in $W$ there is a ground action ``connecting'' the image of the states in $\gamma$ and vice-versa.

Let $(\mrk,\assign)$ be an arbitrary state of $\dwnmodel$, we define the mapping $\dwnetpddl$ to states of  $\pddllng(\dwnmodel)$ as:
\begin{equation}\label{def:dawnet:pddl:map}
    \begin{aligned}
        \dwnetpddl(\mrk,\assign) = {} & \{ v\mapsto v[\assign] \mid v\in\V, v[\assign]\neq v \} \cup {}\\
        & \{ v\mapsto \nullv \mid v\in\V, v[\assign] = v \} \cup {} \\
        & \{ p\mapsto \true \mid p\in P, \mrk(p)>0 \} \cup {} \\
        & \{ p\mapsto \false \mid p\in P, \mrk(p)=0 \}
    \end{aligned}
\end{equation}
Since we restrict to 1-safe nets, it is easy to show that $\dwnetpddl$ is a bijection; moreover:
\begin{lemma}\label{lemma:pddl:query}
  Let  $\dwnmodel$ be a \ournet model, for every state $(\mrk, \assign)$ of $\dwnmodel$ and query $\Phi\in\guardlang[\dmodel_{\dwnmodel}]$
    \begin{equation*}
        \dmodel, \assign \models \Phi \text{ iff }  \dwnetpddl(\mrk,\assign)\models \pddlqlang{\Phi}
    \end{equation*}
\end{lemma}

\begin{proof}
    The statement can be easily proved by induction on \guardlang[\dmodel_{\dwnmodel}] formulae and by using Equation~\eqref{def:enc:pddlqlang} for base cases.
\end{proof}

Now we can show that valid firings of \ournet corresponds to transitions in the planning domain

\begin{restatable}{lemma}{lemmaPDDMenc}\label{lemma:pddlenc}
    Let $\dwnmodel$ a \ournet model and $\pddllng(\dwnmodel)$ its encoding into a planning domain:
    \begin{enumerate}
        \item if $(\mrk,\assign)\fire{t}(\mrk',\assign')$ is a valid firing of $\dwnmodel$, then there is a ground action $a_t$ of $\alpha_t$ s.t.\ $(\dwnetpddl(\mrk,\assign), a_t, \dwnetpddl(\mrk',\assign'))\in \gamma$;
        \item if there is a ground action $a_t$ of $\alpha_t$ s.t.\ $(s, a_t, s')\in \gamma$, then $\pddldwnet(s)\fire{t}\pddldwnet(s')$ is a valid firing of $\dwnmodel$.
    \end{enumerate}
\end{restatable}

The proof is in~\ref{apx:pddl:enc}.

To compare the PDDL encoding of $\dwnmodel$ with the \ournet reachability graph $\RG_\dwnmodel=(T, \ol{S}, \ol{s}_0, \ol{\delta})$ (introduced in Definition~\ref{def:fin-reach-graph}) we consider the transition system $\TS_{\pddllng(\dwnmodel)}$ derived from the planning domain $\pddllng(\dwnmodel)$.

\begin{definition}\label{def:pddl:lts}
    Let $\dwnmodel$ a \ournet model. The transition system $\TS_{\pddllng(\dwnmodel)} = (T, S_{\pddllng}, s_0, \delta_{\pddllng})$ is defined as follows:
    \begin{compactitem}
        \item $s_0 = \dwnetpddl(\ol{s}_0)$
        \item $S_{\pddllng}$ is the minimum set that includes $s_0$ and if $s\in S_{\pddllng}$ and $(s, a_t, s')\in\gamma$ then $s'\in S_{\pddllng}$
        \item $\delta_{\pddllng} = \{ (s,t,s')\in S_{\pddllng}\times T \times S_{\pddllng}\mid (s, a_t, s')\in\gamma, a_t \text{ is a ground action of }\alpha_t \}$
    \end{compactitem}
\end{definition}

\begin{theorem}
    Let $\dwnmodel$ be a \ournet model, then $\RG_\dwnmodel$ and $\TS_{\pddllng(\dwnmodel)}$ are trace equivalent.
\end{theorem}

\begin{proof}\label{th:pddl:sound-comp}
	We define $\enc_{\fdts}(\cdot)$ as the restriction of $\dwnetpddl$ (as defined in Definition~\ref{def:dawnet:pddl:map}) to the set of states of $\RG_\dwnmodel$, and the identity w.r.t.\ the transitions. We show that $\enc_{\fdts}(\cdot)$ satisfies the properties of Definition~\ref{def:trace-equivalence} by induction on the length of the paths.
	
	The base case of a path of length 0 is satisfied because $s_0 = \enc_{\fdts}(\ol{s}_0) = \dwnetpddl(\ol{s}_0)$ is a state in $\TS_{\pddllng(\dwnmodel)}$ by Definition~\ref{def:pddl:lts}.
	
    For the inductive step we just need to consider the last element of the paths:
    \begin{enumerate}
    	\item If $s_0 \fire{t_1} s_1 \fire{t_2} \ldots s_{n-1}\fire{t} s_n$ is a path in $\RG_\dwnmodel$, then $s_{n-1}\fire{t} s_n$ is a valid firing, therefore by Lemma~\ref{lemma:pddlenc} there is a ground action $a_t$ of $\alpha_t$ s.t.\ $(\dwnetpddl(s_{n-1}), a_t, \dwnetpddl(s_n))\in \gamma$. By Definition~\ref{def:pddl:lts} $(\dwnetpddl(s_{n-1}), t, \dwnetpddl(s_n))$ is a transition of $\TS_{\pddllng(\dwnmodel)}$, and so is $(\enc_{\fdts}(s_{n-1}), \enc_{\fdts}(t), \enc_{\fdts}(s_n))$.
        \item If $\{(s_0,t_1,s_1),\ldots,(s_{n-1},t,s_n)\}\subseteq\delta_{\pddllng}$ then there is a ground action $a_t$ of $\alpha_t$ s.t.\ $(s_{n-1}, a_t, s_n)\in\gamma$; therefore $\pddldwnet(s_{n-1})\fire{t}\pddldwnet(s_n)$ is a valid firing of $\dwnmodel$ by Lemma~\ref{lemma:pddlenc}. So $\pddldwnet(s_{n-1})$ and $\pddldwnet(s_n)$ are states of $\RG_\dwnmodel$ and $\enc_{\fdts}(\pddldwnet(s_{i})) = s_{i}$ by construction.
    \end{enumerate}
\end{proof}

The trace equivalence demonstrated by Theorem~\ref{th:pddl:sound-comp} enables the use of classical planning tools to verify the existence of a plan satisfying the termination conditions expressed in Equation~\ref{eq:pddl:final}.
Moreover, it is immediate to verify that the set of states as characterised by the goal in Equation~\ref{eq:pddl:final} corresponds to all and only the final states of the \ournet model mapped via $\enc_{\fdts}(\cdot)$.

%In the case that the model is a trace workflow $\Wtau$, then the plan would exhibit a valid completion of the trace $\tau$; while the failure of finding such a plan would imply that the trace do not admit any completion.

%% file: encoding-nuSMV.tex
%!TEX root = ./main.tex

\section{The encoding in \nuxmv}
\label{sect:nuxmv:enc}

Model checking~\cite{mc-book} tools take as input a specification of system executions and a temporal property to be checked on such executions. The output is ``true'' when the system satisfies the property, otherwise ``false'' is returned as well as a \emph{counterexample} showing one (among the possibly many) evolution that falsifies the formula. In particular, reachability properties can be easily expressed as liveness formulas of the kind: ``Eventually one of the states satisfying some property is visited.''
In \nuxmv, the system is specified in an input language inspired by \textsc{smv}~\cite{McMillan-book} and the properties in a temporal language, such as linear-time temporal logic (\textsc{ltl}, see~\cite{mc-book} Chapter 5, for an introduction).

Let $W$ be a \ournet and $\RG_W$ its reachability graph. We show how to encode $\RG_W$ in the \nuxmv input language and the reachability of a set of goal states as a \textsc{ltl} formula such that, when processed by \nuxmv, an execution reaching one of the final states is generated in the form of a counterexample, if any.

% \todo[inline]{Should I keep this?}
% When the \ournet is a trace workflow $\Wtau$, then the counterexample is a possible completion of $\tau$.

\subsection{The \nuxmv input language}

Many of the available tools for traditional model checking, including \nuxmv, deal with infinite-executions systems, and hence they require systems transition relation to be left-total. On the contrary, our 
%\oursystem
 models are based on workflow-nets which by definition represent process which usually terminates. Also, in \nuxmv transitions have no label. % Thus, our \nuxmv encoding must meet this requisite while the $RG$ as in Definition~\ref{fin-reach-graph} obviously not, thus making the bisimulation relation definition between the two convoluted.
For this reason, we adapt the reachability graph so as to have infinite paths and no transition labels, and we call such a modified reachability graph $\RG_W^\sim$. This kind of transformations are customary using tools requiring infinite-executions to describe finite-execution ones. Despite $\RG_W^\sim$ being semantically different from $\RG_W$, it is also possible to tweak the formal properties to be verified in such a way that the adaptation is sound and complete. From here on we abstract from such technical subtleties and refer to~\cite{DeGiacomoDMM14} for the details.

Intuitively, $\RG_W^\sim$ is obtained from $\RG_W$ by: \begin{inparaenum}[\it(i)] \item moving transition labels in the source state; \item adding a fresh self-looping (with transition $\fine$) state $s_\epsilon$ and \item adding a ($\last$) transition from every sink state of $\RG_W$ to $s_\epsilon$, as the following definition formalizes.\end{inparaenum}

\begin{definition} \label{def:inf-reach-graph}
Let $\RG_W=(T, \ol{S}, \ol{s}_0, \ol{\delta})$ be a reachability graph of a \ournet. We build a transition system $\RG_W^\sim = (T \cup \set{\last, \fine}, S, S_0, \delta)$ as follows:
\begin{itemize}
\item the set of states are triples $(\mrk, \assign, t)$ with $t \in T \cup \set{\last, \fine}$ where we distinguish the special state $s_{\epsilon} = (\mrk_\epsilon, \assign_\epsilon, \fine)$ where $\mrk_\epsilon$ and $\assign_\epsilon$ are arbitrary;
\item set $S_0$ is built as follows: if there exists $(\mrk_0, \assign_0) \fire{t} (\mrk', \assign') \in \ol{\delta}$ then $S_0 = \set{ (\mrk_0, \assign_0, t) \mid \exists t, \mrk', \assign'.(\mrk_0, \assign_0) \fire{t} (\mrk', \assign') \in \ol{\delta} }$, otherwise $S_0 = \set{(\mrk_0, \assign_0, \last)}$.
\item $S$ and $\delta \subseteq S \times S$ are the least sets s.t.:
\item $S_0 \cup \set{s_\epsilon} \subseteq S$;
\item $s_{\epsilon} \fire{} s_\epsilon \in \delta$;
\item if $(\mrk, \assign) \fire{t} (\mrk', \assign') \in \ol{\delta}$ then:
\begin{itemize}
\item if $(\mrk', \assign'), \fire{t'}, (\mrk'', \assign'') \in \ol{\delta}$ then $s=(\mrk, \assign, t) \fire{} s'=(\mrk', \assign', t') \in \delta$ and $s, s' \in S$;
\item otherwise $s=(\mrk, \assign, t) \fire{} s'=(\mrk', \assign', \last) \in \delta$, $(s', s_\epsilon) \in \delta$ and $s, s' \in S$;
\end{itemize}
\item if $\ol{\delta} = \emptyset$ then $(\mrk_0, \assign_0, \last) \fire{} s_\epsilon \in \delta$ (notice that if $\ol{\delta} = \emptyset$ then $S_0 = \set{(\mrk_0, \assign_0, \last)}$).
\end{itemize}
\end{definition}

% \begin{theorem}
% Let $\varphi$ be a formula and let $\tr{}$ be the translation as in \cite{DeGiacomoMDM14}, but with $\hat{tr} \not = \fine$ instead of $end$. Then $\ol{\G} \models_f \varphi$ iff $\G \models \tr{\varphi}$ where $\models_f$ is the model checking problem of testing wether $s_0$ belongs to the set of states satisfying $\varphi$ (\emph{all} linear executions starting from $s_0$ satisfies $\varphi$), while $\models$ is the model checking problem of testing wether $S_0$ is a subset of the set of states satisfying $\tr{\varphi}$.
% \end{theorem}
% \begin{proof}
% Immediate consequence of the results in \cite{DeGiacomoMDM14} by noticing that each finite path $\pi_f$ in $\ol{\G}$ have been transformed in an infinite path $\pi_f\cdot \set{\fine}^{\omega}$ in $\G$.
% \end{proof}

We now describe from a very high-level perspective the syntax and semantics of \nuxmv, i.e., how a transition system $\TS_{\smvlng(m)}$ is built given syntactic description of a so-called \emph{module} $m$. In general, \nuxmv allows for several modules descriptions, but in our case one it is enough to represent $\RG_W^\sim$.
%%
% The tool support the synchronous evolutions of several \emph{modules}, each one defining a specific process. In our case, we only have the MAIN module, i.e., $\G$.

A module is specified by means of several sections and, in particular:
\begin{compactitem}
\item A VAR section where a set of variables and their domains are defined. Each complete assignment of values to variables is a state of the module.
\item A INIT section where a formula representing the initial state(s) of the system is defined.
  \item A TRANS section where the transition relation between states is specified in a declarative way by means of a list of ordered rules $\tup{r_1, \ldots, r_n}$ in a switch-case block\footnote{There are several equivalent ways to describe a transition relation in \nuxmv, we just choose one.}. Each rule $r_i$ has the form $(pre_i, post_i)$, where $pre_i$ is (a formula representing) a precondition on the current state and $post_i$ is (a formula representing) a postcondition on the next state (and possibly the current state). In \nuxmv, postconditions are expressed by means of primed variables, namely $\hat{v}'= \mathit{value}$ means that in the next state variable $\hat{v}$ is assigned to $\mathit{value}$. Intuitively, state $q'$ is a successor of $q$ iff $i$ is the smallest index such that $q$ models $pre_i$ and $(q, q')$ models $post_i$. With slight notational abuse we write $q \models \varphi$ when the current state (recall that a state is an assignment of values to variables) satisfies a formula containing only unprimed variables and we write $(q, q') \models \varphi'$ when the pair current state, next state satisfies a formula containing both primed and unprimed variables.
  \end{compactitem}

  \begin{definition} \label{def:nuxmv-ts}
Let $m$ be a \nuxmv module and $\stvar$ be the set of variables in its VAR section. Module $\mathit{m}$ defines a transition system $\TS_{\smvlng(m)}=(\stvar, Q, Q_0, \rho)$ where:
\begin{itemize}
\item $Q$ is the set of states, each representing a \emph{total} assignment of domain values to variables;
\item the initial states $Q_0$ are those satisfying the formula in the INIT section;
\item states $Q$ and transition function $\rho \subseteq Q \times Q$ are defined by mutual induction as follows:
\begin{itemize}
\item $Q_0 \subseteq Q$;
\item if $q \in Q$ and $i$ is the smallest index such that $q \models pre_i$ (such an $i$ is required to exist by \nuxmv), then $(q, q') \in \rho$ and $q'$ is such that $(q, q') \models post_i$.
\end{itemize}
\end{itemize}
\end{definition}

\subsection{Encoding of \ournet in \nuxmv}

We present our encoding of $W$ in \nuxmv for each section of the module $\smvlng(W)$.
  \paragraph{VAR section} We have:
\begin{compactitem}
\item one boolean variable $\hat{p} \in \hat{P}$ for each place $p \in P$;
\item one $\hat{tr}$ variable which ranges over $T \cup \set{\last, \fine}$;
\item one variable $\hat{v} \in \hat{\V}$ for each $v \in \V$ ranging over $\domf(v) \cup \set{\udef}$ where $\udef$ it is a special value not contained in any of the $\domf(v)$.
\end{compactitem}

Let $\enc_{\smvts}(\mrk, \assign)$ be the encoding for \ournet markings and assignments to formulas in \nuxmv of the kind:
\[
\bigwedge_{p \in P}\hat{p}= \set{\true, \false} \bigwedge_{v \in \V}\hat{v_i}=d 
\]
where $\hat{p}=\true$ iff $\mrk(p)=1$ and $\hat{p}=\false$ otherwise; $\hat{v_i}=d$ iff $\assign(v)$ is defined; and $\assign(v)=d$ or $d=\udef$ if $\assign(v)$ is undefined.
Besides, let $\enc_{\smvts}(t)$ be the encoding of \ournet transition $t$ to the \nuxmv formula $\hat{tr} = t$, for $t \in T \cup \set{\last, \fine}$.

\begin{remark} \label{rem:enc-bijection}
The pair of functions $\enc_{\smvts}(\mrk, \assign)$ and $\enc_{\smvts}(t)$ are a bijection between states $S$ of $\RG_W^\sim$ and states $Q$ of $\TS_{\smvlng(W)}$.
\end{remark}

\paragraph{INIT section}
We have:
%The init section is the following formula:
\[
\begin{array}{l}
\enc_{\smvts} (\mrk_0, \assign_0) \land {} \\
\hat{tr} \not = \fine \land{} \\
\bigwedge_{t \in T} \hat{tr}=t \ra \currpre{t} \land{} \\
\hat{tr}=\last \ra (\bigwedge_{t \in T} \neg \currpre{t})
\end{array}
\]
where $\currpre{t}$ is a formula representing the preconditions for executing $t$ according to the semantics of the \ournet: it thus consists of guard $\guardf(t)$ as well as preconditions on the input set $\pres{t}$.

\paragraph{TRANS section} % We define a single CASE block comprised by an \emph{ordered} list of rules $\tup{r_1, \ldots, r_n}$ each of which has the form $(pre_i, post_i)$ where $pre_i$ is a boolean expression mentioning the current value of $\hat{\V}$ variables and $post_i$ is a boolean expression mentioning the current value of $\hat{\V}$ variables and possibly their value in the successive state trough the $\nxt{\cdot}$ operator. 
We have:
\begin{compactitem}
%%%%%%%%%%%%%
\item one rule for each $t \in T$ which has:
\[pre := \enc_{\smvts}(t) %% \land \forall \overline{t} \in \pres{t}. \overline{t}=\true \land \guardf(t)=\true
\]
\[ \begin{array}{ll}
post := & \forall p \in \pres{t}\setminus\posts{t}. \nxt{\hat{p}} = \false \land {} \\
& \forall p \in \posts{t}\setminus\pres{t}. \nxt{\hat{p}} = \true \land {} \\
& \nxt{\hat{p}}=\hat{p} \text{ otherwise } \land {} \\
& \forall v \in \domf(\writef(t)) \land \writef(t)(v) \not = \emptyset. \nxt{\hat{v}} \in \writef(t)(v) \land {} \\
& \forall v \in \domf(\writef(t)) \land \writef(t)(v) = \emptyset. \nxt{\hat{v}}=\udef \land {}\\
& \nxt{\hat{v}}=\hat{v} \text{ otherwise} \land {} \\
& \bigwedge_{t \in T} \nxt{\hat{tr}}=t \ra \nextpre{t} \land{} \\
& \nxt{\hat{tr}}=\last \ra (\bigwedge_{t \in T} \neg \nextpre{t}) \land{} \\
& \nxt{\hat{tr}}=\fine \ra \hat{tr}=\last
\end{array}
\]
where $\nextpre{t}$ is likewise formula $\currpre{t}$ but with all its variables primed (thus referring to the next state). Intuitively, the postcondition contains: \begin{inparaenum} [\it(i)] \item a condition on the (next) values of pre- and post-sets of $t$; \item a condition on the (next) values of the variables according to function $\writef$ of the \ournet and \item conditions on the (next) value of variable $\hat{tr}$: next value of $\hat{tr}$ is $t$, only if the (next) state is consistent with the preconditions for executing $t$. \end{inparaenum}

%%%%%%%%%%%%
\item a rule for $\hat{tr}=\last$:
\[
pre := \enc_{\smvts}(\last)
\]

\[\begin{array}{ll}
post := 
& \nxt{\hat{tr}}=\fine \land {} \\
& \nxt{\enc_{\smvts}(\mrk_\epsilon, \assign_\epsilon)}
\end{array}
\]

\item a rule for looping in the ended state with $\hat{tr}=\fine$:
\[
pre := \enc_{\smvts}(\fine)
\]
%%%%%
\[\begin{array}{ll}
post := 
& \nxt{\hat{tr}}=\fine \land {} \\
& \nxt{\enc_{\smvts}(\mrk_\epsilon, \assign_\epsilon)}
\end{array}
\]
\end{compactitem}

We notice that since $\bigcup pre_i$ form a partition for all possible values of $\hat{tr}$ we have that for each assignment to $\stvar$, i.e., state $q$, there exists a \emph{unique} $i$ such that $q \models pre_i$. In other words, the \emph{ordering} of rules does not matter.

\subsection{Correctness and completeness of the \nuxmv encoding}

As introduced in Section~\ref{sec:dawnet-fts}, we prove soundness and completeness of the encoding by proving trace equivalence between $\RG_W^\sim$ and the transition system $\TS_{\smvlng(W)}$ generated by module $\smvlng(W)$. A standard technique to prove trace equivalence amounts, in fact, to provide a bisimulation relation between states of $\RG_W^\sim$ and $\TS_{\smvlng(W)}$ as the latter implies the former \cite{mc-book}.
A \emph{bisimulation relation} has a co-inductive definition which says that states of the two transition systems are bisimilar if:
\begin{inparaenum}[\it(i)]
\item they satisfy the same \emph{local} properties;
  \item every transition from one state can be mimicked by the other state and their arrival states are again in bisimulation (and vice-versa).
\end{inparaenum}
If every initial states of $\RG_W^\sim$ is in bisimulation with an initial state of $\TS_{\smvlng(W)}$, then they are trace equivalent, as any trace start from an initial state.

\begin{definition} \label{def:bisim}
$\RG_W^\sim=(T, S, S_0, \delta)$ be the reachability graph of a \ournet and $\TS_{\smvlng(W)}=(\stvar, Q, q_0, \rho)$ be the \nuxmv transition system of its encoding. We define \emph{bisimulation relation} $B \subseteq S \times Q$ as follows: $(s=(M, \assign, t), q) \in B$ implies:
\begin{enumerate}
\item $q = \enc_{\smvts}(\mrk, \assign) \land \enc_{\smvts}(t)$;
\item for all $ s \fire{} s \in \delta$ there exists $q' \in Q$ such that $q \fire{} q' \in \rho$ and $(s', q') \in B$; and
\item for all $q \fire{} q \in \rho$ there exists $s' \in S$ such that $s \fire{} s' \in \delta$ and $(s', q') \in B$.
\end{enumerate}
\end{definition}

We prove that there exists a bisimulation between $\RG_W^\sim$ and $\TS_{\smvlng(W)}$ by showing one.

\begin{restatable}{theorem}{thNuxmvSoundComp}\label{th:nuxmv-sound-comp}
%%\begin{theorem}
  Let $R \subseteq S \times Q$ be such that $((s=(M, \assign, t), q) \in R)$ iff $q = \enc_{\smvts}(\mrk, \assign) \land \enc_{\smvts}(t)$. Then:
  \begin{enumerate}
  \item $R$ is a bisimulation relation;
  \item for each $s_0 \in S_0$ there exists $q_0 \in Q_0$ such that $(s_0, q_0) \in R$ and vice versa.
    \end{enumerate}
  \end{restatable}
  % \end{theorem}

\begin{proof}
We first prove \emph{1}.
Let $s'=(\mrk', \assign', t')$ be such that $s \fire{} s' \in \delta$, then we have to prove that there exists a $q'$ such that $q \fire{} q' \in \rho$ and $(s', q') \in R$. We have three cases:
\begin{compactitem}
\item $t \in T$: then, by inspecting the corresponding rule in the TRANS section, it is immediate to see that $q$ and $s$ are such that $\mrk'(p) \lra \hat{p}$ (such values are deterministic). Let us now consider the value of a generic variable $v \in \V$ assigned nondeterministically by $\assign'$. We have again three cases:\begin{inparaenum}[(i)] \item $v \in \domf(\writef(t))$ and $\writef(t)(v) \not = \emptyset$, hence its new value will be in $\writef(t)(v)$; \item $\domf(\writef(t))$ and $\writef(t)(v) = \emptyset$, hence its new value will be $\udef$; or \item $v \not \in \domf(\writef(t))$ and its new value will be the same as the old one\end{inparaenum}. Such cases are coded in conjunction in the rule: let us assume the first one applies with value $v'=c$, by Definition~\ref{def:nuxmv-ts} then there exists $q'$ such that $(q, q') \models post_i$ and $q'(v)=c$; for the other two cases the choice is deterministic, either $q'(v)=\udef$ or $q'(v)=q(v)$. It follows that there exists a $q'$ that agrees on $\enc_{\smvts}(\mrk', \assign')$. It remains to show that there exists a $q'$ such that also $q'(\hat{tr})=t'$: if $t' \in T$, then it means that $\mrk'$ and $\assign'$ satisfy the preconditions of the firing of $t'$, meaning that $t \ra \nextpre{t}$ is satisfied, hence there exists such a $q'$ (there exists at least one, the choice of $t'$ is nondeterministic); if $t'=\last$ then, from Definition~\ref{def:inf-reach-graph} we have that no $t'', \mrk'', \assign''$ such that $(\mrk', \assign') \fire{t''} (\mrk'', \assign'')$ exist, i.e., $\bigwedge_{t \in T} \neg \nextpre{t}$ holds, hence $q'(\hat{tr})=\last$. Notice that we assumed $t \in T$, hence it is never the case that $t'=\fine$. Since $q'$ agrees with $s'$ on the markings, the values of variables and the transition, it follows that $(s', q') \in R$.

\item $t=\last$, then from Definition~\ref{def:inf-reach-graph} we have that $s'=s_\epsilon$ and $t'=\fine$. Since $q(\hat{tr})=\last$, then rule $t=\last$ applies in the encoding and from its definition, $(s', q') \in R$ immediately follows.

\item $t=\fine$, then from Definition~\ref{def:inf-reach-graph} we have that $s'=s_\epsilon$ and $t'=\last$. Since $q(\hat{tr})=\fine$, then rule $t=\fine$ applies in the encoding and from its definition, $(s', q') \in R$ immediately follows.
\end{compactitem}

Now, let $q'$ such that $q \fire{} q' \in \rho$. We have to prove that there exists a $s' = (\mrk', \assign', t')$ such that $s \fire{} s' \in \delta$ and $q' = \enc_{\smvts}(\mrk', t') \land \enc_{smvts}(t)$.
We have three cases: $t \in T$, $t = \last$ and $t = \fine$. Analogous considerations \rtext{to}
of
 the previous case hold, as the rules in the encoding are not only sound, but also complete.

We now show \emph{2}. Given Remark~\ref{rem:enc-bijection}, it is enough to show that $(s, q) \in R$ implies $s \in S_0 \leftrightarrow q \in Q_0$.
We first prove the left to right direction, i.e., that $(s=(\mrk_0, \assign_0, t), q) \in R$ implies $q \in Q_0$, namely that formula $\enc_{\smvts}(\mrk_0, \assign_0) \land \enc_{\smvts}(t)$ satisfies the first conjunct of the INIT section. By definition of INIT, the first conjunct is satisfied, thus it remains to prove $\enc_{\smvts}(t)$. By Definition~\ref{def:fin-reach-graph} and Definition~\ref{def:inf-reach-graph} either (1) $t \in T$ or (2) $t=\last$. If (1) then there is no valid firing in $(\mrk, \assign)$, which entails that $\bigwedge_{t\in T} \neg \currpre{t}$ and if (2), then $t$ can fire in $(\mrk_0, \assign_0)$ hence its preconditions are satisfied, which means $\enc_{\smvts}(\mrk_0, \assign_0)$ satisfies $\currpre{t}$. In both cases the formula in the INIT section is satisfied and $q \in Q_0$.

The right to left direction in analogous by noticing that $\currpre{}$ trivially encodes the preconditions for firing $t$.
\end{proof}

We conclude the section by observing that, by virtue of Remark~\ref{rem:enc-bijection}, the formula $\Phi := \hat{end} = \true \bigwedge_{p \in P\setminus\set{end}}\hat{p}= \false$ captures all and only states in $\TS_{\smvlng(W)}$ that correspond to final states in the DAW-net.
Therefore, when \nuxmv checks the \textsc{ltl} formula: ``a $\Phi$-state is never visited'', it returns ``true'' if no final state can be reached and ``false'' and a path to a final state otherwise.

% \begin{corollary}
% Lemma~\ref{lemma:init-states-bisim} guarantees that each initial state of the reachability graph is simulated by a state in the nuXmv transition system and vice versa, hence the two transition systems are bisimilar which entails that for every temporal formula (also reachability formulas) $\phi$, $\phi$ holds in one if and only if it holds in the other.
% \end{corollary}

% \endinput

%%% Local Variables:
%%% mode: latex
%%% TeX-master: "main.tex"
%%% save-place: t
%%% End:

%% file: repair.tex
\section{Trace completion} % (fold)
\label{sec:trace_repair}

%
% Event logs are crucial in the analysis of processes, and they can be exploited to verify conformance, or to enhance and optimise models. In many cases, however, these logs are subject to data quality problems, resulting in \emph{incorrect} or \emph{missing} events.
In order to evaluate the proposed encodings and the scalability of the three chosen solvers,
%Planning systems,
 in this paper we focus on the problem of \textbf{repairing execution traces that contain missing entries} (hereafter shortened in trace completion).  
The need for trace completion is motivated in depth in~\cite{Rogge-Solti2013}, where missing entities are described as a frequent cause of low data quality in event logs. 

The starting point of approaches to trace completion are (partial) \emph{execution traces} and the knowledge captured in \emph{process models}. To illustrate the idea of trace completion, and why this problem becomes interesting in data-aware business processes, 	 let us consider the simple process model of the \emph{Loan Request} (LR) depicted in  Figure~\ref{fig:imgs_SampleWFnet}, page~\pageref{fig:imgs_SampleWFnet}.

A sample trace\footnote{In this work we are interested in (reconstructing) the ordered sequence of events that constitutes a trace, and thus we ignore timestamps as attributes. In the illustrative examples we present the events in a trace ordered according to their execution time.}
 that logs the execution of a LR instance is:
\begin{equation}
\label{eq:trace-complete-nodata1}
\{T1,T3, T5, T7, T9, T10, T11, T12\}.
\end{equation}
%\todocginline{Controllate la footnote sui timestamps}
In this trace all the executed activities have been explicitly logged. Consider now the (partial) execution trace 
\begin{equation}
\label{eq:trace-nodata}
\{T3, T7\}.
\end{equation}
which only records the execution of tasks \emph{T3} and \emph{T7}. 
By aligning the trace to the model using the techniques presented in~\cite{Rogge-Solti2013,Di-Francescomarino-C.:2015aa}, we can exploit the events stored in the trace and the ordering of activities (i.e., control flow) specified in the model to reconstruct two possible completions: one is the trace in \eqref{eq:trace-complete-nodata1}; the other is: 
% \begin{align}
    %     \label{eq:trace-complete-nodata2}
$\set{T1, T3, T5, T7, T9, T11, T10, T12}$
% \end{align}
which only differs from \eqref{eq:trace-complete-nodata1} for the ordering of the parallel activities. 

Consider now a different scenario in which the partial trace reduces to \{$T7$\}. In this case, by using the control flow in Figure~\ref{fig:imgs_SampleWFnet} we are not able to reconstruct whether the loan is a student loan or a worker loan. This increases the number of possible completions and therefore lowers the usefulness of trace repair. Assume nonetheless that the event log conforms to the XES standard and stores some observed data attached to $T7$ (enclosed in square brackets): 
%\todo[inline]{Warning: we don't handle ``observations''.}
\begin{equation}
\label{eq:trace-data}
\{T7[request=\cvalue{60k},loan=\cvalue{50k}]\}
\end{equation}
Since the process model is able to specify how transitions can read and write variables, and furthermore some constraints on how they do it, the scenario changes completely. In our case, we know that transition \emph{T4} is empowered with the ability to write the variable $request$ with a value smaller or equal than $\cvalue{30k}$ (being this the maximum amount of a student loan). Using this fact, and the fact that the request examined by \emph{T7} is greater than $\cvalue{30k}$, we can understand that the execution trace has chosen the path of the worker loan. Moreover, since the model specifies that variable \emph{loanType} is written during the execution of \emph{T1}, when the applicant chooses the type of loan she is interested to, we are able to infer that \emph{T1} sets variable \emph{loanType} to $\cvalue{w}$. 

\begin{figure}[t]
  \centering
    \includegraphics[width=.8\textwidth]{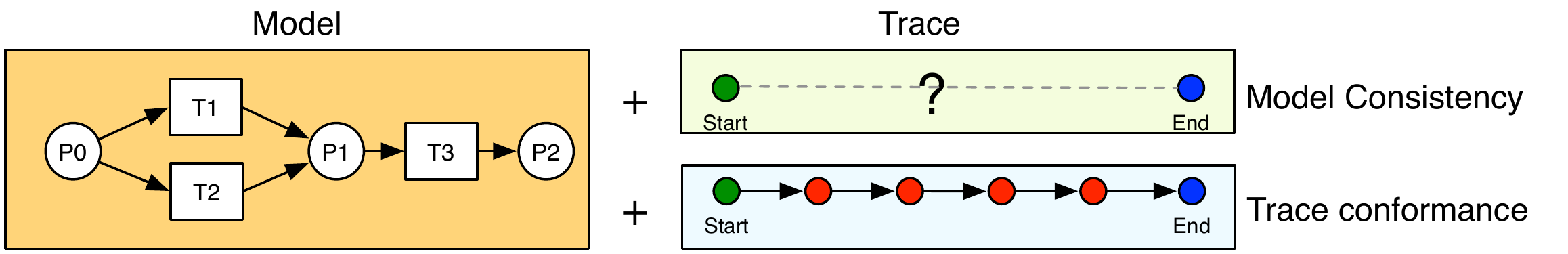}
  \caption{Corner-case incomplete execution traces.}
  \label{fig:imgs_Incompleteness}
\end{figure}
This example, besides illustrating the idea of trace completion, demonstrates why taking into account data (both in the process model and in the execution trace) is essential to accomplish this task.  
Also, as already observed in \cite{Chesanietal:2018}, by considering corner cases of partial traces, such as empty traces and complete traces, the ability to produce a complete trace compliant with the process model can be also exploited to compute model consistency and trace conformance,
%compliance,
 respectively (see Figure~\ref{fig:imgs_Incompleteness}).

% subsection trace_repair (end)

% subsection beyond_trace_repair (end)

%%% Local Variables:
%%% mode: latex
%%% TeX-master: "main.tex"
%%% save-place: t
%%% End:

%% file: encoding-traces.tex
\subsection{Trace completion as reachability}
\label{sec:encoding:traces}

Having introduced the problem of trace completion we now show that we can reduce it to a reachability problem in a modified \ournet that takes into account both the original model and the given (partial) trace. Here we formalise the completion problem and its encoding into reachability. Full details and proofs are contained in~\ref{apx:encoding:traces}.

A \emph{trace} is a sequence of observed \emph{events}, each one associated to the transition it refers to. Events can also be associated to data payloads in the form of attribute/variable-value pairs. Among these variables we can distinguish those updated by its execution (see Definition~\ref{def:our-net}) and those observed. Executions \eqref{eq:trace-complete-nodata1}, \eqref{eq:trace-nodata}, and \eqref{eq:trace-data} at page \pageref{eq:trace-data} provide examples of traces without and with data payloads for our running example.\footnote{We remark the conceptual and practical difference between a case and a trace: a case is a semantic object intuitively representing an execution/behaviour of the business process model while a trace is asyntactic object that comes from the real world of which we would test the soundness by providing a possible repair.}

%%
%%%%%%%%%%
\begin{definition}[Trace] %\footnote{Differente da BPM}
Let $W = \tuple{\dmodel, \nmodel, \writef, \guardf}$ be a \ournet.
An event of $W$ is a tuple \tuple{t,w, w^d} where $t\in T$ is a transition, $w \in \bigcup_{\varset'\subseteq\varset} ({\varset'}\mapsto\domf(\varset))$ is a partial function that represents the variables written or observed by the execution of $t$, and $w^d\subseteq\varset$ the set of variables deleted (undefined) or observed to be undefined by the execution of $t$. Obviously, $w^d$ and the domain of $w$ do not share any variable.

A trace of $W$ is a finite sequence of events
$\tau = (e_1,\ldots, e_n)$. In the following we indicate the $i$-th
event of $\tau$ as $\tstep{\tau}{i}$. Given a set of transitions $T$, the set of traces is inductively
defined as follows:
\begin{compactitem}
\item the empty trace is a trace;
\item if $\tau$ is a trace and $e$ an event, then $\tau \cdot e$
  is a trace.
\end{compactitem}
\end{definition}
%%%%%%%%%

Notice that the definition above constrains the transitions neither to be complete nor correctly ordered w.r.t.\ $W$. What we would like to check is indeed if $\tau$ can be completed so as to represent a valid case of $W$.
Intuitively, a trace $\tau$ is \emph{compliant} with a case $C$ of a \ournet $W$ if $C$ contains all the occurrences of the transitions observed in $\tau$
(with the corresponding variable updates) in the right order, as the following definition formalises.

\begin{definition}[Trace Compliance] \label{def:compliance} An event \tuple{t',w, w^d} is \emph{compliant} with a (valid) firing $(\mrk,\assign)\fire{t}(\mrk',\assign')$ iff $t = t'$, $dom(w)\subseteq dom(\assign') \subseteq \varset \setminus w^d$, and for all $v\in dom(w)$ $w(v) = \assign'(v)$.

A trace $\tau = (e_1,\ldots, e_\ell)$, $\ell > 0$, is \emph{compliant} with a sequence of valid firings $$(\mrk_0,\assign_0)\fire{t_1}(\mrk_1,\assign_1)\ldots (\mrk_{k-1},\assign_{k-1})\fire{t_k}(\mrk_{k},\assign_{k})$$ iff there is an injective mapping $\gamma$ between $[1\ldots \ell]$ and $[1\ldots k]$ such that:\footnote{If the trace is empty then $\ell=0$ and $\gamma$ is empty.}
  \begin{align}
    \forall i,j \text{ s.t. } 1\leq i < j \leq \ell \quad \gamma(i) < \gamma(j)\\
    \forall i \text{ s.t. } 1\leq i \leq \ell \quad  e_i \text{ is compliant with } (\mrk_{\gamma(i-1)},\assign_{\gamma(i-1)})\fire{t_{\gamma(i)}}(\mrk_{\gamma(i)},\assign_{\gamma(i)})
  \end{align}
When $\ell = 0$ we impose that $\tau$ is compliant with any sequence of valid firings.
  
We say that a trace $\tau$ is compliant with a \ournet $W$ if there exists a case $C$ of $W$ such that $\tau$ is compliant with $C$.
\end{definition}

Although the previous definition is general we notice that, since the last state of a case $C$ must be a final state, $\tau$ being compliant with $W$ means that it can be completed so as to represent a \emph{successful} execution of the process $W$.

%%
% f the process does not have this structure, it can be reduced to it by replacing the original start/sink places with a normal place and introducing new transitions as illustrated in the left hand side of Figure~\ref{fig:trace-compilation} by arrows labeled with (1). This change would not modify the behaviour of the net: any sequence of firing valid for the original workflow can be extended by the firing of the additional transitions and viceversa.
%%

\begin{figure*}[t]
 \begin{center}
  \includegraphics[width=.7\linewidth]{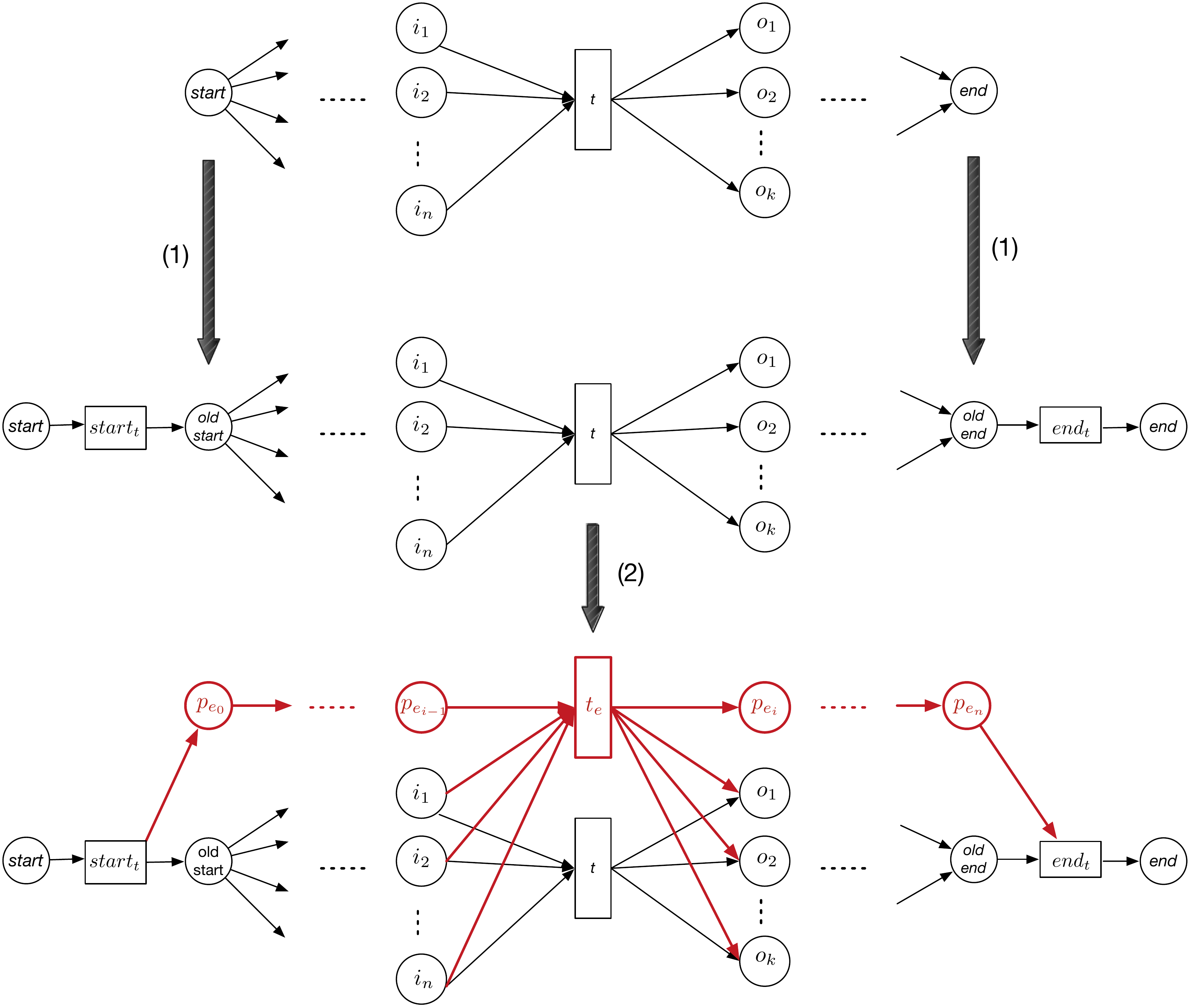}
\end{center}
 \caption{Outline of the trace ``injection''.}
 \label{fig:trace-compilation}
\end{figure*}

To simplify the presentation we assume (without loss of generality) that \ournet models start with the special transition $start_t$ and terminate with the special transition $end_t$. Both have a single input and output places: $start_t$ input is the start place, and $end_t$ output place is the sink. Both the transitions have empty guards (always satisfiable) and don't modify variables; moreover, start and sink places shouldn't have any incoming and outgoing edge respectively.
Every process can be reduced to such a structure as informally illustrated in the top part of of Figure~\ref{fig:trace-compilation} by the arrows labeled with (1). Note that this change would not modify the behaviour of the net: any sequence of firing valid for the original net can be extended by the firing of the additional transitions and vice versa.

We now illustrate the main idea behind our approach by means of the bottom part of Figure~\ref{fig:trace-compilation}: we consider the observed events as additional transitions (in red) and we suitably ``inject'' them in the original \ournet. By doing so, we obtain a new model where we aim at forcing tokens to activate the red transitions when events are observed in the trace. When, instead, there is no red counterpart, i.e., there is missing information in the trace, the tokens are free to move in the black part of the model.
%By doing so, we obtain a new model where, intuitively, tokens are forced to activate the red transitions of \ournet, as those represent happened events, while we leave them to move in the black part, that is the original process, when there is no red counterpart, meaning that there is missing information in the trace.
%%

More precisely, for each event $e$ corresponding to the execution of a transition $t$ with some data payload, we introduce a new transition $t_e$ in the model such that:
\begin{itemize}
  \item $t_e$ is placed in exclusive or with the original transition $t$;
  \item $t_e$ includes an additional input place connected to the preceding event and an additional output place which connects it to the next event;
  \item $\guardf(t_e)$ is a conjunction of the original $\guardf(t)$ with the conditions that state that the variables in the event data payload that are not modified by the transition are equal to the observed ones;
  \item $\writef(t_e)$ updates the values corresponding the variables updated by the event payload, i.e.\ if the event sets the value of $v$ to $d$, then $\writef(t_e)(v) = \set{d}$; if the event deletes the variable $v$, then $\writef(t_e)(v) = \emptyset$.
  %i.e.\ if the event wrote value $d$ in $v$ then $\writef(t_e)(v) = \set{d}$, while for deleted variables $\writef(t_e)(v) = \emptyset$.
  \end{itemize}
  
 Note that an event $\tuple{t,w,w^d}$ cannot be compliant with any firing unless $w(v)\in \writef(t)(v)$ for each variable in $dom(w)\cap dom(\writef(t))$, and $w^d\cap dom(\writef(t))\subseteq \{v \mid (v,\emptyset) \in \writef(t)\}$. Since these properties can be verified just by comparing the events to the model, then in the following we assume that they are both satisfied for all the events in a trace.

Given a \ournet $W$, and a trace $\tau$, the formal definition of the ``injection'' of $\tau$ into $W$, denoted with $W^\tau$, is given below.
  \begin{definition}[Trace workflow]\label{def:trace:workflow}
Let $W = \tuple{\dmodel, \nmodel = \tuple{P,T,F}, \writef, \guardf}$ be a \ournet and $\tau = (e_1,\ldots, e_n)$ -- where $e_i = \tuple{t_i,w_i, w_i^d}$ -- a trace of $W$. The \emph{trace workflow} $W^\tau = \tuple{\dmodel, \nmodel^\tau = \tuple{P^\tau,T^\tau,F^\tau}, \writef^\tau, \guardf^\tau}$ is defined as follows:
\begin{itemize}
  \item $P^\tau = P \cup \set{p_{e_0}}\cup\set{p_{e}\mid e\in \tau}$, with $p_{e_0}$, $p_e$ new places; 
  \item $T^\tau = T \cup \set{t_e\mid e\in \tau},$ with $t_e$ new transitions; 
  \item $F^\tau = F \cup {}$  \\
	    \hspace*{.8cm} $\set{(t_{e_i},p)\mid i=1\ldots n, (t_i,p)\in F}\cup\set{(p,t_{e_i})\mid i=1\ldots n, (p,t_i)\in F} \cup {}$\\
		\hspace*{.8cm} $\set{(t_{e_i},p_{e_i})\mid i=1\ldots n}\cup\set{(p_{e_{i-1}},t_{e_i})\mid i=1\ldots n} \cup \set{(start_t,p_{e_0}), (p_{e_n},end_t)}$
  \item $\writef^\tau(t) = {}
  \begin{cases}
    \set{v \mapsto \{w_i(v)\}\mid v\in dom(w_i)\cap dom(\writef(t_i))}\cup {}\\\;\set{v\mapsto\writef(t_i)(v)\mid v\in dom(\writef(t_i))\setminus dom(w_i)} & \text{for $t = t_{e_i}$}\\
    \writef(t) & \text{for $t\in T$}
  \end{cases}$
 \item $\guardf^\tau(t) = {} 
  \begin{cases}
    \guardf(t_i)\land \bigwedge_{v\in dom(w_i)\setminus dom(\writef(t_i))} v = w_i(v) \land {}\\\; \bigwedge_{v\in w_i^d\setminus dom(\writef(t_i))} \neg\deff(v) & \text{for $t = t_{e_i}$}\\
    \guardf(t) & \text{otherwise}
  \end{cases}$
\end{itemize}
\end{definition}

% The resulting so-called \emph{trace workflow} modify the original workflow so that the only cases (i.e., sequences of valid firings) of this new workflow are those that satisfy the observed trace. This, so called, \emph{trace workflow} includes new transitions (filled transitions in Figure~\ref{fig:trace-compilation}) corresponding to the events observed in the trace and making sure that they can be activated only in the right order and behave as the ``original'' transition.
% %
% For each event $e$ corresponding to a transition $t$ setting and deleting the values of some variables we introduce a new transition $t_e$ such that:
% \begin{compactitem}
%   \item it is placed in ``parallel'' with the original transition;
%   \item it includes an additional input place connected to the output place of the preceding event, and an additional output place which connects it to the next event;
%   \item its guard is the same as the original transition; and
%   \item the $\writef(t_e)$ specifies exactly the variables and values of the event, i.e.\ for variables $v$ set to value $d$ then $\writef(t_e)(v) = \set{d}$, while for deleted ones $\writef(t_e)(v) = \emptyset$.
% \end{compactitem}

It is immediate to see that $W^\tau$ is a strict extension of $W$ (only new nodes are introduced) and, since all newly introduced nodes are in a path connecting the start and sink places, it is a \ournet, whenever the original one is a \ournet net.
	
\begin{figure}[h]
  \centering
    \includegraphics[width=.8\textwidth]{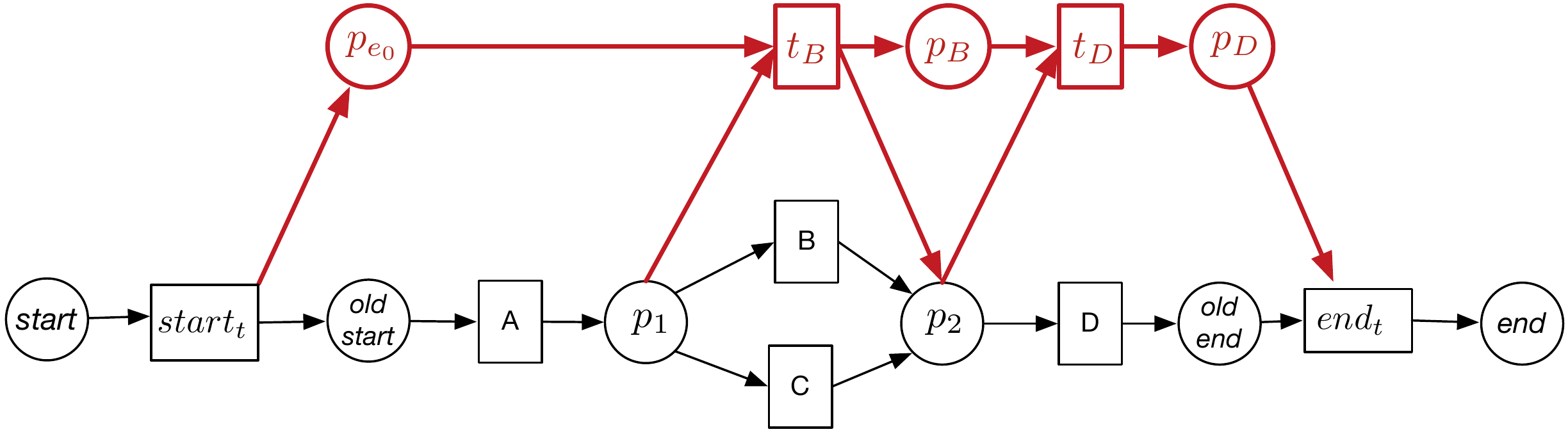}
  \caption{A sample of trace workflow obtained by adding trace $\langle B,D\rangle$.}
  \label{fig:imgs_traceinjectionsample}
\end{figure}

By looking at this definition, at its graphical illustration at the bottom of Figure~\ref{fig:trace-compilation}, and also at the simple example in Figure~\ref{fig:imgs_traceinjectionsample} it is easy to see why the red transitions must be preferred over the black ones in all the sequences of valid firings that go from $start$ to $end$ (i.e., in all cases). Intuitively, transition $end_t$ needs a token in the red place $p_{e_n}$ to fire, in addition to the token in the original end of the initial net. By construction all the red places are only connected to red transitions. Thus, to properly terminate, when choosing between $t_e$ and $t$, $t_e$ must be preferred so that it produces in output all the tokens for the original black output places plus the one for the new red place. Note that whenever $t$ was able to fire in the original net, $t_e$ can fire as well. Indeed it is easy to see that, by construction, the token generated by $start_t$ in $p_{e_0}$ is only consumed and produced by red transitions, and propagated into red places.
  
%%
% Besides, $W^\tau$ is such that any of its terminating cases is compliant with $\tau$. % Intuitively this is due to the fact each observed event in $\tau$ is associated to one of the new places and these new places mark the fact that the transition observed in the event is part of the case. \boh{Tokens can only enter in these new places once and before ending up in a place enabling one of the event transitions all the preceding ones (w.r.t.\ the trace) had to be fired.}

We now prove the correctness and completeness of the approach by showing that $W^\tau$ characterises all and only the cases $C$ of $W$ to which $\tau$ is compliant with. We focus here only on the main intuitive steps of the proof, whose technical details are reported in full in~\ref{apx:encoding:traces}. 

Roughly speaking what we need to show is that: 
\begin{itemize}
	\item (correctness) if $C$ is an arbitrary case of $W^\tau$, then there is a ``corresponding'' case $C'$ of $W$ such that $\tau$ is compliant with $C'$; and
	\item (completeness) if $\tau$ is compliant with a case $C$ of $W$, then there is a case $C'$ in $W^\tau$ ``corresponding'' to $C$. 
\end{itemize}

A fundamental step for our proof is therefore a way to relate cases from $W^\tau$ to the original \ournet $W$. To do this we introduce a projection function $\wkfProj_\tau$ that maps elements from sequences of valid firings (cases) of the enriched \ournet $W^\tau$ to sequences of valid firings (cases) composed of elements from the original \ournet $W$. 
The formal definition of $\wkfProj_\tau$ is given in Definition~\ref{def:trace:workflow:project} of~\ref{apx:encoding:traces}. In essence $\wkfProj_\tau$ maps newly introduced transitions $t_e$ to the original transition $t$ corresponding to the event $e$, and projects away the newly introduced places in the markings.

Once defined the projection function $\wkfProj_\tau$ the correctness and completeness statement can be formally formulated as follows: 

\begin{restatable}{theorem}{thTraceWorkflow}
	\label{Th:traceworkflow}
  Let $W$ be a \ournet and $\tau = (e_1,\ldots, e_n)$ a trace; then $W^\tau$ characterises all and only the cases of $W$ compatible with $\tau$. That is
  \begin{compactitem}
    \item[$\Rightarrow$] if $C$ is a case of $W^\tau$ then $\tau$ is compliant with the case of $W$ $\wkfProj_\tau(C)$; and 
    \item[$\Leftarrow$] if $\tau$ is compliant with the case of $W$ $C$ then there is a case $C'$ of $W^\tau$ s.t.\ $\wkfProj_\tau(C') = C$.
  \end{compactitem}
\end{restatable}

The proof of correctness follows from two foundamental properties of $W^\tau$ and the projection function $\wkfProj_\tau$ respectively. The first property, ensured by construction of $W^\tau$, states that $\tau$ is compliant with all cases of $W^\tau$. The second property, instead, states that any case for $W^\tau$ can be replayed on $W$ trough the projection $\wkfProj_\tau$, by mapping the new transitions $t_e$ into the original ones $t$, as shown by Lemma~\ref{lemma:validFiringMap} in~\ref{apx:encoding:traces}.  
Given these two properties, it is possible to prove that, if $C$ is a case
%whenever the case $C$
 of $W^\tau$, then $\tau$ is compatible with it and with its projection $\wkfProj_\tau(C)$ on $W$ (See Lemma~\ref{lemma:wkf:trace:encoding:crct} in~\ref{apx:encoding:traces}).
The proof of completeness instead is based in the fact that we can build a case for $W^\tau$ starting from the case of $W$ with which $\tau$ is compliant, by substituting the occurrences of firings corresponding to events in $\tau$ with the newly introduced transitions (See Lemma~\ref{lemma:wkf:trace:encoding:cmpl} in~\ref{apx:encoding:traces}).

Theorem~\ref{Th:traceworkflow} provides the main result of this section and is the basis for the reduction of the trace completion for $W$ and $\tau$ to the reachability problem for $W^\tau$: indeed, to check if $\tau$ is compliant with $W$ all we have to do is finding a case of $W^{\tau}$.

\begin{corollary}[Trace completion as reachability] \label{cor:reach-problem}
  Let $W$ be a \ournet and $\tau$ a trace; then $\tau$ is compliant with $W$ iff there is a final state $(\mrk_e, \eta')$ of $W^\tau$ that can be reached from the initial state $(\mrk_s, \eta_s)$.
\end{corollary}

\begin{proof}
	This is immediate from Theorem~\ref{Th:traceworkflow} because a final state of $W^\tau$ can be reached from its initial state iff there exists at least a case of $W^{\tau}$.
\end{proof}

%% 
%% Forse inserire questo
%%
% It immediately follows that each of the above case is a finished and possible completion for $\tau$. Consider indeed the border case in which $\tau$ is empty: in such a way we obtain a check on the soundness of the model itself, i.e., we check if there is (at least) one way to correctly execute it from the start to the end.

We conclude the section by stating that the transformation generating $W^\tau$ is preserving the safeness properties of the original workflow (proof in~\ref{apx:encoding:traces}).

\begin{restatable}{theorem}{lemmaKSafe}
	\label{lemma:ksafe}
  Let $W$ be a \ournet and $\tau$ a trace of $W$. If $W$ is $k$-safe then $W^\tau$ is $k$-safe as well.
\end{restatable}
% The preservation of $k$-safeness, and in particular of 1-safeness, \boh{is important in order to employ the proposed automated reasoning techniques described in Section~\ref{sec:encoding:planning}.}

%%% Local Variables:
%%% mode: latex
%%% TeX-master: "main.tex"
%%% save-place: t
%%% End:

%% file: evaluation.tex
\section{Evaluation}
\label{sec:evaluation}
%\todocdfinline{Is it fine to call the three tools solvers? An alternative would be reasoners. And is it fine to focus the evaluation on the three tools - each with its own paradigm?}

In this section we aim at evaluating the three investigated solvers \clingo, \fastdw and \nuxmv, and the respective paradigms answer set programming, automated planning and model checking, on the task of trace completion for
%completing incomplete traces of
 data-aware workflows. In detail, we are interested to answer the following two research questions:
\begin{enumerate}[start=1,label={\textbf{RQ\arabic*}.}]
%\item Are	the three solvers able to return a solution to the task of completing incomplete traces in a reasonable amount of time?
%\item Which of the three solvers does outperform the others by returning a solution to the task of completing incomplete traces in a reasonable amount of time?
\item What are the performance of the three solvers when accomplishing the task of trace completion
%completing incomplete traces
 by leveraging also data payloads?
\item What are the main factors impacting the solvers' capability to deal with the task of trace completion?
%completing incomplete traces?
\end{enumerate}

The first research question aims at evaluating and comparing the performance of the three solvers, both in terms of capability of returning a result within a reasonable amount of time and, when a result is returned, in terms of the time required for solving the task. The second research question, instead, aims at investigating which factors could influence the solvers' capability to deal with such a task (e.g., the level of incompleteness of the execution traces, the characteristics of the traces, the size of the process model, the domain of the data payload associated to the events, or the data payload itself).  

In order to answer the above research questions by evaluating the three investigated solvers in different scenarios, we carried out two different evaluations: one based on synthetic logs and a second one based on real-life logs. Indeed, on the one hand, the synthetic log evaluation (Subsection~\ref{ssec:synthetic_eval}) allows us to investigate the solvers' capabilities on logs with specific characteristics. On the other hand, the real-life log evaluation (Subsection~\ref{ssec:real-life_eval}) allows us to investigate the capability of the solvers to deal with real-life problems. 
In the next subsections, we detail the two evaluations (Subsections~\ref{ssec:synthetic_eval} and \ref{ssec:real-life_eval}) and we investigate some threats to the validity of the obtained results (Subsection~\ref{ssec:threats}). 
%In the next subsections, besides detailing the two evaluations (Subsections~\ref{ssec:synthetic_eval} and \ref{ssec:real-life_eval}), we carry on an overall discussion on the obtained results (Subsection~\ref{ssec:discussion}) and we finally investigate some threats to the validity of the obtained results. 

All the experiments have been carried out on a Kubernetes cluster running over a pool of virtual machines hosted on hardware based on Intel Xeon X5650 2.67GHz processors. For the container running the experiments a single CPU has been allocated since the reasoning systems we used do not exploit multiprocessing. For each run we set up a (real) time limit of one hour, and a memory limit of 8GiB.\footnote{Given the single task nature of the container, the CPU time is only marginally lower than the real time.} 
The code used for the experiments is available for download from~\url{https://doi.org/10.5281/zenodo.3459656}. The directory \texttt{experiments} includes the models and experiments descriptions, as well as a \texttt{README.md} file with the details on how to run them.

\subsection{Synthetic Log Evaluation}
\label{ssec:synthetic_eval}
In order to investigate the performance of the three solvers in different settings, we built synthetic models with different characteristics. For each of them, we identified different traces with different features and, in turn, for each trace, we considered different levels of trace completeness ranging from the complete trace (100\% complete) to the empty trace (0\% complete). 
In the next subsections we describe the datasets, the procedure and the metrics used, as well as the obtained results.
\subsubsection{Dataset, procedure and metrics}
We built five different synthetic models $M1 - M5$ with different characteristics and of different size. In detail, $M1$ is the model reported in Figure~\ref{fig:m1}. The model deals with 5 different variables: $number$, $first$, $second$, $third$, $fourth$ and $fifth$. Activity $A$ is able to write four variables ($first$, $second$, $third$ and $fourth$), activity $I$ writes variable $number$ and, finally, activity $G$ writes variable $fifth$. 

\begin{figure*}[h]
  \centering
    \tikzstyle{place}=[circle,draw=black,thick]
    \tikzstyle{transition}=[text width=.45cm,text centered,rectangle,draw=black!50,fill=black!20,thick]
    \scalebox{0.57}{\begin{tikzpicture}[font=\large]
      \node (start) [place,label=above:$start$] {};
      % \node at ( 0,1) [place] {};
%       \node at ( 0,0) [place] {};
      \node[transition] (a)  [right=.5cm of start, anchor=west] {A};
      \node[place,label=right:$p_1$] (p1) [right=.5cm of a, anchor=west] {};
      \node[transition] (b)  [above right=1cm of p1, anchor=west] {B};
      \node[transition] (c)  [below right=1cm of p1, anchor=west] {C};
      \node[place,label=above:$p_2$] (p2) [right=.5cm of b] {};
      \node[place,label=below:$p_3$] (p3) [right=.5cm of c] {};
      \node[transition] (d)  [right=.5cm of p2, anchor=west] {D};
      \node[transition] (e)  [right=.5cm of p3, anchor=west] {E};
      \node[place,label=left:$p_4$] (p4) [right=5cm of p1, anchor=west] {};
      \node[transition] (f)  [right=.5cm of p4] {F};
      \node[place,label=above:$p_6$] (p6)  [right=.5cm of f] {};
      \node[place,label=above:$p_5$] (p5)  [above=1cm of p6] {};
      \node[place,label=above:$p_7$] (p7)  [below=1cm of p6] {};
      \node[transition] (g)  [right=.5cm of p5] {G};
      \node[transition] (h)  [right=.5cm of p6] {H};
      \node[transition] (i)  [right=.5cm of p7] {I};
      \node[place,label=above:$p_8$] (p8)  [right=.5cm of g] {};
      \node[place,label=above:$p_9$] (p9)  [right=.5cm of h] {};
      \node[place,label=above:$p_{10}$] (p10)  [right=.5cm of i] {};
      \node[transition] (l)  [below=1cm of p10] {L};
      \node[transition] (m)  [right=.5cm of p9] {M};
      \node[place,label=above right:$p_{11}$] (p11) [right=.5cm of m] {};
      \node[transition] (o)  [right=1cm of p11] {O};
      \node[transition] (n)  [above=1cm of o] {N};
      \node[transition] (p)  [below=1cm of o] {P};
      \node[place,label=above:$p_{12}$] (p12)  [right=.5cm of n] {};
      \node[place,label=above:$p_{13}$] (p13)  [right=.5cm of o] {};
      \node[place,label=above:$p_{14}$] (p14)  [right=.5cm of p] {};
      \node[transition] (r)  [right=.5cm of p12] {R};
      \node[transition] (s)  [right=.5cm of p13] {S};
      \node[transition] (t)  [right=.5cm of p14] {T};
      \node[place,label={above right:$p_{15}$}] (p15)  [right=.5cm of s] {};
      \node[transition] (u)  [right=.5cm of p15] {U};
      \node[place,label=above:$end$] (end)  [right=.5cm of u] {};

      \draw [->,thick] (start.east) -- (a.west);
      \draw [->,thick] (a.east) -- (p1.west);
      \draw [->,thick] (p1.north) to [bend left=45] (b.west);
      \draw [->,thick] (p1.south) to [bend right=45] (c.west);
      \draw [->,thick] (b.east) to (p2.west);
      \draw [->,thick] (c.east) to (p3.west);
      \draw [->,thick] (p2.east) to (d.west);
      \draw [->,thick] (p3.east) to (e.west);
      \draw [->,thick] (d.east) to [bend left=45] (p4.north);
      \draw [->,thick] (e.east) to [bend right=45] (p4.south);
      \draw [->,thick] (p4.east) to (f.west);
      \draw [->,thick] (f.east) to (p6.west);
      \draw [->,thick] (f.north) to [bend left=45] (p5.west);
      \draw [->,thick] (f.south) to [bend right=45] (p7.west);
      \draw [->,thick] (p5.east) to (g.west);
      \draw [->,thick] (p6.east) to (h.west);
      \draw [->,thick] (p7.east) to (i.west);
      \draw [->,thick] (g.east) to (p8.west);
      \draw [->,thick] (h.east) to (p9.west);
      \draw [->,thick] (i.east) to (p10.west);
      \draw [->,thick] (p10.south) to node[fill=white,inner sep=1pt, midway] {{num$\neq$5}} (l.north);
      \draw [->,thick] (l.west) to [bend left=30] (p7.south);
      \draw [->,thick] (p8.east) to [bend left=45](m.north);
      \draw [->,thick] (p9.east) to (m.west);
      \draw [->,thick] (p10.east) to [bend right=45] node[fill=white,inner sep=1pt, midway] {{num=5}} (m.south);
      \draw [->,thick] (p9.east) to (m.west);
      \draw [->,thick] (m.east) to (p11.west);
      \draw [->,thick] (p11.north) to [bend left=30] node[fill=white,inner sep=1pt, midway] {{f=1}} (n.west);
      \draw [->,thick] (p11.east) to node[below,fill=white,inner sep=1pt, midway] {{s=2}} (o.west);
      \draw [->,thick] (p11.south) to [bend right=30]node[fill=white,inner sep=1pt, midway] {{t=3}} (p.west);
      \draw [->,thick] (n.east) to (p12.west);
      \draw [->,thick] (o.east) to (p13.west);
      \draw [->,thick] (p.east) to (p14.west);
      \draw [->,thick] (p12.east) to (r.west);
      \draw [->,thick] (p13.east) to (s.west);
      \draw [->,thick] (p14.east) to (t.west);
      \draw [->,thick] (r.east) to [bend left=30] (p15.north);
      \draw [->,thick] (s.east) to  (p15.west);
      \draw [->,thick] (t.east) to [bend right=30](p15.south);
      \draw [->,thick] (p15.east) to (u.west);
      \draw [->,thick] (u.east) to (end.west);  
      \node (bb) [draw,thick,dotted,inner ysep=1.5em,fit=(a) (u) (l) (r)] {};
      \node [anchor=south east, at=(bb.south east)] {\scalebox{1.5}{$K$}};    
    \end{tikzpicture}}
  \caption{Model M1 and its core $K$}
  \label{fig:m1}
\end{figure*}
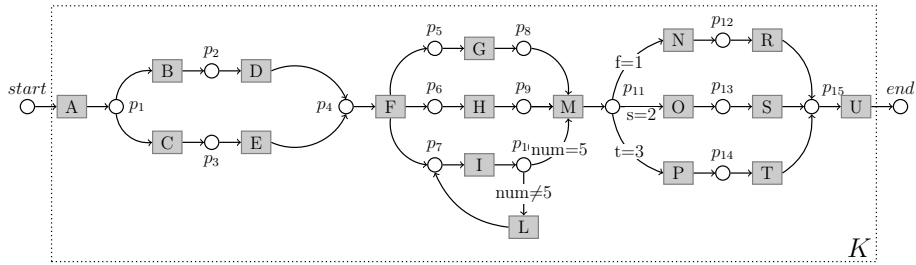

Figure~\ref{fig:m1} also shows the guards related to model $M1$. In detail, the exit condition from the loop ($L - I$) is variable $number$ equals to $5$ at $I$, while whenever the value of $number$ written by $I$ is different from $5$, the loop is exited and $M$ is executed. Moreover, variables $first$, $second$ and $third$ determine the branch to be taken at the end of the model. In detail, if at $M$ $first$ is equal to $1$, the branch with $N$ is taken, if $second$ is equal to $2$, the branch $O$ is taken, while if $third$ is equal to $3$, the branch starting with $P$ is taken\footnote{For the sake of readability $first$, $second$ and $third$ are shortened in $f, s$ and $t$ in Figure ~\ref{fig:m1}.}. Whenever more than one of the three conditions are verified, one of the three branches can be non-deterministically taken. 

\begin{figure}
	\centering
	\subfigure[$M1$]{%
	\label{fig:m1b}%
	\includegraphics[scale=0.3]{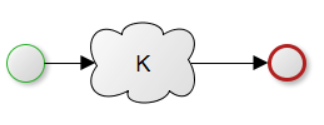}}%
	\qquad
	\subfigure[$M2$]{%
	\label{fig:m2}%
	\includegraphics[scale=0.3]{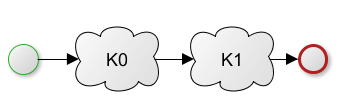}}%
	\qquad
	\subfigure[$M3$]{%
	\label{fig:m3}%
	\includegraphics[scale=0.3]{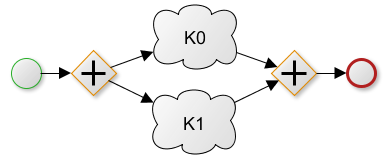}}
	\\%
	\subfigure[$M4$]{%
	\label{fig:m4}%
	\includegraphics[scale=0.3]{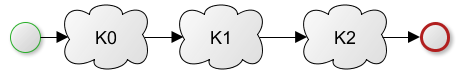}}%
	\qquad
	\subfigure[$M5$]{%
	\label{fig:m5}%
	\includegraphics[scale=0.3]{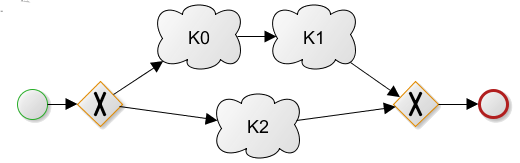}}%
	\caption{Models $M1,\ldots, M5$ used in the evaluation}
	\label{fig:models}
\end{figure}

We built the synthetic models $M2$ $\ldots$ $M5$ starting from model $M1$ and creating replicas of $M1$. For notation simplicity, we denote with $K$ the core part of $M1$, i.e., the model obtained by excluding the start and the end event (see Figure~\ref{fig:m1b}) and we enumerate the replicas of $K$ in $M2$ $\ldots$ $M5$ with a progressive number. In detail, $M2$ concatenates two replicas of $K$, labelled as $K0$ and $K1$ (see Figure~\ref{fig:m2}), $M3$ uses the parallel construct to compose together two replicas of $K$, namely $K0$ and $K1$ (see Figure~\ref{fig:m3}). $M4$ concatenates three replicas of $K$, namely $K0$, $K1$ and $K2$ (see Figure~\ref{fig:m4}). Finally, in $M5$, three replicas of $K$ are composed in an exclusive choice. In detail, one of the two alternative branches concatenates two replicas of $K$, i.e., $K0$ and $K1$, while the other exclusive branch is composed of a single replica of $K$, named $K2$ (see Figure~\ref{fig:m5}). Different replicas of $K$ do not share variables, which are renamed in order to avoid interactions that could limit the compatible traces. 

For each of the five models $M1 \ldots M5$, we generated $8$ different execution traces of different length and with different characteristics. In detail, we generated $4$ compliant traces and $4$ non-compliant traces. Out of the $4$ compliant traces, in two traces (trace types $T3$ and $T4$)
%$t3$ and $t4$)
 the loop(s) of $K$ replica(s) is(are) never executed, while in the remaining two, the loop is executed twice. Moreover, in two (trace types $T1$ and $T3$)
%$t1$ and $t3$)
 out of the four cases of the compliant traces, the path followed by the trace is deterministically driven by the values associated by the activity $A$ to the variables $\mathit{first}$, $\mathit{second}$ and $\mathit{third}$, while in the remaining two traces more than one branch could be activated based on the values of the variables at activity $M$. Also in the case of non-compliant traces, two traces (trace types $T5$ and $T7$)
%$t5$ and $t7$)
 do not execute any loop iteration, while the other $2$ traces iterate over the loop two times. Finally, out of the four non-compliant traces, two (trace type $T5$ and $T6$)
%$t5$ and $t6$)
 are not compliant only because of control flow violations (mutually exclusive paths are executed in the same trace), while the remaining two traces are non-compliant due to data constraint violations (a path that is not activated by the corresponding guard is executed). Table~\ref{tab:synthetic_traces} reports, for each trace $t_{ij}$ of model $M_i$ and trace type $T_j$, some information about its characteristics.

\definecolor{Gray}{gray}{0.9}
\newcolumntype{a}{>{\columncolor{Gray}}c}
\definecolor{LightGray}{gray}{0.9}
\newcolumntype{b}{>{\columncolor{LightGray}}c}
\begin{table}
	\centering
	\scalebox{0.43}{
  \begin{tabular}{|a|c c|c c c c c c|c c|c c c c c c|}
  \toprule
	\textbf{Model} & \textbf{Trace type} & \textbf{Trace} & \textbf{Cycle It.} & \textbf{Det.} & \textbf{Conf.} & \textbf{Cause} & \textbf{Compl.} & \textbf{Length} & \textbf{Trace type} & \textbf{Trace} & \textbf{Cycle It.} & \textbf{Det.} & \textbf{Conf.} & \textbf{Cause} & \textbf{Compl.} & \textbf{Length}\\	\hline
	& & & & & & & $100\%$ & 11 & & & & & & & $100\%$ & 13 \\
	& & & & & & & $75\%$ & 8 & & & & & & & $25\%$ & 10 \\
	& & & & & & & $50\%$ & 5 & & & & & & & $50\%$ & 7 \\ 
	& \multirow{-4}{*}{$T_{1}$} & \multirow{-4}{*}{$t_{11}$} & \multirow{-4}{*}{N} & \multirow{-4}{*}{Y} & \multirow{-4}{*}{C} & & $25\%$ & 3 & \multirow{-4}{*}{$T_{5}$} & \multirow{-4}{*}{$t_{15}$} & \multirow{-4}{*}{Y} & \multirow{-4}{*}{N} & \multirow{-4}{*}{NC} & \multirow{-4}{*}{CF} & $75\%$ & 4 \\ \cline{2-17}
	& & & & & & & $100\%$ & 11 & & & & & & & $100\%$ & 17 \\
	& & & & & & & $75\%$ & 8 & & & & & & & $75\%$ & 13 \\ 
	& & & & & & & $50\%$ & 5 & & & & & & & $50\%$ & 10 \\
	& \multirow{-4}{*}{$T_{2}$}& \multirow{-4}{*}{$t_{12}$} & \multirow{-4}{*}{N} & \multirow{-4}{*}{N} & \multirow{-4}{*}{C} & & $25\%$ & 3 & \multirow{-4}{*}{$T_{6}$}& \multirow{-4}{*}{$t_{16}$} & \multirow{-4}{*}{2} & \multirow{-4}{*}{Y} & \multirow{-4}{*}{NC} & \multirow{-4}{*}{CF} & $25\%$ & 6 \\ \cline{2-17}
	& & & & & & & $100\%$ & 15 & & & & & & & $100\%$ & 11 \\ 
	& & & & & & & $75\%$ & 11 & & & & & & & $75\%$ & 8 \\
	& & & & & & & $50\%$ & 8 & & & & & & & $50\%$ & 5 \\
	& \multirow{-4}{*}{$T_{3}$} & \multirow{-4}{*}{$t_{13}$} & \multirow{-4}{*}{2} & \multirow{-4}{*}{Y} & \multirow{-4}{*}{C} & & $25\%$ & 4 & \multirow{-4}{*}{$T_{7}$}& \multirow{-4}{*}{$t_{17}$} & \multirow{-4}{*}{N} & \multirow{-4}{*}{Y} & \multirow{-4}{*}{NC} & \multirow{-4}{*}{CF+DF} & $25\%$ & 3 \\ \cline{2-17}
	& & & & & & & $100\%$ & 15 & & & & & & & $100\%$ & 15 \\
	& & & & & & & $75\%$ & 11 & & & & & & & $75\%$ & 11\\ 
	& & & & & & & $50\%$ & 8 & & & & & & & $50\%$ & 8 \\
	\multirow{-16}{*}{$M1$} & \multirow{-4}{*}{$T_{4}$} & \multirow{-4}{*}{$t_{14}$} & \multirow{-4}{*}{2} & \multirow{-4}{*}{N} & \multirow{-4}{*}{C} & & $25\%$ & 4 & \multirow{-4}{*}{$T_{8}$} &\multirow{-4}{*}{$t_{18}$} & \multirow{-4}{*}{2} & \multirow{-4}{*}{Y} & \multirow{-4}{*}{NC} & \multirow{-4}{*}{CF+DF} & $25\%$ & 4 \\ \hline
	& & & & & & & $100\%$ & 22 & & & & & & & $100\%$ & 26 \\ 
	& & & & & & & $75\%$ & 16 & & & & & & & $75\%$ & 20 \\ 
	& & & & & & & $50\%$ & 10 & & & & & & & $50\%$ & 14 \\ 
	& \multirow{-4}{*}{$T_{1}$} & \multirow{-4}{*}{$t_{21}$} & \multirow{-4}{*}{N} & \multirow{-4}{*}{Y} & \multirow{-4}{*}{C} & & $25\%$ & 6 & \multirow{-4}{*}{$T_{5}$} & \multirow{-4}{*}{$t_{25}$} & \multirow{-4}{*}{N} & \multirow{-4}{*}{Y} & \multirow{-4}{*}{NC} & \multirow{-4}{*}{CF} & $25\%$ & 8 \\ \cline{2-17}
	& & & & & & & $100\%$ & 22 & & & & & & & $100\%$ & 34 \\ 
	& & & & & & & $75\%$ & 16 & & & & & & & $75\%$ & 26 \\ 
	& & & & & & & $50\%$ & 10 & & & & & & & $50\%$ & 20 \\ 
	& \multirow{-4}{*}{$T_{2}$} & \multirow{-4}{*}{$t_{22}$} & \multirow{-4}{*}{N} & \multirow{-4}{*}{N} & \multirow{-4}{*}{C} & & $25\%$ & 6 & \multirow{-4}{*}{$T_{6}$} & \multirow{-4}{*}{$t_{26}$} & \multirow{-4}{*}{2} & \multirow{-4}{*}{Y} & \multirow{-4}{*}{NC} & \multirow{-4}{*}{CF} & $25\%$ & 12 \\ \cline{2-17}
  &	& & & & & & $100\%$ & 30 & & & & & & & $100\%$ & 22 \\ 
  & & & & & & & $75\%$ & 22 & & & & & & &$75\%$ & 16 \\ 
	& & & & & & & $50\%$ & 16 & & & & & & & $50\%$ & 10 \\ 
	& \multirow{-4}{*}{$T_{3}$} & \multirow{-4}{*}{$t_{23}$} & \multirow{-4}{*}{2} & \multirow{-4}{*}{Y} & \multirow{-4}{*}{C} & & $25\%$ & 8 & \multirow{-4}{*}{$T_{7}$} & \multirow{-4}{*}{$t_{27}$} & \multirow{-4}{*}{N} & \multirow{-4}{*}{Y} & \multirow{-4}{*}{NC} & \multirow{-4}{*}{CF+DF} &  $25\%$ & 6 \\ \cline{2-17}
	& & & & & & & $100\%$ & 30 & & & & & & & $100\%$ & 30 \\ 
	& & & & & & & $75\%$ & 22 & & & & & & & $75\%$ & 22 \\ 
	& & & & & & & $50\%$ & 16 & & & & & & & $50\%$ & 16 \\ 
	\multirow{-16}{*}{$M2$} & \multirow{-4}{*}{$T_{4}$} & \multirow{-4}{*}{$t_{24}$} & \multirow{-4}{*}{2} & \multirow{-4}{*}{N} & \multirow{-4}{*}{C} & & $25\%$ & 8 & \multirow{-4}{*}{$T_{8}$} & \multirow{-4}{*}{$t_{28}$} & \multirow{-4}{*}{2} & \multirow{-4}{*}{Y} & \multirow{-4}{*}{NC} & \multirow{-4}{*}{CF+DF}& $25\%$ & 8 \\	\hline
  & & & & & & & $100\%$ & 22 & & & & & & & $100\%$ & 26 \\
  & & & & & & & $75\%$ & 16 & & & & & & & $75\%$ & 20 \\
  & & & & & & & $50\%$ & 10 & & & & & & & $50\%$ & 14 \\
& \multirow{-4}{*}{$T_{1}$} & \multirow{-4}{*}{$t_{31}$} & \multirow{-4}{*}{Y} & \multirow{-4}{*}{Y} & \multirow{-4}{*}{C} & & $25\%$ & 6 & \multirow{-4}{*}{$T_{5}$} & \multirow{-4}{*}{$t_{35}$} & \multirow{-4}{*}{N} & \multirow{-4}{*}{Y} & \multirow{-4}{*}{NC} & \multirow{-4}{*}{CF} & $25\%$ & 8 \\ \cline{2-17}	
 & & & & & & & $100\%$ & 22  & & & & & & & $100\%$ & 34 \\ 
 & & & & & & & $75\%$ & 16 & & & & & & & $75\%$ & 26 \\ 
 & & & & & & & $50\%$ & 10 & & & & & & & $50\%$ & 20 \\ 
 & \multirow{-4}{*}{$T_{2}$} & \multirow{-4}{*}{$t_{32}$} & \multirow{-4}{*}{N} & \multirow{-4}{*}{Y} & \multirow{-4}{*}{C} & & $25\%$ & 6 & \multirow{-4}{*}{$T_{6}$} & \multirow{-4}{*}{$t_{36}$} & \multirow{-4}{*}{2} & \multirow{-4}{*}{Y} & \multirow{-4}{*}{NC} & \multirow{-4}{*}{CF} & $25\%$ & 12 \\ \cline{2-17}
	& & & & & & & $100\%$ & 30 & & & & & & & $100\%$ & 22 \\ 
	& & & & & & & $75\%$ & 22 & & & & & & &$75\%$ & 16 \\ 
	& & & & & & & $50\%$ & 16 & & & & & & & $50\%$ & 10 \\ 
	& \multirow{-4}{*}{$T_{3}$} & \multirow{-4}{*}{$t_{33}$} & \multirow{-4}{*}{2} & \multirow{-4}{*}{Y} & \multirow{-4}{*}{C} & & $25\%$ & 8 & \multirow{-4}{*}{$T_{7}$} & \multirow{-4}{*}{$t_{37}$} & \multirow{-4}{*}{N} & \multirow{-4}{*}{Y} & \multirow{-4}{*}{NC} & \multirow{-4}{*}{CF+DF} &  $25\%$ & 6 \\ \cline{2-17}
	& & & & & & & $100\%$ & 30 & & & & & & & $100\%$ & 30 \\ 
	& & & & & & & $75\%$ & 22 & & & & & & & $75\%$ & 22 \\ 
	& & & & & & & $50\%$ & 16 & & & & & & & $50\%$ & 16 \\ 
	\multirow{-16}{*}{$M3$} & \multirow{-4}{*}{$T_{4}$} & \multirow{-4}{*}{$t_{34}$} & \multirow{-4}{*}{2} & \multirow{-4}{*}{N} & \multirow{-4}{*}{C} & & $25\%$ & 8 & \multirow{-4}{*}{$T_{8}$} & \multirow{-4}{*}{$t_{38}$} & \multirow{-4}{*}{2} & \multirow{-4}{*}{Y} & \multirow{-4}{*}{NC} & \multirow{-4}{*}{CF+DF}& $25\%$ & 8 \\	\hline	
  & & & & & & & $100\%$ & 33 & & & & & & & $100\%$ & 39 \\
  & & & & & & & $75\%$ & 24 & & & & & & & $75\%$ & 30 \\
  & & & & & & & $50\%$ & 15 & & & & & & & $50\%$ & 21 \\
& \multirow{-4}{*}{$T_{1}$} & \multirow{-4}{*}{$t_{41}$} & \multirow{-4}{*}{Y} & \multirow{-4}{*}{Y} & \multirow{-4}{*}{C} & & $25\%$ & 9 & \multirow{-4}{*}{$T_{5}$} & \multirow{-4}{*}{$t_{45}$} & \multirow{-4}{*}{N} & \multirow{-4}{*}{Y} & \multirow{-4}{*}{NC} & \multirow{-4}{*}{CF} & $25\%$ & 12 \\ \cline{2-17}	
 & & & & & & & $100\%$ & 33 & & & & & & & $100\%$ & 51 \\ 
 & & & & & & & $75\%$ & 24 & & & & & & & $75\%$ & 39 \\ 
 & & & & & & & $50\%$ & 15 & & & & & & & $50\%$ & 30 \\ 
 & \multirow{-4}{*}{$T_{2}$} & \multirow{-4}{*}{$t_{42}$} & \multirow{-4}{*}{N} & \multirow{-4}{*}{Y} & \multirow{-4}{*}{C} & & $25\%$ & 9 & \multirow{-4}{*}{$T_{6}$} & \multirow{-4}{*}{$t_{46}$} & \multirow{-4}{*}{2} & \multirow{-4}{*}{Y} & \multirow{-4}{*}{NC} & \multirow{-4}{*}{CF} & $25\%$ & 18 \\ \cline{2-17}
	& & & & & & & $100\%$ & 45 & & & & & & & $100\%$ & 33 \\ 
	& & & & & & & $75\%$ & 33 & & & & & & &$75\%$ & 24\\ 
	& & & & & & & $50\%$ & 24 & & & & & & & $50\%$ & 15 \\ 
	& \multirow{-4}{*}{$T_{3}$} & \multirow{-4}{*}{$t_{43}$} & \multirow{-4}{*}{2} & \multirow{-4}{*}{Y} & \multirow{-4}{*}{C} & & $25\%$ & 12 & \multirow{-4}{*}{$T_{7}$} & \multirow{-4}{*}{$t_{47}$} & \multirow{-4}{*}{N} & \multirow{-4}{*}{Y} & \multirow{-4}{*}{NC} & \multirow{-4}{*}{CF+DF} &  $25\%$ & 9 \\ \cline{2-17}
	& & & & & & & $0\%$ & 45 & & & & & & & $0\%$ & 45 \\ 
	& & & & & & & $25\%$ & 33 & & & & & & & $25\%$ & 33 \\ 
	& & & & & & & $50\%$ & 24 & & & & & & & $50\%$ & 24 \\ 
	\multirow{-16}{*}{$M4$} & \multirow{-4}{*}{$T_{4}$} & \multirow{-4}{*}{$t_{44}$} & \multirow{-4}{*}{2} & \multirow{-4}{*}{N} & \multirow{-4}{*}{C} & & $75\%$ & 12 & \multirow{-4}{*}{$T_{8}$} & \multirow{-4}{*}{$t_{48}$} & \multirow{-4}{*}{2} & \multirow{-4}{*}{Y} & \multirow{-4}{*}{NC} & \multirow{-4}{*}{CF+DF}& $75\%$ & 12 \\	\hline		
  & & & & & & & $0\%$ & 22 & & & & & & & $0\%$ & 13 \\
  & & & & & & & $25\%$ & 8 & & & & & & & $25\%$ & 10 \\
  & & & & & & & $50\%$ & 10 & & & & & & & $50\%$ & 7 \\
& \multirow{-4}{*}{$T_{1}$} & \multirow{-4}{*}{$t_{51}$} & \multirow{-4}{*}{Y} & \multirow{-4}{*}{Y} & \multirow{-4}{*}{C} & & $75\%$ & 3 & \multirow{-4}{*}{$T_{5}$} & \multirow{-4}{*}{$t_{55}$} & \multirow{-4}{*}{N} & \multirow{-4}{*}{Y} & \multirow{-4}{*}{NC} & \multirow{-4}{*}{CF} & $75\%$ & 4 \\ \cline{2-17}	
 & & & & & & & $0\%$ & 11 & & & & & & & $0\%$ & 34 \\ 
 & & & & & & & $25\%$ & 16 & & & & & & & $25\%$ & 26 \\ 
 & & & & & & & $50\%$ & 10 & & & & & & & $50\%$ & 10 \\ 
 & \multirow{-4}{*}{$T_{2}$} & \multirow{-4}{*}{$t_{52}$} & \multirow{-4}{*}{N} & \multirow{-4}{*}{Y} & \multirow{-4}{*}{C} & & $75\%$ & 6 & \multirow{-4}{*}{$T_{6}$} & \multirow{-4}{*}{$t_{56}$} & \multirow{-4}{*}{2} & \multirow{-4}{*}{Y} & \multirow{-4}{*}{NC} & \multirow{-4}{*}{CF} & $75\%$ & 12 \\ \cline{2-17}
	& & & & & & & $0\%$ & 30 & & & & & & & $0\%$ & 11 \\ 
	& & & & & & & $25\%$ & 11 & & & & & & &$25\%$ & 8\\ 
	& & & & & & & $50\%$ & 16 & & & & & & & $50\%$ & 10 \\ 
	& \multirow{-4}{*}{$T_{3}$} & \multirow{-4}{*}{$t_{53}$} & \multirow{-4}{*}{2} & \multirow{-4}{*}{Y} & \multirow{-4}{*}{C} & & $75\%$ & 8 & \multirow{-4}{*}{$T_{7}$} & \multirow{-4}{*}{$t_{57}$} & \multirow{-4}{*}{N} & \multirow{-4}{*}{Y} & \multirow{-4}{*}{NC} & \multirow{-4}{*}{CF+DF} &  $75\%$ & 6 \\ \cline{2-17}
	& & & & & & & $0\%$ & 30 & & & & & & & $0\%$ & 15 \\ 
	& & & & & & & $25\%$ & 22 & & & & & & & $25\%$ & 22 \\ 
	& & & & & & & $50\%$ & 16 & & & & & & & $50\%$ & 8 \\ 
	\multirow{-16}{*}{$M5$} & \multirow{-4}{*}{$T_{4}$} & \multirow{-4}{*}{$t_{54}$} & \multirow{-4}{*}{2} & \multirow{-4}{*}{N} & \multirow{-4}{*}{C} & & $75\%$ & 8 & \multirow{-4}{*}{$T_{8}$} & \multirow{-4}{*}{$t_{58}$} & \multirow{-4}{*}{2} & \multirow{-4}{*}{Y} & \multirow{-4}{*}{NC} & \multirow{-4}{*}{CF+DF}& $75\%$ & 4 \\	%\hline		
	\bottomrule
	\end{tabular}
	\label{tab:synthetic_traces}
	}
	\caption{Synthetic trace characteristics}
\end{table}

Finally, for each of the $40$ traces, we considered five different levels of completeness\footnote{The execution traces have been made incomplete by removing events so as to preserve the non-compliance of traces (compliance is always preserved), as well as the capability for a solver to reconstruct the missing events.}, ranging from $100\%$ (i.e., the complete trace) to 0\% (i.e., the empty trace). Table~\ref{tab:synthetic_traces}, column \textit{Completeness} shows four level of completeness ($100\%$, $75\%$, $50\%$ and $25\%$) for each of the considered traces and the corresponding characteristics. In the table we did not report the information related to the empty traces ($0\%$ of completeness), which is the same for all the $8$ different types of traces.
For each of the three solvers, hence, we carried on, in total, $165$ runs ($5$ models, $8$ traces for each model and $4$ levels of incompleteness for each trace and model, as well as $5$ empty traces per model). 

We evaluated the first research question (\textbf{RQ1}) by computing, for each solver (i) the number/percentage of runs the solver is able to return; (ii) when a solution is returned, the time required by the solver for returning the solution. For \textbf{RQ2}, instead, we evaluated the metrics reported above, when changing factors, such as the level of incompleteness of the traces, the characteristics of the traces, as well as the model used for generating the traces.

\subsubsection{Results}
Only 2 runs out of the $495$ runs carried out for the synthetic datasets did not return a result because they exceeded the maximum time threshold. 
%Only $455$ runs out of the $495$ ones carried out for the synthetic datasets returned a result. The other runs were not able to provide an outcome, either because they exceeded the maximum time threshold or because of an out-of-memory problem.
Differently from \nuxmv, which exceeded the maximum time limit in $2$ cases ($\sim$$1\%$ of the runs), both \clingo and \fastdw  were always able to return results.
%Differently from \clingo, that was unable to return results in $38$ cases (around $23\%$ of the runs), \nuxmv was almost always and \fastdw always able to return results, respectively. In detail, \nuxmv exceeded the maximum time limit in only $2$ cases ($\sim$$1\%$ of the runs), none of which was among the $38$ cases that were problematic for \clingo.
%, while \fastdw was the only solver able to always return results. 
%\todocdfinline{@Sergio: Do we need/are we able to say something about why does \clingo exceed so often the time limit?}    

%\begin{table}
%	\centering
%		\scalebox{0.7}{
%		\begin{tabular}{c c c c}
%		\rowcolor{lightgray}\textbf{runs}& \clingo & \fastdw &  \nuxmv \\
%		%\rowcolor{lightgray}\multirow{-2}{*}{\textbf{total}} & \textbf{timed-out} & \textbf{timed-out} & \textbf{out-of-memory}\\
%		\toprule
%		\multirow{2}{*}{165} & \multirow{2}{*}{\btext{0}} & \multirow{2}{*}{0} & 2 \\
%		& & & (timed-out) \\ 
%		\bottomrule
%		\end{tabular}
%		}
%	\caption{Timed-out and out-of-memory runs for the three solvers}
%	\label{tab:synthetic}
%\end{table}

Figure~\ref{fig:synthetic}, besides the small percentage of timed-out runs of \nuxmv, 
%\btext{(}/out-of-memory\btext{)} runs,
 shows the average CPU time (and the corresponding variance) required by each of the three solvers to return results.
We can clearly observe that \clingo and \fastdw are the solvers with the best performance. Differently from \nuxmv, which fails in returning results within the time-out threshold for two runs, 
% a  result for two timed-out runs,
 the two solvers are able to return results for each of the runs. Moreover, \clingo and \fastdw have comparable performance in terms of time required for trace completion (see the overlapping lines of \clingo and \fastdw in Figure~\ref{fig:synthetic}). The time required by the two solvers (few seconds on average)
% besides being able to return results for each of the runs, differently from \nuxmv, which reports a time-out $2$ cases,
% the time that the two solvers require for completing incomplete traces (few seconds on average)
 is always lower than the time required by \nuxmv, which takes, on average, about $8$ minutes. 

This analysis allows us to assess that, overall, in case of synthetic datasets (of medium size), \clingo and \fastdw are the most reliable in finding solutions in a reasonable amount of time, and the fastest ones among the three considered solvers. \nuxmv has instead lower performance both in terms of capability to return results in a reasonable amount of time, as well as in terms of the average amount of time required to return computations (\textbf{RQ1}).

%This analysis allows us to assess that, overall, in case of synthetic datasets (of medium size), \fastdw is the most reliable and the fastest one among the three considered solvers. \nuxmv is the second best solver in terms of capability to return results in a reasonable amount of time, while \clingo, considering only the cases in which it returns a result, is the second fastest solver to return computations (\textbf{RQ1}).
%\fastdw is followed by \nuxmv in terms of capability to return results in a reasonable amount of time and by \clingo in terms of time required for returning computations - when it is able to  (\textbf{RQ1}).
%Overall, both in terms of reliability and required time, it is followed by \nuxmv and \clingo. 

In detail, Figure~\ref{fig:synthetic_completeness} shows the performance of the three solvers for different levels of incompleteness. By looking at the figure, we can observe that the level of incompleteness has an impact on \nuxmv: the more the trace is empty, the more the solver is able to return results (i.e., the less timed-out runs occur), the faster it is, and the lower the variance of the runs is. 
%Concerning \clingo and \fastdw, we can only notice, differently from \nuxmv, a decrease \clingo's timed-out runs for increasing levels of completeness. 

%In detail, Figure~\ref{fig:synthetic_completeness} shows the performance of the three solvers for different levels of incompleteness. By looking at the figure, we can observe that the level of incompleteness has an impact especially on \nuxmv: the more the trace is empty, the more the solver is able to return results (i.e., the less number of timed-out runs occur), the fastest it is and the lower the variance of the runs is. 
%Concerning \clingo and \fastdw, we can only notice, differently from \nuxmv, a decrease \clingo's timed-out runs for increasing levels of completeness. 

The plot in Figure~\ref{fig:synthetic_trace}, which reports solvers' performance for traces with different characteristics (the empty trace and the 8 types of traces considered), shows that the only runs for which \nuxmv is unable to return results relate to traces of type $T5$ and $T6$, i.e., traces non-compliant due to issues  related to the control flow. However, the characteristics of these two types of traces, do not seem to impact the average time required by \nuxmv to return results.  Differently from what happens for timed-out runs, the types of traces that require more time to \nuxmv seem to be $T3$ and $T8$.
%However, these are not the traces requiring on average more time to the \nuxmv solver.
%
% The average time required by the other two solvers, instead, is almost constant for each of the traces. Overall, we cannot identify a pattern related to the trace characteristics on none of the three solvers.}

%The plot in Figure~\ref{fig:synthetic_trace}, which reports solvers' performance for traces with different characteristics, shows that while all solvers are able to deal with trace $t7$, $t3$ and $t8$ are difficult to manage for \clingo (both in terms of timed-out runs and required time) and for \nuxmv (in terms of required time). In particular, the high variance reported for these two traces suggests that the time required highly depends on the specific run. Overall, we cannot identify a pattern related to the trace characteristics on none of the three solvers.

%Differently from \nuxmv and \clingo, none of the three investigated factors affects the performance of \fastdw, which is always able to return a solution in few seconds.

%Finally, Figure~\ref{fig:synthetic_model} investigates the impact of different models on the solvers. The plot suggests that all of them are able to manage well small models (e.g., $M1$), while both \clingo and \nuxmv have difficulties with larger models. In particular, these two solvers have a higher number of timed-out runs and require more time for returning results with model $M4$, which is the largest model, i.e., the model whose traces are always the longest traces (see Table~\ref{tab:synthetic_traces}).  
Finally, Figure~\ref{fig:synthetic_model} investigates the impact of different models on the solvers. The plot suggests that all the solvers are able to manage well small and medium models (e.g., $M1$, $M2$ and $M3$). \nuxmv has instead difficulties with larger models. In particular, both the two timed-out runs relate to model $M4$, which is the largest model, i.e., the model whose traces are always the longest traces (see Table~\ref{tab:synthetic_traces}). Moreover, the runs related to this model that are able to complete, require more time than other runs for returning results.

Differently from \nuxmv, none of the three investigated factors affects the performance of \clingo and \fastdw, which are always able to return a solution in few seconds.

\begin{figure}
\centering
\subfigure[Trace completeness]{
\label{fig:synthetic_completeness}
\includegraphics[width=0.45\textwidth]{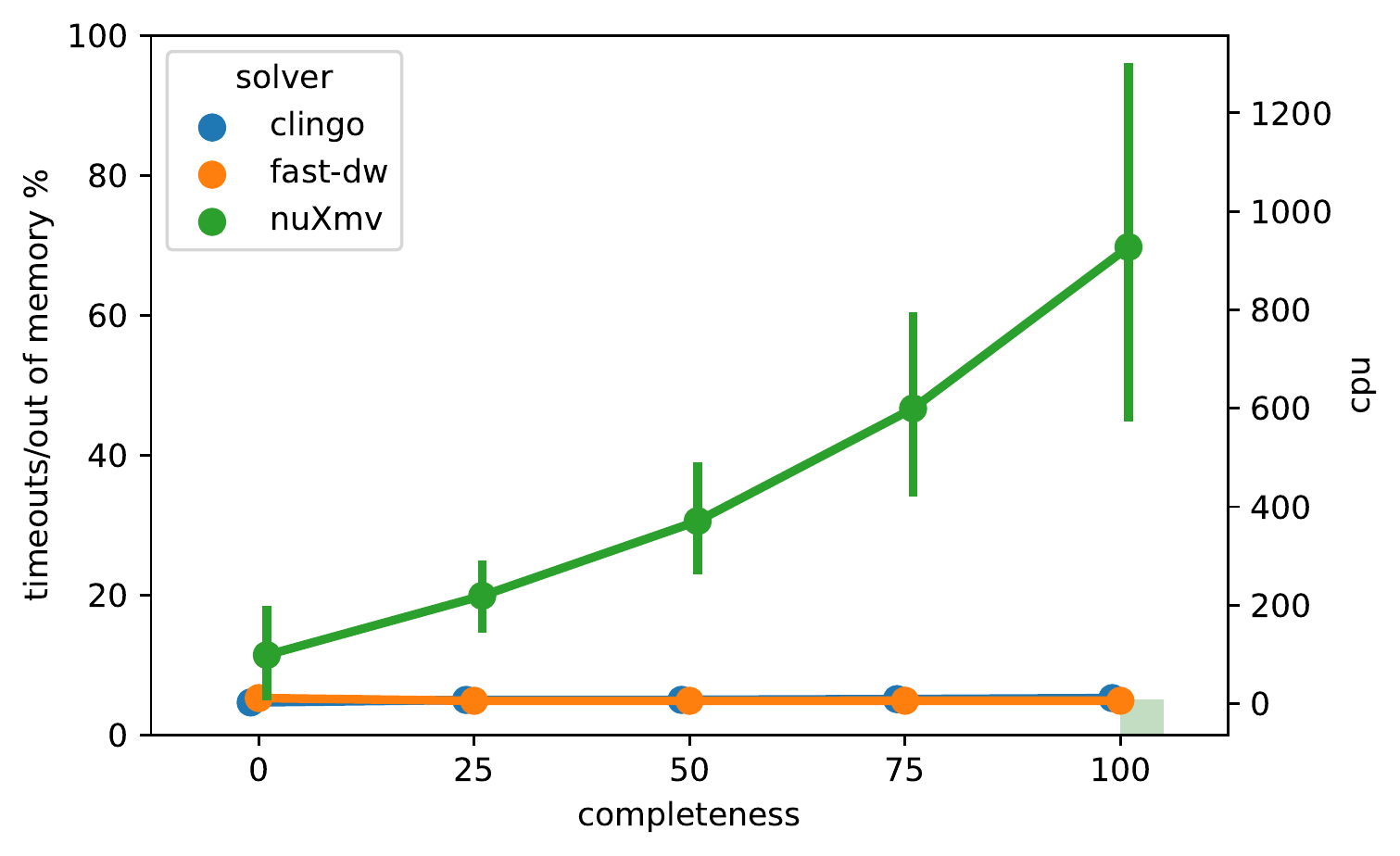}}%
\qquad
\subfigure[Trace type]{%
\label{fig:synthetic_trace}%
\includegraphics[width=0.45\textwidth]{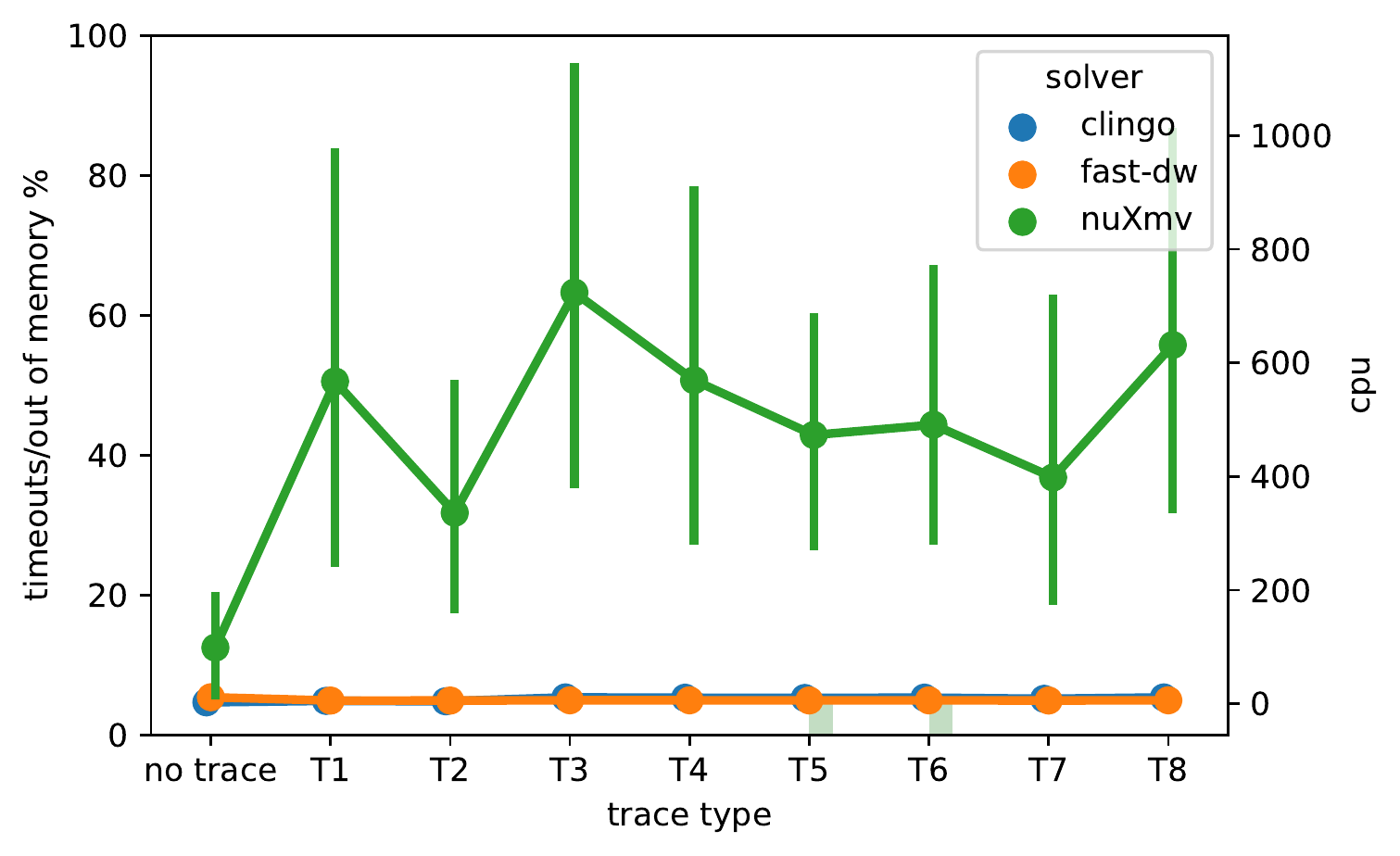}}\\%
\subfigure[Model type]{%
\label{fig:synthetic_model}%
\includegraphics[width=0.45\textwidth]{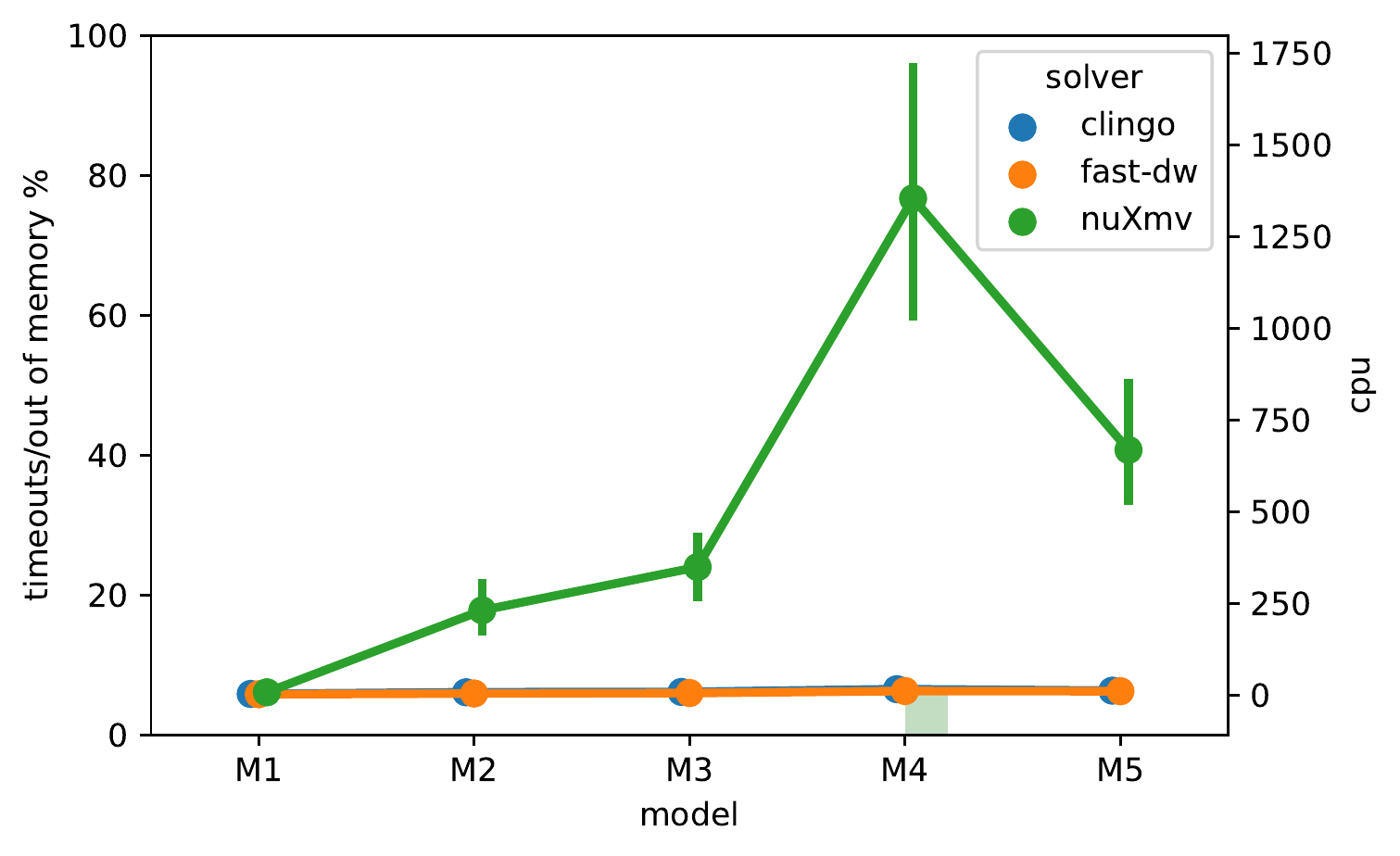}}%
\caption{Solver performance for different levels of completeness, different traces and different models. In the plots, lines correspond to CPU time (right scale) and bars to interrupted computations.}\label{fig:synthetic}
\end{figure}

Overall, we can assess that, in case of synthetic datasets (of medium size), trace completeness has a strong impact on \nuxmv, which performs better on empty traces. \nuxmv performance are also influenced by the model size and, in particular, by the length of the traces that it generates: the larger the model and the longer the traces, the worse the performance of the solver. On the contrary, none of the investigated factors has an impact on the \clingo and \fastdw performance (\textbf{RQ2}).

\subsection{Real-life Log Evaluation}
\label{ssec:real-life_eval}
In order to evaluate the capability of the three solvers to deal with real-life settings, we exercised them on a real world event log. In the next subsections we first describe the considered dataset, the procedure and the metrics used for its evaluation; we then show the obtained results.

\subsubsection{Dataset, procedure and metrics}
In order to exercise the three solvers in a real-life setting, we resorted to the BPI Challenge 2011~\cite{bpichallenge2011} (BPIC2011) event log. The log pertains to a healthcare process and, in particular, contains the executions of a process related to the treatment of patients diagnosed with cancer in a large Dutch academic hospital. The whole event log contains $1143$ execution cases and $150291$ events distributed across $623$ event classes (activities). Each case refers to the treatment of a different patient. The event log contains domain specific attributes that are both case attributes and event attributes. For example, $Age$, $Diagnosis$, and $Treatment\_code$ are case attributes and $Activity\_code$, $Number\_of\_executions$, $Producer\_code$, $Section$ and $Group$ are event attributes.

%In order to exercise the three reasoners on a real-life setting, we resorted to the BPI Challenge 2012~\cite{bpichallenge2012} event log. The event log is the log of an application process for a personal loan or overdraft of a Dutch Financial Institute. The whole event log contains $13087$ cases, i.e., execution traces, and $262200$ events distributed across $35$ event classes (activities).
%The event log contains domain specific attributes that are both case attributes and event attributes. For example, $Age$, $Diagnosis$, and $Treatment\_code$ are case attributes and $Activity\_code$, $Number\_of\_executions$, $Specialism\_code$, and $Group$ are event attributes.

%In order to be able to use the BPI Challenge 2011 event log in our evaluation, we extracted the corresponding data-aware Petri Net model, as well as a set of incomplete traces by applying the following procedure:
The inputs required for accomplishing the task of trace completion are a process model and incomplete traces. We hence extracted a data-aware Petri Net model and we generated a set of incomplete traces by applying the following procedure:
%In order to let the three solvers dealing with the task of completing incomplete traces, we need a process model and incomplete traces. We hence extracted a data-aware Petri Net model and we generated a set of incomplete traces by applying the following procedure:
\begin{enumerate}
  \setlength\itemsep{0em}
\item We discovered the data-aware Petri net from the BPIC2011 event log by applying the ProM \textsc{Data-flow Discovery} plugin~\cite{deLeoni2013}. \footnote{As input for the \textsc{Data-flow Discovery} plugin we used the Petri Net discovered with the \textsc{Inductive Miner} algorithm by setting the \emph{Noise threshold} parameter of the plugin to $0.20$.} The discovered data-aware Petri Net has $355$ transitions, $61$ places and $710$ edges and $4$ variables ($Activity\_code$, $Producer\_code$, $Section$ and $Group$), which are read and written by the Petri net transitions and constrained in $25$ guards throughout the data-aware Petri net. An example of a discovered guard is reported in~\eqref{eq}.
\begin{equation}\label{eq}
{\scriptstyle(Producer\_code="CHE2") \&\& (Activity\_code=="370422")}
\end{equation}
\item For the testing set generation, we randomly selected $9$ complete traces - details on the traces and their length are reported in columns \emph{Trace ID} and \emph{Trace length} in Table~\ref{tab:bpi_traces} - from the dataset.
\item For each of the $9$ randomly selected traces, $3$ incomplete traces - obtained by randomly removing $25\%$, $50\%$ and $75\%$ of the events 
%- details on the length of the incomplete traces are reported in Table~\ref{tab:bpi_traces} -
 are added to the testing set.
\item The empty trace is added to the testing set. 
\end{enumerate}
%For evaluating the three solvers we executed in total $37$ runs per solver. Indeed, we tested $9$ different traces for $4$ different levels of completeness (plus the empty trace) for each solver. 
We evaluated the $3$ solvers on each trace of the testing set for a total of $37$ runs per solver ($9$ different traces for $4$ different levels of completeness and the empty trace).

\begin{table}
\parbox{.3\linewidth}{
\centering
	\scalebox{0.7}{
		\begin{tabular}{|a|c|}
		\toprule
		\rowcolor{lightgray}\textbf{Trace ID} & \textbf{Trace length} \\	\hline
		$t207$ & 14  \\ \hline
		$t296$ & 515 \\ \hline
		$t365$ & 30  \\ \hline
		$t533$ & 565 \\ \hline
		$t610$ & 127 \\ \hline
		$t679$ & 5  \\ \hline
		$t729$ & 3  \\ \hline
		$t847$ & 56  \\ \hline
		$t1132$ & 34 \\
		\bottomrule
		\end{tabular}
		}
	\caption{BPIC2011 testing set trace length}
	\label{tab:bpi_traces}
	}
	\hfill
\parbox{.65\linewidth}{
		\centering
		\scalebox{0.7}{
		\begin{tabular}{l c c c c}
		\rowcolor{lightgray}\textbf{Scenario} & \textbf{total}& \clingo & \fastdw & \nuxmv \\
		%\rowcolor{lightgray}\multirow{-2}{*}{\textbf{total}} & \textbf{timed-out} & \textbf{timed-out} & \textbf{out-of-memory}\\
		\toprule
		\multirow{2}{*}{\full} & \multirow{2}{*}{37} & 37 & 37 & 22\\
		& & (out-of-memory) & (out-of-memory) & (timed-out)\\ \hline
		\multirow{2}{*}{\restricted} & \multirow{2}{*}{37} & 35 & 37 & 12\\
		& & (out-of-memory) & (out-of-memory) & (timed-out\\ \hline
		\multirow{2}{*}{\nodata} & \multirow{2}{*}{37} & 20 & \multirow{2}{*}{0} & 12 \\
		& &(out-of-memory) &  & (timed-out)\\
		\bottomrule
		\end{tabular}
		}
	\caption{Timed-out and out-of-memory runs for the three solvers on the BPIC2011}
	\label{tab:bpi}
	}
	
\end{table}

%\begin{table}
%\centering
%	\scalebox{0.7}{
%		\begin{tabular}{|a|c|c|c|c|}
%		\toprule
%		& \multirow{2}{*}{\textbf{Trace length}} & \textbf{75\%-complete} & \textbf{50\%-complete} & \textbf{25\%-complete} \\	
%		\multirow{-2}{*}{\textbf{Trace ID}}& & \textbf{trace length} & \textbf{trace length} & \textbf{trace length} \\	\hline
%		$t207$ & 14 & 10 & 7 & 3 \\ \hline
%		$t296$ & 515 & 386 & 257 & 129 \\ \hline
%		$t365$ & 30 & 22 & 15 & 7 \\ \hline
%		$t533$ & 565 & 424 & 282 & 141 \\ \hline
%		$t610$ & 127 & 95 & 63 & 32 \\ \hline
%		$t679$ & 5 & 4 & 2 & 1 \\ \hline
%		$t729$ & 3 & 2 & 1 & 1 \\ \hline
%		$t847$ & 56 & 42 & 28 & 14 \\ \hline
%		$t1132$ & 34 & 25 & 17  & 8 \\
%		\bottomrule
%		\end{tabular}
%		}
%	\caption{BPIC2011 testing set trace length}
%	\label{tab:bpi_traces}
%\end{table}

Also in this case, we evaluated \textbf{RQ1} by computing, for each solver (i) the number/percentage of runs the solver is able to return; (ii) the time required for returning the solution. For the second research question (\textbf{RQ2}), instead,  
%we evaluated the metrics above, while changing factors, such as the incompleteness of the traces, the type of traces, as well as the model generating the traces. 
%
%Being interested in evaluating the factors impacting the results, and, in particular, with a large real-life dataset as the BPI Challenge 2011 the impact of the data domain and data in general, we executed the experiments in three different settings:
%In the evaluation,
 we focused on the factors that could impact the results. Besides looking at the degree of trace completeness and at different traces, we also investigated the impact of data and of its domain on the performance of the solvers on large real-life datasets. To this aim, we evaluated the three solvers by computing the metrics reported above in three different settings: 
\begin{itemize} 
  \setlength\itemsep{0em}
\item \full: in this setting we considered the data payloads associated to events and we did not apply any restriction to their domain;
\item \restricted: in this setting we reduced the domain of the data to the only constant values occurring in the model;~\footnote{Please notice that we were able to reduce the domain of the variables to the sole constant values, as we only have guards with equality constraints.}
\item \nodata: in this setting we do not use data at all.
\end{itemize}
For each setting, we performed the $37$ runs described above, we collected the metrics and we compared the results.

\subsubsection{Results}

%\begin{table}
%	\centering
%		\scalebox{0.7}{
%		\begin{tabular}{l c c c c}
%		\rowcolor{lightgray}\textbf{Scenario} & \textbf{total}& \clingo & \nuxmv & \fastdw \\
%		%\rowcolor{lightgray}\multirow{-2}{*}{\textbf{total}} & \textbf{timed-out} & \textbf{timed-out} & \textbf{out-of-memory}\\
%		\toprule
%		\multirow{2}{*}{\full} & 37 & 22 & 37 & 37\\
%		& & (timed-out) & (timed-out) & (out-of-memory)\\ \hline
%		\multirow{2}{*}{\restricted} & 37 & 12 & 37 & 37\\
%		& & (timed-out) & (timed-out) & (out-of-memory)\\ \hline
%		\multirow{2}{*}{\nodata} & 37 & 12 & 12 & \multirow{2}{*}{0}\\
%		& & (timed-out) & (timed-out) & \\
%		\bottomrule
%		\end{tabular}
%		}
%	\caption{Timed-out and out-of-memory runs for the three solvers on the BPIC2011}
%	\label{tab:bpi}
%\end{table}

Table~\ref{tab:bpi} reports, for each of the three scenarios, the number of timed-out and out-of-memory runs for the three solvers on the BPIC2011. When applying the three solvers to the task of completing incomplete traces of a large real-life event log in the \full scenario, the only solver able to return results (in $15$ runs out of the $37$ ones carried out) is \nuxmv; \clingo and \fastdw fail in each of the executed runs due to out-of-memory issues.
%while \clingo fails in each of the executed runs because of timeout violations, \fastdw fails due to out-of-memory issues.
%, both \clingo and \fastdw were unable to return any result. While the former failed in each of the executed runs due to timeout violations, the latter failed for out-of-memory issues.
%Indeed, when applied to large datasets, the \fastdw grounding phase, which likely allows the solver to complete the runs in short time for small datasets, is unable to complete due to a memory space explosion.  
Indeed, when applied to large datasets, the \clingo and \fastdw grounding phase, which likely allows the solver to complete the runs in short time for small datasets, is unable to complete due to a memory space explosion.  

%\todocdfinline{@Sergio: Could you check if this is fine?}
%that on large datasets, the memory of \fastdw, which leverages grounding that , is unable to complete the grounding phase, thus failing before starting the exploration on the grounded variables. 
%It seems that the causes for this out-of-memory problems in \fastdw, which is a reasoner based on a grounding mechanism, fails before completing 

\begin{figure}[t]
\centering
\includegraphics[width=.45\textwidth]{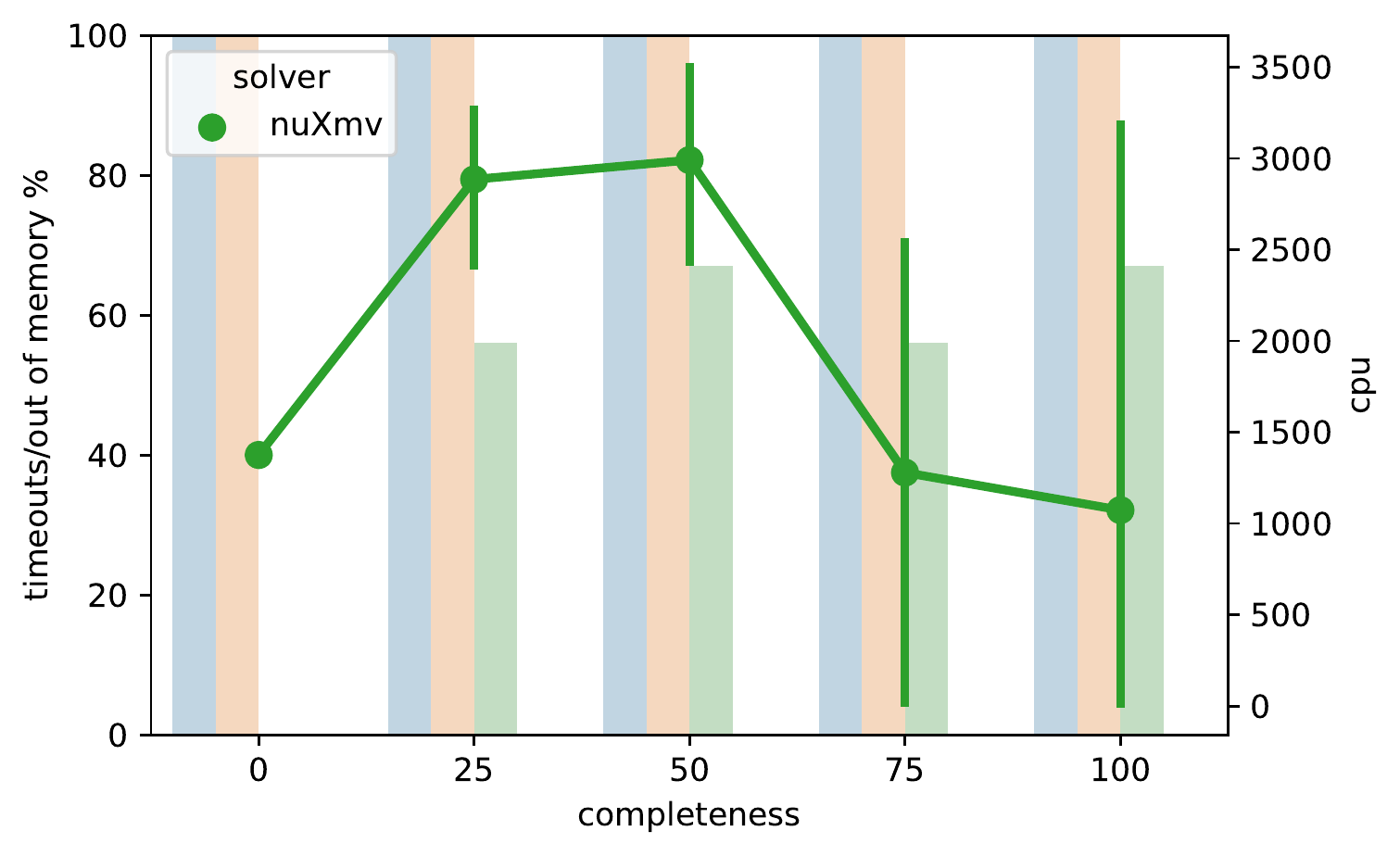}
\includegraphics[width=.45\textwidth]{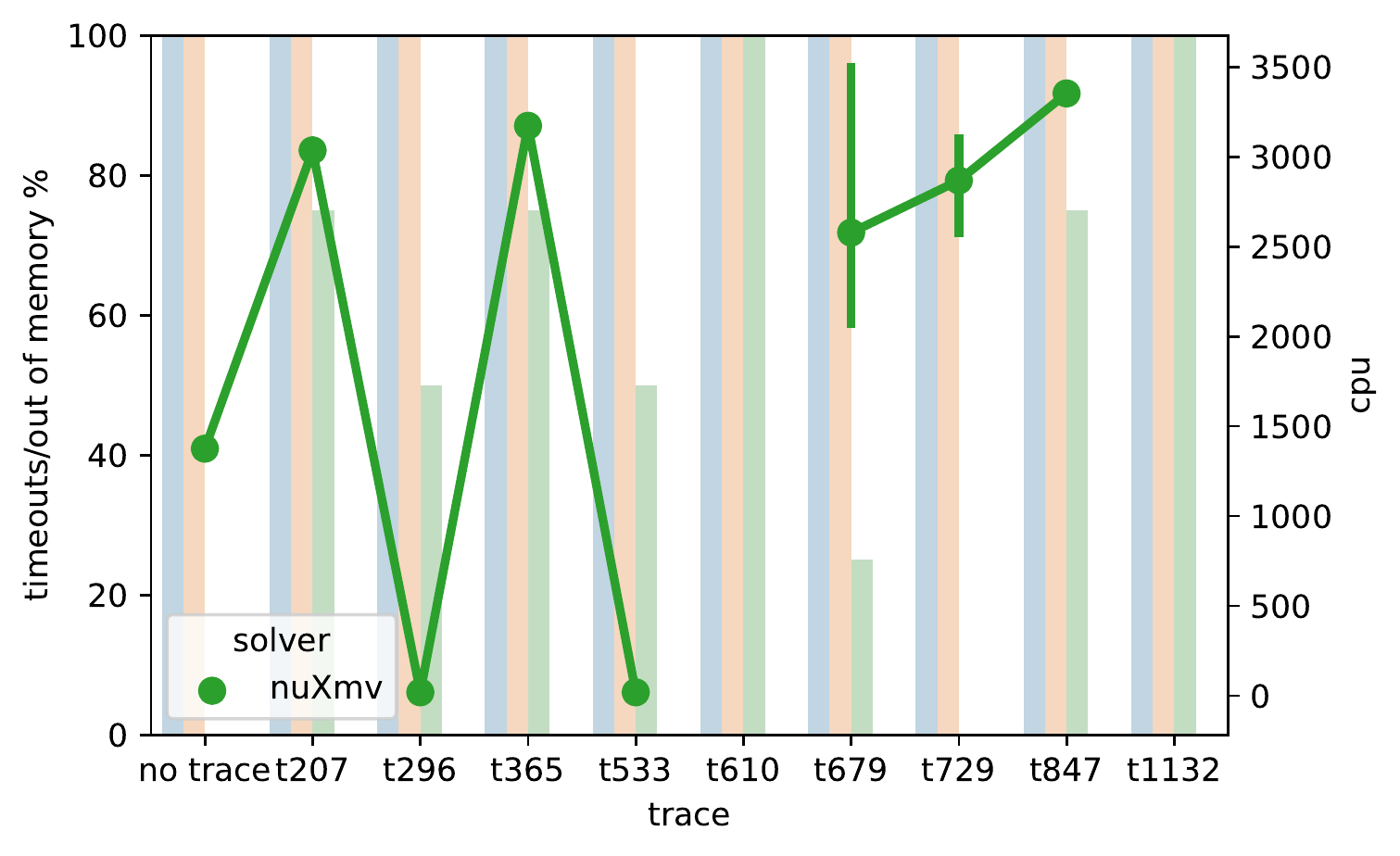}
\caption{Solver performance for different levels of completeness and different traces of the BPIC2011 event log in the \full scenario}
\label{fig:bpi_full}
\end{figure}

Figure~\ref{fig:bpi_full} plots the performance (percentage of timed-out/out-of-memory runs and average time) of the three solvers for different levels of trace completeness and for the different considered traces (the empty trace and the $9$ traces of the BPIC2011). Overall, no general trend or pattern related to the trace completeness and to the trace type can be identified for the $15$ runs for which \nuxmv was able to return a result and the required time ranges from $\sim$$20$ seconds to one hour. 
%Figure~\ref{fig:bpi_full} shows the performance of the three solvers on the real-life dataset for different levels of trace completeness and for the $9$ traces and the empty trace in the \full scenario. The plot shows that the only solver able to return results is \nuxmv: it was able to complete incomplete traces only in $17$ runs out of the $37$ ones carried out for each solver. While \clingo was not able to return any result due to timeout violations, all runs carried out with \fastdw registered an out-of-memory. 
%By focusing on the time required by \nuxmv to return results in these $17$ cases, no general trends or patterns related to the trace completeness and to the trace type can be observed.
Differently from the synthetic datasets, an opposite trend can be observed with large real-life logs: performance improves with more complete traces. For instance, the timed-out runs and the average time required for 75\%-complete traces are less than the timed-out runs and the average time required for the 50\%-complete traces. Moreover, we can notice that for a couple of traces - $t610$ and $t1132$ -\nuxmv is also not able to provide any result, as the other two solvers. 
\begin{figure}
\centering
\includegraphics[width=.45\textwidth]{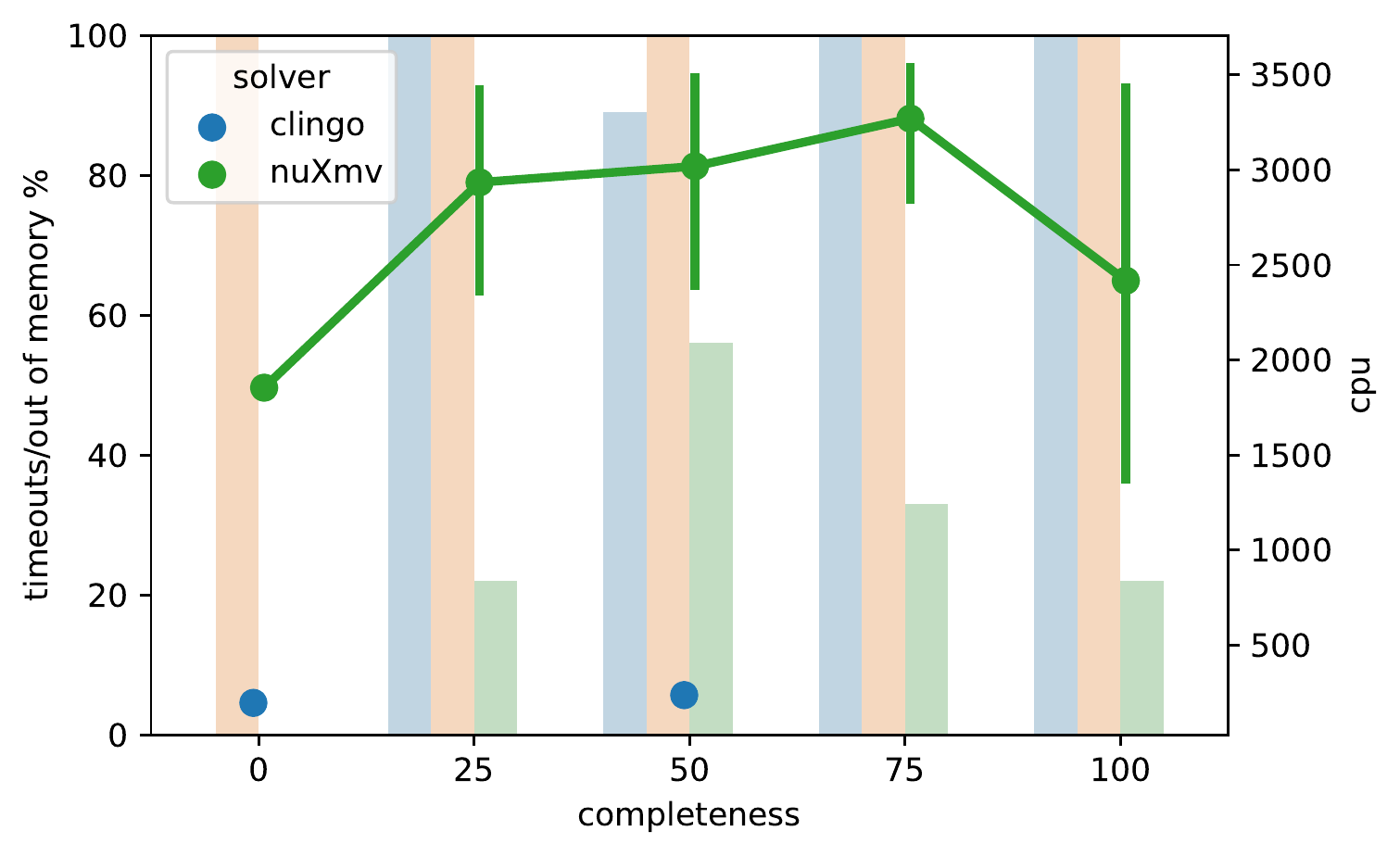}
\includegraphics[width=.45\textwidth]{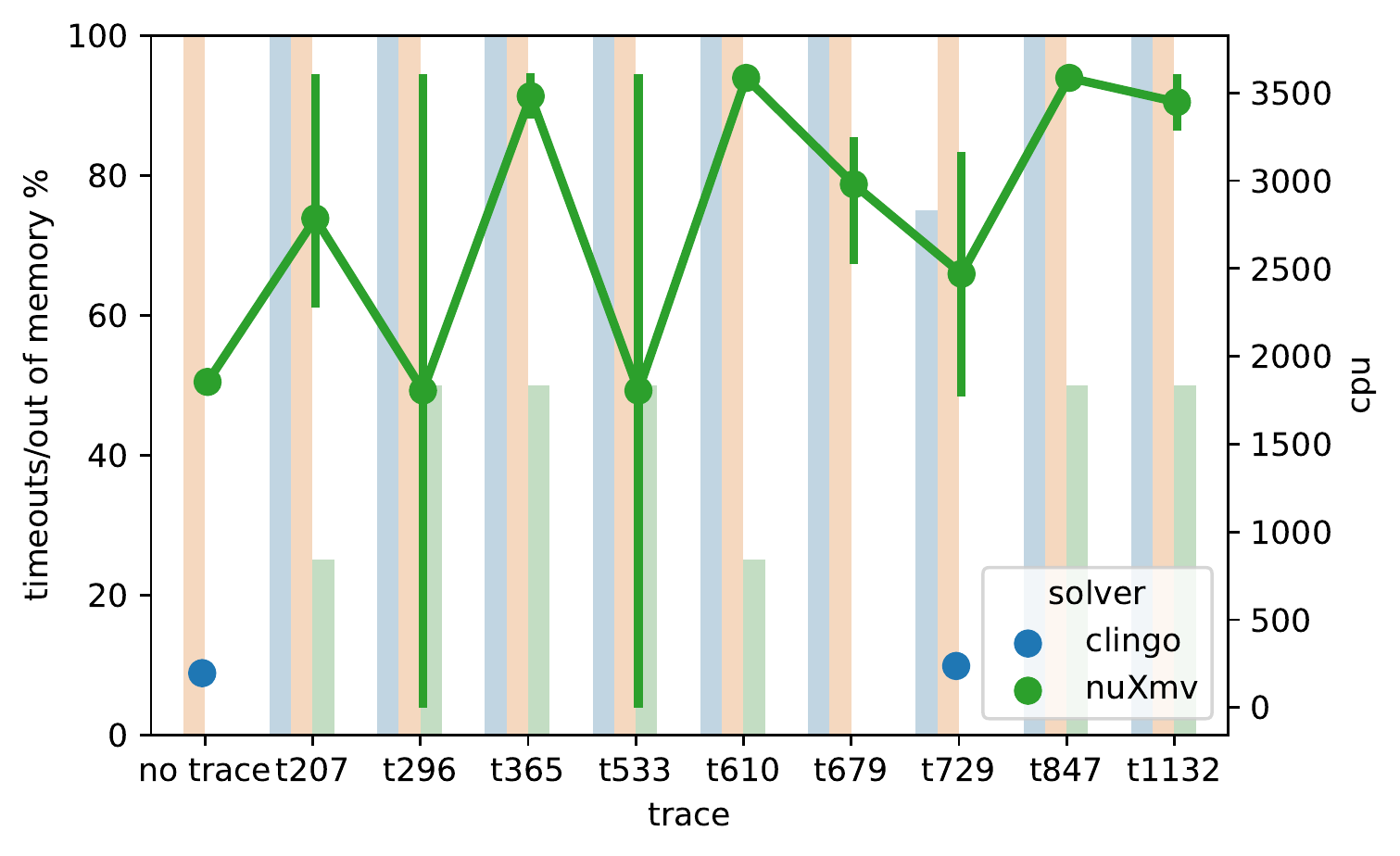}
\caption{Solver performance for different levels of completeness and different traces of the BPIC2011 event log in the \restricted scenario}
\label{fig:bpi_restricted}
\end{figure}

%Hypothesizing that the reasons behind such an inability of the reasoners to deal with real-life log could be to be reconducted to the huge domain characterizing the data attributes of the event log, we evaluated the performance of the three reasoners with a restriction on the domain. In detail, since we have only equality conditions, we were able to restrict the domain to the only constant variables. 
Figure~\ref{fig:bpi_restricted} shows the obtained results for the \restricted scenario.  Despite the domain restriction, \clingo and \fastdw are still unable to return results, while an improvement can be observed in the capability of \nuxmv to return results: the number of timed-out runs decreases of about $50\%$ (from $22$ to $12$ runs), as reported in Table~\ref{tab:bpi}. The trend observed for the timed-out runs is similar to the one of the \full scenario (the percentage of timed-out runs is null for the empty traces and reaches its peak for $50\%$-complete traces), while some differences are registered in terms of required time (e.g., the average time required for $75\%$-complete traces increases almost of a factor of $3$ with respect to the \full scenario). If we look at different traces, trends are instead quite different from the \full scenario. For instance, in the \restricted scenario, \nuxmv was unable to deal only with $20\%$ of the runs for the trace $t610$, which was a problematic trace in the \full scenario. On the contrary, the percentage of timed-out runs for trace $t296$ ($50\%$) remains unchanged moving from the \full to the \restricted scenario.

\begin{figure}
\centering
\includegraphics[width=.45\textwidth]{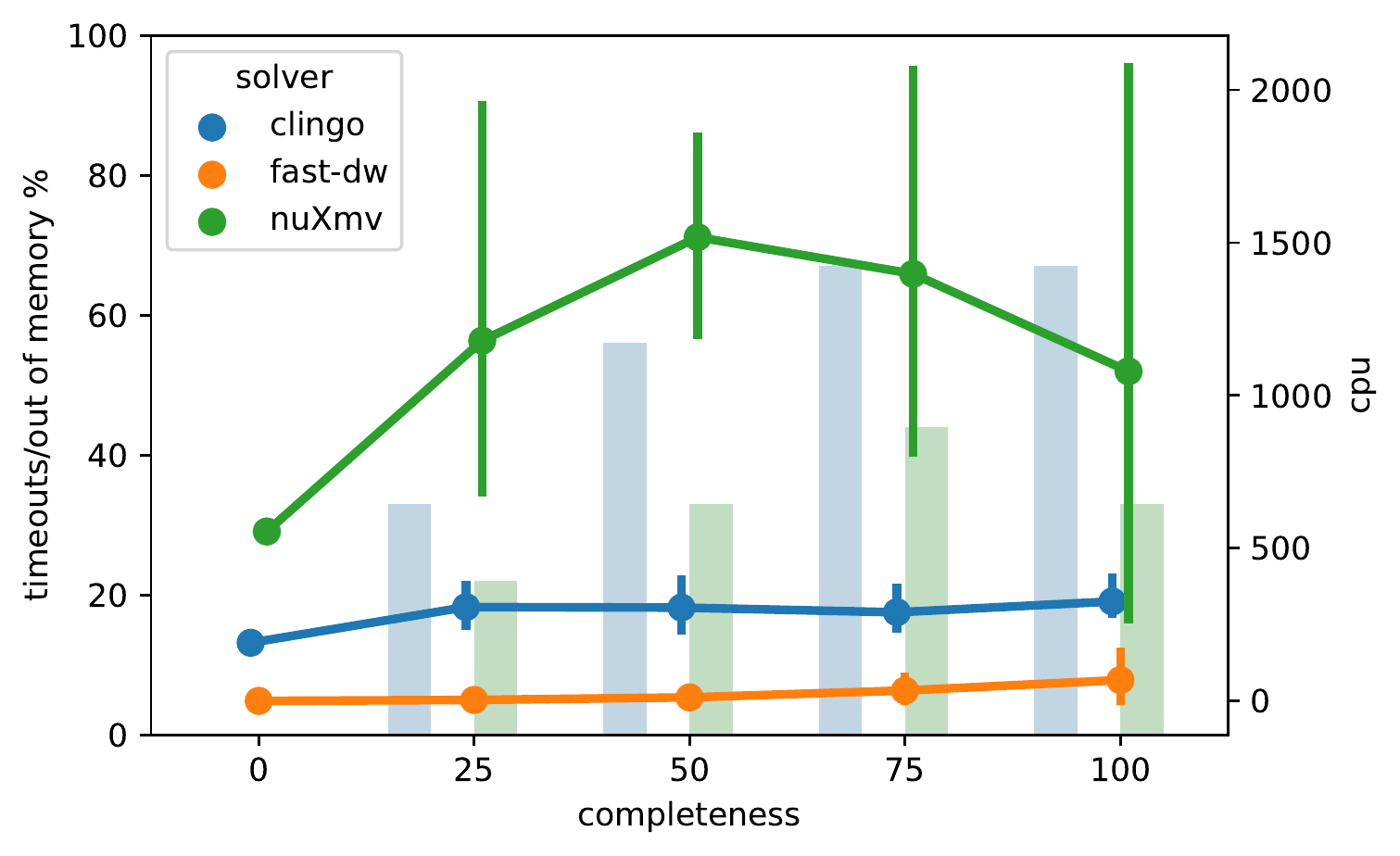}
\includegraphics[width=.45\textwidth]{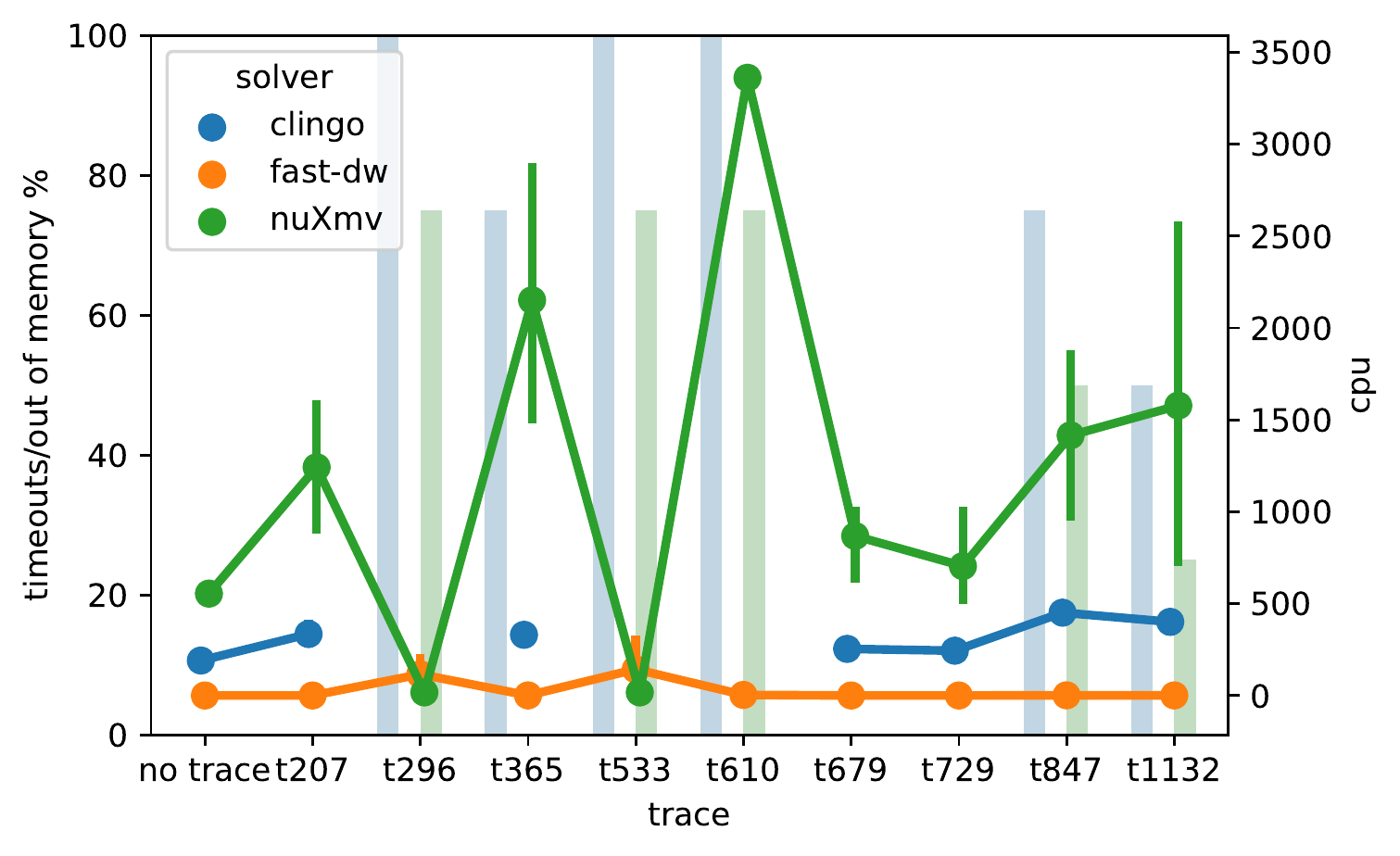}\caption{Solver performance for different levels of completeness and different traces of the BPIC2011 event log in the \nodata scenario}
\label{fig:bpi_no_data}
\end{figure}

%\begin{table}
%	\centering
%		\scalebox{0.7}{
%		\begin{tabular}{c c c c}
%		\rowcolor{lightgray}\textbf{total}& \clingo & \nuxmv & \fastdw \\
		%\rowcolor{lightgray}\multirow{-2}{*}{\textbf{total}} & \textbf{timed-out} & \textbf{timed-out} & \textbf{out-of-memory}\\
%		\toprule
%		37 & 12 & 12 & \multirow{2}{*}{0}\\
%		& (timed-out) & (timed-out) & \\
%		\bottomrule
%		\end{tabular}
%		}
%	\caption{Timed-out and out-of-memory runs for the three solvers in the \nodata scenario}
%	\label{tab:bpi_no_data}
%\end{table}

Finally, the three solvers have been evaluated in the \nodata scenario. In this setting, each of the three solvers is able to return some results (see the last row in Table~\ref{tab:bpi}). When data are not taken into account, the performance of the solvers are close to the results obtained with the syntetic datasets: \fastdw is the only one able to deal with all the $37$ runs, followed by \nuxmv registering $12$ timed-out runs and, finally, by \clingo with $20$ out-of-memory runs. Moving from a \restricted to a \nodata scenario has hence a strong impact on the capability of \fastdw and \clingo to return results, while this is not the case for \nuxmv, whose number of timed-out runs remains the same of the \restricted scenario. 
%\fastdw is the only one able to deal with all the $37$ runs, while both \clingo and \nuxmv register $12$ timed-out runs. Moving from a \restricted to a \nodata scenario has hence a strong impact on the capability of \fastdw and \clingo to return results, while this is not the case for \nuxmv, whose number of timed-out runs remains the same of the \restricted scenario. 

The plots in Figure~\ref{fig:bpi_no_data} report, besides the percentage of timed-out and out-of-memory runs, also the time required by the solvers for the returned runs. As for the synthetic datasets, \fastdw is the solver with the best performance (it takes from few seconds to few minutes) followed by \clingo ($\sim$$300$ seconds, for the runs it is able to complete) and \nuxmv ($\sim$$20$ minutes with a high variance).  The time required by \fastdw seems to slightly depend on the level of trace incompleteness: the more the trace is complete, the larger the encoded net, the more time \fastdw requires to complete the task. 

Overall, when dealing with real-life datasets enriched with data (both with and without domain restrictions), \nuxmv is the best solver, while \clingo and \fastdw are unable to deal with this scenario. When, instead, data are not taken into account, \fastdw outperforms the other two solvers both in terms of returned results and of required time. Although \nuxmv fails in returning results in less runs (about half) than \clingo, \clingo has better performance than \nuxmv in terms of required time (\textbf{RQ1}).
%Although both \nuxmv and \clingo fail in returning results in about half of the runs, \clingo has the best performance in terms of required time (\textbf{RQ1}).
Moreover, although no general pattern and trend can be identified with respect to trace incompleteness and trace type, data has a strong impact on  \clingo and \fastdw and the data domain affects the capability of \nuxmv of returning results (\textbf{RQ2}).

%\subsection{Discussion}
%\label{ssec:discussion}
%Given the above considerations we can provide the following answers to the two research questions. 
Summarizing the above results, we can conclude that \nuxmv is the most robust approach as it is able to deal reasonably well with very complex settings as real-life event logs with data. On the other hand, \fastdw, which on large search spaces ends up in out-of-memory errors, on synthetic and on real-life medium-size search spaces
%both synthetic and real-life medium-size search spaces
 outperforms its competitors, carrying on the trace completion task in few seconds (\textbf{RQ1}).
Concerning \textbf{RQ2}, for each of the three solvers, no global trends or patterns related to the level of trace completeness and to the characteristics of the trace have been identified. The only exception is \nuxmv and its capability to better deal with incomplete rather than with complete traces for synthetic event logs and in  part viceversa for large real-life event logs. The size of the resulting search space and, in particular, leveraging data, has a huge impact on the difficulties of \clingo and \fastdw to solve the trace completion task.

We remark that, although the focus of the evaluation is on the specific task of trace completion for
%completing incomplete traces of
 data-aware workflows, 
%since the proposed solution is based on reachability
%since the proposed solution relies on the transformation of the completion problem into a rechability one 
since the task is recasted to a reachability problem
 (see Section~\ref{sec:encoding:traces}), the problem we are tackling is actually a wider one.
%, i.e., the reachability problem.
 The evaluation carried out shows the performance of the three solvers on the reachability problem for large automata, as for instance, the one obtained from the analysis of the \ournet discovered from the BPI2011 event log.

\subsection{Threats to Validity}
\label{ssec:threats}
The main threats affecting the validity of the reported results are \emph{external} validity issues, which hamper the generalization of the  findings. Indeed, we evaluated the solvers on a single real-life log and tested them only on a small number of log traces. Nevertheless, we tried to mitigate the above threat, by carrying on a systematic evaluation on the synthetic datasets and randomly selecting the subset of traces in the real-life event log used in the evaluation. A second threat to the external validity of the results relates to the choice of the data and of the guards. However, also in this case, we tried to reduce the threat by considering different data attributes and values in the synthetic evaluation, and by relying on data and guards discovered by the ProM \textsc{Data-flow Discovery} plugin. 
Finally, 
%we have an \emph{internal} validity threat. We indeed 
we used the solvers as informed users, i.e., we did not apply to them special optimizations. 
%without applying special optimizations to them.
 Dedicated encodings and tuning strategies could improve the performance of the solvers. However, the threat is mitigated by the fact that none of the three solvers has been optimized.
% we did not apply optimizations to none of the three solvers.}

%% file: relatedworks.tex
%!TEX root = ./main.tex

\section{Related Work}
\label{sec:related_works}

%verification and reachability
In the field of verification techniques for data-aware processes, a number of theoretical works exist both in the area of data-aware processes and of (variants of) PNs. Unfortunately, when combining processes and data, verification problems suddenly become undecidable \cite{calvanese-13-PODS-keynote}.
We can divide this literature in two streams. In the first stream, PNs are enriched, by making tokens able to carry various forms of data and by making transitions aware of such data, such in CPNs~\cite{vanderAalst:2011:MBP:2000715} or in data variants such as (Structured) Data Nets~\cite{Badouel:2015aa,Lazic:2007aa}, $\nu$-PNs~\cite{RosaVelardo20114439} and Conceptual WF-nets with data~\cite{sidorovastahletal:2011}. 
%variants of PNs are enriched, by making tokens able to carry various forms of data, and by making transitions aware of such data, such in CPNs~\cite{vanderAalst:2011:MBP:2000715} or data variants such as (Structured) Data Nets~\cite{Badouel:2015aa,Lazic:2007aa}, $\nu$-PNs~\cite{RosaVelardo20114439} and Conceptual WF-nets with data~\cite{sidorovastahletal:2011}. 
For full CPNs, reachability is undecidable and usually obtained by imposing finiteness of color domains. Data variants instead weaken data-related aspects. Specifically Data Nets and $\nu$-PNs consider data as unary relations, while semistructured data tokens are limited to tree-shaped data structures. Also, for these models coverability is decidable, but reachability is not. The work in \cite{sidorovastahletal:2011} considers data elements (e.g., \emph{Price}) that can be used on transitions' preconditions. 
However, these nets do not consider data values (e.g., in the example of Section~\ref{sec:trace_repair} we would not be aware of the values of the variable $request$ that \emph{T4} is enabled to write) but only whether the value is ``defined'' or ``undefined'', thus limiting the reasoning capabilities that can be provided on top of them. For instance, in the example of Section~\ref{sec:trace_repair}, we would not be able to discriminate between the worker and the student loan for the trace in \eqref{eq:trace-data}, as we would only be aware that $request$ is \texttt{defined} after \emph{T4}. 
The second stream contains proposals that take a different approach: instead of making the control-flow model increasingly data-aware, they consider standard data models (such as relational databases and XML repositories) and make them increasingly ``dynamics-aware''.
Notable examples are relational transducers~\cite{DBLP:journals/jcss/AbiteboulVFY00}, active XML~\cite{AbSV09}, the artifact-centric paradigm \cite{Geredeetal2007,DaDV11,bagheri-13-PODS-dcds}, and  DCDSs~\cite{bagheri-13-PODS-dcds}.
Such works differ on the limitations imposed to achieve decidability, but they all lack an intuitive control-flow perspective.
% In \cite{DaDV11} decidability is achieved by restricting the query language over the part of database where new data can accumulate. Interestingly, the system can be verified independently from the initial data state.
%%
A further recent work presents \rawsys~\cite{De-Masellis:2017:aaai}, a framework that joins the two streams above by directly combining a control-flow model based on PNs and standard data models (\`{a} la DCDS) as first class citizens, in which activities/tasks are expressed in a STRIP-like fashion. Verification in \rawsys is shown to be decidable when the number of objects accumulated in the same state is bounded (but still infinitely many values may occur in a run) and first-order quantification is restricted on the active domain of the current state. The definition of \ournet follows the approach of \rawsys in combining control-flow model based on PNs and standard data models (\`{a} la DCDS) as first class citizens. However, the current work differentiates from that of \rawsys in that we consider a data model with variables only, which, on the one hand, does not require complex restrictions to achieve decidability of reasoning tasks and, on the other hand, is still expressive enough to address the trace completion problem with events carrying data. We indeed remark that the payload of XES standard for traces is a set of attribute-value pairs, thus making data-aware process models with relational data structures needlessly cumbersome for solving reasoning tasks on concrete logs.

The VERIFAS system~\cite{liVERIFASPracticalVerifier2017a} leverages model checking techniques for the verification of temporal properties of infinite-state transition systems arising from processes that carry and manipulate unbounded data. Although in principle the system can be used to verify reachability properties of \ournet models, it's difficult to compare because the representation language for the models is more expressive, and the language for specifying the properties that can be verified can be used for more than reachability. Moreover the software used for their experiments~\cite{liHasverifierImplementationVerifiers2017} cannot be used directly for our purposes and it would require nontrivial modifications in order to be included in our experiments.

Moving to the exploitation of verification techniques in the BPM field, 
%in the context of business processes,
 we can notice that planning techniques have already been used in this context, e.g., for verifying process constraints~\cite{RegisRAM12}, for accomplishing business process reengineering~\cite{expertsystems07}, for conformance checking~\cite{deLeoniM17} as well as for the construction and adaptation of autonomous process models~\cite{SilvaL11,CoopIS2012}.
In~\cite{DeGiacomo2016,deLeoniLM18} automated planning techniques have been applied for aligning execution traces and process models.
%Intuitively,
%the declarative model and the current trace are encoded in terms of a planning domain and problem instance, respectively, and  
%a solution plan corresponds to the actions to be performed in order to align the model and the trace.
%, is computed.
 In~\cite{Di-Francescomarino-C.:2015aa} and~\cite{DBLP:conf/bpm/MasellisFGT17}, planning techniques have been used for addressing the problem of incomplete execution traces with respect to procedural models. 
 Compared to these two works, where the focus was on an ad hoc encoding of the problem of trace repair using a specific \emph{Action Language}, this paper tackles the more general problem of formal verification of reachability properties on imperative data-aware business processes, it introduces a more general encoding technique based on finite transition systems and provides an extensive evaluation of different solvers. The new technique highlights the core dynamic properties of Workflow Nets, and enables a seamless encoding of the decision problems under investigation in a variety of formal frameworks.
In~\cite{giordano_enriched_2018} an action language similar to \bclng (see Section~\ref{sec:bc:encoding}), and its encoding in ASP, is exploited for the verification of Business Processes enriched with a lightweight ontology language for describing semantic relations among data objects. 

%\todo[inline]{Reference to data-centric processes, e.g.\ GSM}

%\todo[inline]{Add~\cite{Di-Francescomarino-C.:2015aa,DeGiacomo:etal:ICAPS:2016}} 

%\todo{Add sidorova}
%trace repair via alignment
The problem of trace completion, which we used for the empirical investigation, has been explored in a number of works of trace alignment in the field of process mining.
Several works have addressed the problem of aligning event logs and procedural models, without~\cite{Adriansyahetal:2011} and with~\cite{deLeonietal:2012b,de_leoni:2013,Mannhardt2016} data. 
All these works, however, explore the search space of possible moves in order to find the best one aligning the log and the model.
Differently from them, in this work\begin{inparaenum}[\it(i)] \item we assume that the model is correct and we focus on the repair of incomplete execution traces; \item we want to exploit state-of-the-art formal verification techniques to solve the problem as a reachability problem on control and data flow rather than solving an optimisation problem\end{inparaenum}.

%% file: conclusions.tex
\section{Conclusions}
%\todo[inline]{completion vs alignment}
%\todo[inline]{$\C$ can be used instead of \klng (wider range of tools)}
This work provides a concrete solution for formal verification of reachability properties on imperative data-aware business processes and contributes to advancing the state-of-the-art on the concrete exploitation of formal verification techniques on business processes in two different ways: first it presents a rigorous encoding of a data-aware workflow nets based language into Action Languages, Classical Planning, and Model Checking based on a common interpretation in terms of transition systems; second it provides a first comprehensive assessment of the performance of different solvers, one for each language, in terms of reasoning with data-aware workflow net languages on both synthetic and real-life data. The evaluation shows that \nuxmv is the most robust approach as it is able to deal reasonably well with diverse types of event logs with data. On the other hand, \fastdw, outperforms its competitors on medium-size search spaces. 

In the future, we plan to investigate the impact of optimized encodings and tuning strategies to further improve the performance of the solvers.
Also, we believe that it is crucial to push forward the general research direction of this paper, namely providing real verification support for data-aware workflows models. We plan at fulfilling this objective by moving to more expressive models, such as the one in~\cite{De-Masellis:2017:aaai}, and by supporting not only reachability properties, but ideally a full-fledged temporal language.

%%Despite this work mainly focuses on the problem of trace completion, the proposed automated planning approach can easily exploit reachability for model satisfiability and trace compliance and furthermore can be easily extended also for aligning data-aware procedural models and execution traces.
%% Moreover, the presented encoding in the planning language \klng, can be directly adapted to other action languages with an expressiveness comparable to $\C$~\cite{lif99}.
 %Another restriction of this work is the encoding provided in the \klng  planning language. However, also in this case, the encoding is easily extensible to other planning languages, as for instance to the classical language $\C$.
 %% In the future, we would like to explore these extensions and implement the proposed approach and its variants in a prototype. 
  %We indeed believe that the exploitation of mature and scalable planning techniques on data-aware processes comes after, as it relies on, a clear theoretical investigation.
  %empirically evaluate the efficiency of the proposed solutions. 
%will explore these extensions in order to make the approach as much as general as possible. Moreover, we would also be interested to empirically evaluate the efficiency of the proposed approach

%% file: additional-encoding-bc.tex
%!TEX root = ./main.tex

\section{Technical details of Section~\ref{sec:bc:encoding}}\label{apx:bc:encoding}

\lemmaBCquery*
\begin{proof}
	Without loss of generality we can assume that $\Phi$ is in the DNF form described in Section~\ref{sec:bc:encoding}; i.e.\ $\bigvee_{i=1}^k t_1^i\land\ldots \land t_{\ell_i}^i$ where each term $t_j^i$ is either of the form $v = o$ or $\neg\deff(v)$. In this case to prove the lemma is sufficient to show that for the base cases $\dmodel_i, \assign_i \models t$ iff $i:\bcqlang{t}\in\dwnetbc_i(\rho)$:
	\begin{description}
		\item[($v = o$)] $\dmodel_i, \assign_i \models v = o$ iff $\assign_i(v) = o$ iff (by Definition~\ref{def:dwnet:bc:map}) $\{i:v=o\}\subseteq \dwnetbc_i^\nu(\rho)\subseteq\dwnetbc_i(\rho)$;
		\item[($\neg\deff(v)$)] $\dmodel_i, \assign_i \models \neg\deff(v)$ iff $\assign_i(v)$ is undefined iff (by Definition~\ref{def:dwnet:bc:map}) $\{i:v=\nullv\}\subseteq \dwnetbc_i^\nu(\rho)\subseteq\dwnetbc_i(\rho)$.
	\end{description}
\end{proof}

\lemmaBCencComplete*
\begin{proof}
    We prove the lemma by induction on the length of the case $\rho$ that $\bigcup_{i=0}^\ell\dwnetbc_i(\rho)$ is a stable model of $P_\ell(\bclng(\dwnmodel))$.
    For $\ell = 0$ there is only the initial state and $P_0(\bclng(\dwnmodel))$ includes only rules derived from the static laws (none), the \lstinline{initially} statements:
    \begin{align*}
        0:start=\true &\\
        0:p=\false & & & \text{for \lstinline|p|}\in P \text{ and \lstinline|p|}\neq \lstmath{start}\\
        0:v=\nullv & & & \text{for \lstinline|v|}\in \V'\\
        0: trans=\true
    \end{align*}
    and the encoding of \bclng into ASP:
    \begin{itemize}
        \item for all the places and the additional fluent for enforcing actions $p\in P\cup\{trans\}$
    \begin{align*}
        0:p=d \lor \neg (0:p=d) & & & \text{for $d\in\{\true, \false\}$} \\
        \neg 0:p=d' &\leftarrow 0:p=d & & \text{for $d,d'\in\{\true, \false\}$ and $d\neq d'$}
    \end{align*}
    \begin{align*}
        &\leftarrow 0:p=d, \neg 0:p=d & & \text{for $d\in\{\true, \false\}$}\\
        &\leftarrow \naf (0:p=\true), \naf (0:p=\false)
    \end{align*}
    \item for all the variables $v\in V'$
     \begin{align*}
        0:v=d \lor \neg (0:v=d) & & & \text{for $d\in\adom(v)\cup \{\nullv\}$} \\
        \neg 0:v=d' &\leftarrow 0:v=d & & \text{for $d,d'\in\adom(v)\cup \{\nullv\}$ and $d\neq d'$}
    \end{align*}
    \begin{align*}
        &\leftarrow 0:v=d, \neg 0:v=d & & \text{for $d\in\adom(v)\cup\{\nullv\}$}\\
        &\leftarrow \naf (0:v=d_1), \ldots, \naf (0:v=d_k),\naf (0:v=\nullv) && \text{s.t.\ $\{d_1,\ldots, d_k\} = \adom(v)$}
    \end{align*}
   \end{itemize}
  By the definition of $(\mrk_0,\assign_0)$
    \begin{align*}
        \dwnetbc_0^\nu((\mrk_0,\assign_0)) = {}
        &\{ \feqlit{0}{start}{\true}, \neg\feqlit{0}{start}{\false} \} \cup {} \\
        &\{ \feqlit{0}{p}{\false}, \neg\feqlit{0}{p}{\true}\mid p\in P\setminus\{start\}  \} \cup {} \\
        &\{ \feqlit{0}{v}{\nullv} \mid v\in\V'\} \cup  \{ \neg\feqlit{0}{v}{o} \mid v\in\V', o\in\adom(v)\} \cup {} \\
        &\{ \feqlit{0}{trans}{\true}, \neg\feqlit{0}{trans}{\false}\}\\
        \dwnetbc_0^\tau((\mrk_0,\assign_0)) = {} &\emptyset
    \end{align*}
   It's not difficult to see that all the above rules are satisfied by $\dwnetbc_0((\mrk_0,\assign_0))$. We need to show also that it's the minimal model of the reduct $P_0(\bclng(\dwnmodel))^{\dwnetbc_0((\mrk_0,\assign_0))}$ of $P_0(\bclng(\dwnmodel))$ w.r.t.\ $\dwnetbc_0((\mrk_0,\assign_0))$. To this end we note that in the reduct all the rules including NAF literals are removed because each fluent is assigned to a value, and the remaining constraints \[\leftarrow \feqlit{0}{f}{o}, \neg\feqlit{0}{f}{o}\] are satisfied by construction because $\{\feqlit{0}{f}{o}, \neg\feqlit{0}{f}{o}\}\not\subseteq \dwnetbc_0((\mrk_0,\assign_0))$.\footnote{Remember that a constraint $\leftarrow \ell_1,\ldots, \ell_n$ corresponds to the rule $f \leftarrow \naf f, \ell_1,\ldots, \ell_n$.}
   
   In $\dwnetbc_0((\mrk_0,\assign_0))$ there are just literals $\feqlit{0}{f}{c}$ and $\neg\feqlit{0}{f}{c}$. Removing any positive (i.e.\ $\feqlit{0}{f}{c}$) literal would falsify the the \lstinline{initially} statements, and removing the negative ones would contradict the corresponding rules: \[\neg\feqlit{0}{f}{o'}\leftarrow\feqlit{0}{f}{o}\] where $\feqlit{0}{f}{o}\in\dwnetbc_0((\mrk_0,\assign_0))$ and $o'\neq o$.
   This concludes that $\dwnetbc_0((\mrk_0,\assign_0))$ is a minimal model for $P_0(\bclng(\dwnmodel))^{\dwnetbc_0((\mrk_0,\assign_0))}$.
   
   For the inductive step we assume that for an arbitrary case $\rho$ of length $\ell$, $\bigcup_{i=0}^\ell\dwnetbc_i(\rho)$ is a stable model of $P_\ell(\bclng(\dwnmodel))$ and we show that $\bigcup_{i=0}^{\ell+1}\dwnetbc_i(\rho\fire{t_{\ell}}(\mrk_{\ell+1},\assign_{\ell+1}))$ is a stable model of $P_{\ell+1}(\bclng(\dwnmodel))$.
   
   Note that, by construction, \[\bigcup_{i=0}^{\ell+1}\dwnetbc_i(\rho\fire{t_{\ell}}(\mrk_{\ell+1},\assign_{\ell+1})) = \bigcup_{i=0}^{\ell}\dwnetbc_i(\rho)\cup \dwnetbc_\ell^\tau(\rho\fire{t_{\ell}}(\mrk_{\ell+1},\assign_{\ell+1}))\cup \dwnetbc_{\ell+1}^\nu(\rho\fire{t_{\ell}}(\mrk_{\ell+1},\assign_{\ell+1}))\]
      and $P_{\ell+1}(\bclng(\dwnmodel))$ is equal to $P_{\ell}(\bclng(\dwnmodel))$ plus the rules:
   \begin{align*}
\flit{\ell}{t} \lor \neg \flit{\ell}{t} & & & \text{for all $t\in T$}\\
\feqlit{\ell+1}{v}{o}&\leftarrow \feqlit{\ell}{v}{o}, \naf\neg\feqlit{\ell+1}{v}{o} & & \text{for $v\in\V'$ and $o\in \adom(v)\cup\{\nullv\}$} \\
\feqlit{\ell+1}{p}{o}&\leftarrow \feqlit{\ell}{p}{o}, \naf\neg\feqlit{\ell+1}{p}{o} & & \text{for $p\in P$ and $o\in \{\true, \false\}$}\\
\feqlit{\ell+1}{v}{d}&\leftarrow \flit{\ell}{t}, \naf\neg\feqlit{\ell+1}{v}{d} & & \text{for all $t\in T$, $(v,d)\in\writef(t)$} \\
\feqlit{\ell+1}{p}{\false}&\leftarrow \flit{\ell}{t} & & \text{for all $t\in T$, $p\in\pres{t}\setminus\posts{t}$} \\
\feqlit{\ell+1}{p}{\true}&\leftarrow \flit{\ell}{t} & & \text{for all $t\in T$, $p\in\posts{t}\setminus\pres{t}$} \\
\feqlit{\ell+1}{v}{\nullv}&\leftarrow \flit{\ell}{t} & & \text{for all $v\in\V'$, $t\in T$, s.t.\ $\writef(t)(v)=\emptyset$} \\
\feqlit{\ell+1}{trans}{\true}&\leftarrow \flit{\ell}{t} & & \text{for all $t\in T$}\\
\neg \feqlit{\ell+1}{p}{d'} &\leftarrow \feqlit{\ell+1}{p}{d} & & \text{for $p\in P\cup\{trans\}$, $d,d'\in\{\true, \false\}$ and $d\neq d'$}\\
\neg \feqlit{\ell+1}{v}{d'} &\leftarrow \feqlit{\ell+1}{v}{d} & & \text{for $v\in\V'$ $d,d'\in\adom(v)\cup \{\nullv\}$ and $d\neq d'$}
\end{align*}
the constraints
\begin{align*}
&\leftarrow \flit{\ell}{t}, \flit{\ell}{s} && \text{for $t,s\in T$, s.t.\ $t\neq s$}\\
&\leftarrow \flit{\ell}{t}, \naf \neg\feqlit{\ell+1}{v}{d} && \parbox[t]{15em}{for $t\in T$, $v\in\V'$ s.t.\ $\writef(t)(v)\neq\emptyset$, $d\in \{\nullv\}\cup\adom(v)\setminus\writef(t)(v)$}\\
&\leftarrow \flit{\ell}{t}, \feqlit{\ell}{p}{\false} && \text{for $t\in T$, $p\in\pres{t}$}\\
&\leftarrow \feqlit{\ell+1}{p}{d}, \feqlit{\ell+1}{p}{d} & & \text{for $p\in P\cup\{trans\}$, $d\in\{\true, \false\}$}\\
&\leftarrow \naf \feqlit{\ell+1}{p}{\true}, \naf \feqlit{\ell+1}{p}{\false} && \text{for $p\in P\cup\{trans\}$}\\
&\leftarrow \feqlit{\ell+1}{v}{d}, \feqlit{\ell+1}{v}{d} & & \text{for $v\in \V'$, $d\in\adom(v)\cup \{\nullv\}$}\\
&\leftarrow \naf \feqlit{\ell+1}{v}{d_1},\ldots, \naf \feqlit{\ell+1}{v}{d_k}, \naf \feqlit{\ell+1}{v}{\nullv} && \text{for $v\in \V'$, $\{d_1,\ldots, d_k\} = \adom(v)$}
\end{align*}
and the constraints corresponding to the guards; that is, for each transition $t$ s.t.\ $\guardf(t) \not\equiv true$ we consider $\ol{\guardf(t)}=\bigvee_{i=1}^k t_1^i\land\ldots \land t_{\ell_i}^i$:
\begin{align*}
&\leftarrow\flit{\ell}{t}, \flit{\ell}{\bcqlang{t_1^1}}, \ldots, \flit{\ell}{\bcqlang{t_{n_1}^1}} \\
&\ldots \\
&\leftarrow\flit{\ell}{t}, \flit{\ell}{\bcqlang{t_1^k}}, \ldots, \flit{\ell}{\bcqlang{t_{n_k}^k}}
\end{align*}
To show that $\bigcup_{i=0}^{\ell+1}\dwnetbc_i(\rho\fire{t_{\ell}}(\mrk_{\ell+1},\assign_{\ell+1}))$ is a stable model of $P_{\ell+1}(\bclng(\dwnmodel))$ we use the \emph{Splitting Sets} technique as introduced in~\cite{Lifschitz:1994:splitting-sets}; a generalisation of stratification which enables the splitting of a program in two parts on the basis of atoms in the head of the rules, and provides a way of characterising the stable models in terms of the stable modes of the two parts. More specifically, A splitting set for a program $P$ is any set of atoms $U$ such that, for every rule $r \in P$, if $head(r) \cap U\neq\emptyset$, then $atoms(r) \subseteq U$. The set of rules $r \in P$ such that $atoms(r) \subseteq U$ is the bottom of $P$ relative to $U$, denoted by $bot_U(P)$. The set $top_U(P) = P \setminus bot_U(P)$ is the top of $P$ relative to $U$.

Let's consider the set
\begin{align*}
    U = {}
    & \{ \feqlit{i}{f}{d} \mid i\leq\ell, \text{ $f$ fluent constant and $d$ constant} \} \cup {}\\
    & \{ \flit{i}{a} \mid i<\ell, \text{ $a$ action constant} \}
\end{align*}
then $U$ is a splitting set for $P_{\ell+1}(\bclng(\dwnmodel))$ and $bot_U(P_{\ell+1}(\bclng(\dwnmodel))) = P_{\ell}(\bclng(\dwnmodel))$; therefore we just need to show that $\dwnetbc_\ell^\tau(\rho\fire{t_{\ell}}(\mrk_{\ell+1},\assign_{\ell+1}))\cup \dwnetbc_{\ell+1}^\nu(\rho\fire{t_{\ell}}(\mrk_{\ell+1},\assign_{\ell+1}))$ is a stable model of $e_U(P_{\ell+1}(\bclng(\dwnmodel)) \setminus P_{\ell}(\bclng(\dwnmodel)),\bigcup_{i=0}^{\ell}\dwnetbc_i(\rho))$ where $e_U(P,X)$ is defined in~\cite{Lifschitz:1994:splitting-sets} as:
    \begin{multline*}
        e_U(P,X) = \{ r \mid \text{exists $r'\in P$ s.t.\ $body^+(r')\cap U\subseteq X, body^-(r')\cap U\cap X = \emptyset$ }\\
        head(r)=head(r'), body^+(r) = body^+(r)\setminus U, body^+(r) = body^+(r)\setminus U \}
    \end{multline*}
    that is, only rules whose bodies are not falsified by $X$, and the remaining literal in $U$ are removed.

Following Definition~\ref{def:dwnet:bc:map}:
    \begin{align*}
        \dwnetbc_{\ell+1}^\nu(\rho) = {}&\{ {\ell+1}:p=\true, \neg({\ell+1}:p=\false)\mid p\in P, \mrk_{\ell+1}(p) > 0 \} \cup {} \\
        &\{ {\ell+1}:p=\false, \neg({\ell+1}:p=\true)\mid p\in P, \mrk_{\ell+1}(p) = 0 \} \cup {} \\
        &\{ {\ell+1}:v=o, \neg({\ell+1}:v=\nullv) \mid v\in\V', \assign_{\ell+1}(v) = o \} \cup {} \\
        &\{ {\ell+1}:v=\nullv \mid v\in\V', \assign_{\ell+1}(v) \text{ is undefined}\} \cup {} \\
        &\{ \neg({\ell+1}:v=o) \mid v\in\V', o\in\adom(v), \assign_{\ell+1}(v) \neq o \text{ or } \assign_{\ell+1}(v) \text{ is undefined}\} \cup {} \\
        &\{ {\ell+1}:\lstmath{trans}=\true, \neg({\ell+1}:\lstmath{trans}=\false)\} \cup {} \\
        \dwnetbc_\ell^\tau(\rho) = {}& \{ \ell:t_{\ell}\} \cup \{ \neg(\ell:t)\mid t\in T, t\neq t_{\ell} \}
     \end{align*}
    
We can discard the actual constraints, since they are satisfied by construction of  $\dwnetbc_\ell^\tau(\rho\fire{t_{\ell}}(\mrk_{\ell+1},\assign_{\ell+1}))\cup \dwnetbc_{\ell+1}^\nu(\rho\fire{t_{\ell}}(\mrk_{\ell+1},\assign_{\ell+1}))$ because we assume that $(\mrk_{\ell},\assign_{\ell})\fire{t_{\ell}}(\mrk_{\ell+1},\assign_{\ell+1})$ is a valid firing. The constraints referring to guards are satisfied because either $\neg\flit{\ell}{t}\in\dwnetbc_\ell^\tau(\rho\fire{t_{\ell}}(\mrk_{\ell+1},\assign_{\ell+1}))$ or at least one of the $\neg\flit{\ell}{\bcqlang{t_i^j}}\in\dwnetbc_\ell^\nu(\rho\fire{t_{\ell}}(\mrk_{\ell+1},\assign_{\ell+1}))$ for each of the constraint, because $t_i^j$ holds in $\assign_{\ell}$ iff $\flit{\ell}{\bcqlang{t_1^k}}\in \dwnetbc_\ell^\nu(\rho)$ (Lemma~\ref{lemma:bc:query}).

We then focus on the rules in $e_U(P_{\ell+1}(\bclng(\dwnmodel)) \setminus P_{\ell}(\bclng(\dwnmodel)),\bigcup_{i=0}^{\ell}\dwnetbc_i(\rho))$ with terms in the head. The positive ones:
\begin{align*}
\flit{\ell}{t} \lor \neg \flit{\ell}{t} & & & \text{for all $t\in T$}\\
\feqlit{\ell+1}{p}{\false}&\leftarrow \flit{\ell}{t} & & \text{for all $t\in T$, $p\in\pres{t}\setminus\posts{t}$} \\
\feqlit{\ell+1}{p}{\true}&\leftarrow \flit{\ell}{t} & & \text{for all $t\in T$, $p\in\posts{t}\setminus\pres{t}$} \\
\feqlit{\ell+1}{v}{\nullv}&\leftarrow \flit{\ell}{t} & & \text{for all $v\in\V'$, $t\in T$, s.t.\ $\writef(t)(v)=\emptyset$} \\
\feqlit{\ell+1}{trans}{\true}&\leftarrow \flit{\ell}{t} & & \text{for all $t\in T$}\\
\neg \feqlit{\ell+1}{trans}{d'} &\leftarrow \feqlit{\ell+1}{trans}{d} & & \text{for $d,d'\in\{\true, \false\}$ and $d\neq d'$}\\
\neg \feqlit{\ell+1}{p}{d'} &\leftarrow \feqlit{\ell+1}{p}{d} & & \text{for $p\in P$, $d,d'\in\{\true, \false\}$ and $d\neq d'$}\\
\neg \feqlit{\ell+1}{v}{d'} &\leftarrow \feqlit{\ell+1}{v}{d} & & \text{for $v\in\V'$ $d,d'\in\adom(v)\cup \{\nullv\}$ and $d\neq d'$}
\end{align*}
and the ones with weak negation in the body:
\begin{align*}
\feqlit{\ell+1}{v}{o}&\leftarrow \naf\neg\feqlit{\ell+1}{v}{o} & & \text{for $v\in\V'$ and $o=\assign_{\ell}(v)$} \\
\feqlit{\ell+1}{v}{\nullv}&\leftarrow \naf\neg\feqlit{\ell+1}{v}{\nullv} & & \text{for $v\in\V'$ and $v\not\in\assign_{\ell}$} \\
\feqlit{\ell+1}{p}{\true}&\leftarrow \naf\neg\feqlit{\ell+1}{p}{\true} & & \text{for $p\in P$ and $\mrk_\ell(p) > 0$}\\
\feqlit{\ell+1}{p}{\false}&\leftarrow \naf\neg\feqlit{\ell+1}{p}{\false} & & \text{for $p\in P$ and $\mrk_\ell(p) = 0$}\\
\feqlit{\ell+1}{v}{d}&\leftarrow \flit{\ell}{t}, \naf\neg\feqlit{\ell+1}{v}{d} & & \text{for all $t\in T$, $(v,d)\in\writef(t)$}
\end{align*}
whose reduct w.r.t.\ $\dwnetbc_\ell^\tau(\rho\fire{t_{\ell}}(\mrk_{\ell+1},\assign_{\ell+1}))\cup \dwnetbc_{\ell+1}^\nu(\rho\fire{t_{\ell}}(\mrk_{\ell+1},\assign_{\ell+1}))$ are:
\begin{align*}
\feqlit{\ell+1}{v}{o}& & & \text{for $v\in\V'$ and $o=\assign_{\ell+1}(v)$} \\
\feqlit{\ell+1}{v}{\nullv}& & & \text{for $v\in\V'$ and $v\not\in\assign_{\ell+1}$} \\
\feqlit{\ell+1}{p}{\true}& & & \text{for $p\in P$ and $\mrk_{\ell+1}(p) > 0$}\\
\feqlit{\ell+1}{p}{\false}& & & \text{for $p\in P$ and $\mrk_{\ell+1}(p) = 0$}\\
\feqlit{\ell+1}{v}{d}&\leftarrow \flit{\ell}{t} & & \text{for all $t\in T$, $(v,d)\in\writef(t)$, $d=\assign_{\ell+1}(v)$}
\end{align*}
We show that $\dwnetbc_\ell^\tau(\rho\fire{t_{\ell}}(\mrk_{\ell+1},\assign_{\ell+1}))\cup \dwnetbc_{\ell+1}^\nu(\rho\fire{t_{\ell}}(\mrk_{\ell+1},\assign_{\ell+1}))$ is also a minimal model for the reduct. It's easy to see that it satisfy all the rules by construction, so we show minimality.

Removing any term from $\dwnetbc_\ell^\tau(\rho)$ would contradict one of the rules
\[\flit{\ell}{t} \lor \neg \flit{\ell}{t}\]
therefore $\flit{\ell}{t_\ell}$ must be in $\dwnetbc_\ell^\tau(\rho)$ and it's the only positive action constant term. Therefore $\feqlit{\ell+1}{trans}{\true}$ cannot be removed as well, otherwise the rule
\[\feqlit{\ell+1}{trans}{\true}\leftarrow \flit{\ell}{t_\ell}\]
would be violated as well.

Positive literals $\feqlit{\ell+1}{v}{o}$, $\feqlit{\ell+1}{p}{o}$ in $\dwnetbc_{\ell+1}^\nu(\rho\fire{t_{\ell}}(\mrk_{\ell+1},\assign_{\ell+1}))$ for $v\in\V'$, $p\in P$ cannot be removed because it'd contradict the corresponding  rule among the ones:
\begin{align*}
\feqlit{\ell+1}{v}{o}& & & \text{for $v\in\V'$ and $o=\assign_{\ell+1}(v)$} \\
\feqlit{\ell+1}{v}{\nullv}& & & \text{for $v\in\V'$ and $v\not\in\assign_{\ell+1}$} \\
\feqlit{\ell+1}{p}{\true}& & & \text{for $p\in P$ and $\mrk_{\ell+1}(p) > 0$}\\
\feqlit{\ell+1}{p}{\false}& & & \text{for $p\in P$ and $\mrk_{\ell+1}(p) = 0$}
\end{align*}
On the other hand, the negated literals must be included otherwise some of the rules
\begin{align*}
\neg \feqlit{\ell+1}{p}{d'} &\leftarrow \feqlit{\ell+1}{p}{d} & & \text{for $p\in P$, $d,d'\in\{\true, \false\}$ and $d\neq d'$}\\
\neg \feqlit{\ell+1}{v}{d'} &\leftarrow \feqlit{\ell+1}{v}{d} & & \text{for $v\in\V'$ $d,d'\in\adom(v)\cup \{\nullv\}$ and $d\neq d'$}
\end{align*}
would be unsatisfied.
\end{proof}

\lemmaDWNetBCmap*
\begin{proof}
    To prove the first statement we should note that $P_0(\bclng(\dwnmodel))$ includes the following facts derived from the \lstinline{initially} statements:
    \begin{align*}
        0:start=\true &\\
        0:p=\false & & & \text{for \lstinline|p|}\in P \text{ and \lstinline|p|}\neq \lstmath{start}\\
        0:v=\nullv & & & \text{for \lstinline|v|}\in \V'\\
        0: trans=\true
    \end{align*}
    which fixes a value for each fluent, therefore for any pair of stable models $M, M'$ of $P_0(\bclng(\dwnmodel))$ $\bcassign{0}(M) = \bcassign{0}(M') = s_0$.
    
    Regarding the second property, we first note that for all $s\in S$, $s(trans) = \true$ because  $P_\ell(\bclng(\dwnmodel))$ only the literals $\feqlit{i}{trans}{\true}$ appear in the head of any rule for all $i\leq\ell$, and the constraint
    \[\leftarrow \naf \feqlit{i}{trans}{\true}, \naf \feqlit{i}{trans}{\false}\]
    forces either $\feqlit{i}{trans}{\true}$ or $\feqlit{i}{trans}{\false}$ to be in any stable model. Therefore $\feqlit{i}{trans}{\true}$ must be in any stable model of $P_\ell(\bclng(\dwnmodel))$ for any $i\leq \ell$.
    
    Then we recall that $(s,A,s')\in \delta$ iff there is a stable model $M_{\ell+1}$ of $P_{\ell+1}(\bclng(\dwnmodel))$ for some $\ell\geq 0$, s.t.\ $\bcassign{\ell}(M_{\ell+1}) = s'$, $\bcassign{\ell}(M_{\ell+1}) = s$, and $\{a\in T \mid \flit{\ell}{a}\in M_{\ell+1} \} = A$. First $A$ cannot be empty because the only rules with $\feqlit{\ell+1}{trans}{\true}$ in the head are
    \begin{align*}
        \feqlit{\ell+1}{trans}{\true}&\leftarrow \flit{\ell}{t} & & \text{for all $t\in T$}
    \end{align*}
    therefore there must be at least a $a\in T$ s.t.\ $\flit{\ell}{t}\in M_{\ell+1}$. Moreover, the dynamic laws
      \begin{align*}
&\text{\lstinline|false after t, s|} & & \text{for (\lstmath{t}, \lstinline|s|)}\in T\times T \text{ and \lstinline|t|${}\neq{}$\lstinline|s|}
\end{align*}
correspond to the constraints
\begin{align*}
    & \leftarrow i:t, i:s & & \text{for (\lstmath{t}, \lstinline|s|)}\in T\times T \text{ and \lstinline|t|${}\neq{}$\lstinline|s|}
\end{align*}
which prevent two parallel actions.
\end{proof}

\lemmaBCencCorrect*
\begin{proof}
   We prove the lemma by induction on $\ell$. First we consider $P_0(\bclng(\dwnmodel))$, which contains the facts
\begin{align*}
    \feqlit{0}{start}{\true} &\\
    \feqlit{0}{p}{\false} & & & \text{for \lstinline|p|}\in P \text{ and \lstinline|p|}\neq \lstmath{start}\\
    \feqlit{0}{v}{\nullv} & & & \text{for \lstinline|v|}\in \V'\\
    \feqlit{0}{trans}{\true}
\end{align*}
that, together with Lemma~\ref{lemma:dwnet:bc:map}, ensures that $\bcdwnet_0(M_0) = (\mrk_0, \assign_0)$ satisfies the conditions
\begin{align*}
    \mrk_0 = & \{ (start, 1) \}\cup \{ (p, 0)\mid p\in P\setminus\{start\} \} \\
    \assign_0 = & \emptyset
\end{align*}
corresponding to the initial state $(M_s, \eta_s)$.

For the inductive step, let $M_{\ell+1}$ be a stable model for $P_{\ell+1}(\bclng(\dwnmodel))$, then we show that $\bcdwnet_\ell(M_{\ell+1}) \fire{\tau_\ell(M_{\ell+1})}\bcdwnet_{\ell+1}(M_{\ell+1})$ is a valid firing according to Def.~\ref{def:dwnet:firing}. Since Lemma~\ref{lemma:dwnet:bc:map} ensures that the mappings are well defined, we need to show that
  \begin{compactenum}
  \item $a = \tau_\ell(M_{\ell+1})$ is enabled in $\bcdwnet_\ell(M_{\ell+1})$:
   \begin{align*}
      \feqlit{\ell}{p}{\true} &\in M_{\ell+1} && \text{for $p\in\pres{a}$}
  \end{align*}
  and that's the case because if  there's $p'\in\pres{a}$ s.t.\ $\feqlit{\ell}{p}{\true} \not\in M_{\ell+1}$ then $\feqlit{\ell}{p}{\false} \in M_{\ell+1}$ and the constraint
  \begin{align*}
      &\leftarrow \flit{\ell}{a}, \feqlit{\ell}{p'}{\false}
  \end{align*}
  deriving from~\ref{eq:bcEnc:inplaces} would be contradicted.
  \item the guard of $a$ is satisfied in $\bcdwnet_\ell(M_{\ell+1})$; i.e.\ the formula $\bigvee_{i=1}^k t_1^i\land\ldots \land t_{n_i}^i\equiv_{\dmodel'}\neg\guardf(t)$ is not satisfied in $\bcdwnet_\ell(M_{\ell+1})$:\footnote{Each term $t_j^i$ is either of the form $v = o$ or $\neg\deff(v)$, and $\bcqlang{v = o} \mapsto \lstmath{v = o}$, $\bcqlang{\neg\deff(v}) \mapsto \lstmath{v = null}$.}
  \begin{align*}
      \{ \flit{\ell}{\bcqlang{t_1^1}},\ldots, \flit{\ell}{\bcqlang{t^1_{n_1}}} \} &\not\subseteq M_{\ell+1}\\
      \ldots\\
      \{ \flit{\ell}{\bcqlang{t_1^k}},\ldots, \flit{\ell}{\bcqlang{t^k_{n_k}}} \} &\not\subseteq M_{\ell+1}
  \end{align*}
  if for any $r$ $\{ \flit{\ell}{t_1^r},\ldots, \flit{\ell}{t^r_{n_r}} \} \subseteq M_{\ell+1}$ then the corresponding constraint
  \begin{align*}
      &\leftarrow\flit{\ell}{a}, \flit{\ell}{\bcqlang{t_r^1}}, \ldots, \flit{\ell}{\bcqlang{t_{n_r}^r}}
  \end{align*}
  would be contradicted.
 \item the marking of $\bcdwnet_{\ell+1}(M_{\ell+1})$ is updated according to Def.~\ref{def:Firing}:
  \begin{align*}
      \feqlit{\ell+1}{p}{\false} &\in M_{\ell+1} && \text{for $p\in \pres{a}\setminus\posts{a}$}\\
      \feqlit{\ell+1}{p}{\true} &\in M_{\ell+1} && \text{for $p\in \posts{a}\setminus\pres{a}$}\\
      \feqlit{\ell+1}{p}{d} &\in M_{\ell+1} && \text{$\feqlit{\ell}{p}{d}\in M_{\ell+1}$, for $p\in (P\setminus\posts{a}\cup\setminus\pres{a})\cup (\posts{a}\cap\pres{a})$}
  \end{align*}
  the first two conditions are ensured by the rules generated by laws~\ref{eq:bcEnc:inp_update} and \ref{eq:bcEnc:outp_update}
  \begin{align*}
    \feqlit{\ell+1}{p}{\false}&\leftarrow \flit{\ell}{a} & & \text{for all $p\in\pres{a}\setminus\posts{a}$} \\
    \feqlit{\ell+1}{p}{\true}&\leftarrow \flit{\ell}{a} & & \text{for all $p\in\posts{a}\setminus\pres{a}$} \\
  \end{align*}
  while the ``inertial'' laws~\ref{eq:bcEnc:inertiap}
  \begin{align*}
    \feqlit{\ell+1}{p}{o}&\leftarrow \feqlit{\ell}{p}{o}, \naf\neg\feqlit{\ell+1}{p}{o} & & \text{for $p\in P$ and $o\in \{\true, \false\}$}
  \end{align*}
  ensure the third condition; because those are the only kind of rules where a term like $\feqlit{\ell+1}{p}{b}$ appears in the head.
  \item the assignment $\assign'$ of $\bcdwnet_{\ell+1}(M_{\ell+1})$ is updated according to $\writef(a)$:
  \begin{align*}
      \feqlit{\ell+1}{v}{\nullv} &\in M_{\ell+1} && \text{for $v\in\V'$ s.t. $\writef(a)(v)=\emptyset$}\\
      \feqlit{\ell+1}{v}{d} &\in M_{\ell+1} && \text{for $v\in\V'$ s.t. $\writef(a)(v)\neq\emptyset$, and some $d\in\writef(a)(v)$ }\\
      \feqlit{\ell+1}{v}{d} &\in M_{\ell+1} && \text{$\feqlit{\ell}{v}{d}\in M_{\ell+1}$, for $v\in\V'\setminus dom(\writef(a))$}
  \end{align*}
  The first condition is ensured by the rule
  \begin{align*}
      \feqlit{\ell+1}{v}{\nullv}&\leftarrow \flit{\ell}{a} & & \text{for all $v\in\V'$, s.t.\ $\writef(a)(v)=\emptyset$}
  \end{align*}
  derived from law~\ref{eq:bcEnc:var:del}. The second from the constraints
  \begin{align*}
      &\leftarrow \flit{\ell}{a}, \naf \neg\feqlit{\ell+1}{v}{d} && \parbox[t]{20em}{for $v\in\V'$ s.t.\ $\writef(a)(v)\neq\emptyset$, $d\in \{\nullv\}\cup\adom(v)\setminus\writef(a)(v)$}
  \end{align*}
  derived from the law~\ref{eq:bcEnc:var:comp} since if $\neg\feqlit{\ell+1}{v}{o}\not\in M_{\ell+1}$ then $\feqlit{\ell+1}{v}{o}\in M_{\ell+1}$, because a value must be associated to any variable and this will cause the inclusion of all the ``negated'' assignments for the other values different from $o$:
  \begin{align*}
      &\leftarrow \naf \feqlit{\ell+1}{v}{d_1},\ldots, \naf \feqlit{\ell+1}{v}{d_k}, \naf \feqlit{\ell+1}{v}{\nullv} && \text{for $v\in \V'$, $\{d_1,\ldots, d_k\} = \adom(v)$}\\
      &\neg \feqlit{\ell+1}{v}{d'}\leftarrow \feqlit{\ell+1}{v}{d} & & \parbox[t]{15em}{for $v\in\V'$ $d,d'\in\adom(v)\cup \{\nullv\}$ and $d\neq d'$}
  \end{align*}
  The last property derives from the ``inertial'' constraint
  \begin{align*}
      \neg \feqlit{\ell+1}{v}{d'} &\leftarrow \feqlit{\ell+1}{v}{d} & & \text{for $v\in\V'$ $d,d'\in\adom(v)\cup \{\nullv\}$ and $d\neq d'$}
  \end{align*}
  derived from law~\ref{eq:bcEnc:var:inertia} and the fact that any other rule with terms like $\feqlit{\ell+1}{v}{o}$ in the head have $\flit{\ell}{t}$ in the body for some $t\in T$ therefore would be covered by the previous two conditions (if $t = a$) or trivially satisfied by a false body.
\end{compactenum}
\end{proof}

%% file: additional-encoding-pddl.tex
%!TEX root = ./main.tex

\section{Further technical details of Section \ref{sect:pddl:enc}}\label{apx:pddl:enc}

\lemmaPDDMenc*
\begin{proof}

\begin{enumerate}
    \item Let $(\mrk,\assign)\fire{t}(\mrk',\assign')$ be a valid firing of $\dwnmodel$, then by Definition~\ref{def:dwnet:firing}:
      \begin{compactenum}
  \item $\{ p\in P\mid M(p)>0\}\supseteq \pres{t}$
  \item for every $p\in P$:\footnote{The simplification derives from the 1-safe assumption.}
  \begin{displaymath}
	  \small
    M'(p) =
    \begin{cases}
      0 & \text{if $p\in \pres{t}\setminus\posts{t}$}\\
      1  & \text{if $p\in \posts{t}$}\\
      M(p) & \text{otherwise}
    \end{cases}
  \end{displaymath}
  \item $\dmodel, \assign \models \guardf(t)$, 
  \item assignment $\assign'$ is such that, if $\textsc{wr} = \set{v \mid \writef(t)(v)\neq\emptyset}$, $\textsc{del}  = \set{ v \mid \writef(t)(v)=\emptyset}$:
  \begin{compactitem}
  	\item its domain $dom(\assign') = dom(\assign)\cup \textsc{wr} \setminus \textsc{del}$;
        \item for each $v\in dom(\assign')$:
        \begin{displaymath}
          \assign'(v) =
          \begin{cases}
            d \in \writef(t)(v) & \text{if $v\in \textsc{wr}$}\\
            \assign(v)  & \text{otherwise.}
          \end{cases}
        \end{displaymath}
  \end{compactitem}
\end{compactenum}
Let's consider a grounding $a_t$ of $\alpha_t = (t(z_{v_1}, \ldots, z_{v_k}), \pddlpre(\alpha_t), \pddleff(\alpha_t))$ as in Equation~\eqref{eq:pddl:transition:template} where $z_{v_i} = \assign'(v_i)$ if $v_i\in \textsc{wr}$ and $z_{v_i} = \nullv$ otherwise. We show that $(\dwnetpddl(\mrk,\assign), a_t, \dwnetpddl(\mrk',\assign'))\in \gamma$ by proving that $\pddlpre(a)$ is satisfied in $\dwnetpddl(\mrk,\assign)$ and that $\gamma(\dwnetpddl(\mrk,\assign),a_t) = \dwnetpddl(\mrk',\assign')$.

    By definition
    \[\pddlpre(\alpha_t) = \pddlqlang{\guardf(t)} \land \bigwedge_{v\in\{{v_1}, \ldots, {v_k}\}} wr_{t,v}(z_v) \land \bigwedge_{p\in \pres{t}} (p = \true)\]
    $\pddlqlang{\guardf(t)}$ is satisfied because $\dmodel, \assign \models \guardf(t)$ and Lemma~\ref{lemma:pddl:query}; for all $v_i\in\{{v_1}, \ldots, {v_k}\}$ $z_{v_i}=\assign'(v_i)\in \writef(t)(v_i)$ or $z_{v_i}=\nullv$, therefore $\assign'(v_i)\in wr_{t,v_i}$ as defined in Eq.~\ref{eq:pddl:parameter:domain}; finally for all $p\in \pres{t}$ $p = \true$ because $\{ p\in P\mid M(p)>0\}\supseteq \pres{t}$.
    
    By Definition~\ref{def:pddl:domain} and given that the effects of the action are
    \[\pddleff(\alpha_t) = \bigwedge_{v\in\{{v_1}, \ldots, {v_k}\}} (v = z_v )\land\bigwedge_{p\in \pres{t}\setminus\posts{t}} (p = \false) \land \bigwedge_{p\in\posts{t}} (p = \true)\]
    the planning state $\gamma(\dwnetpddl(\mrk,\assign),a_t)$ is defined as
    \begin{align*}
        \gamma(\dwnetpddl(\mrk,\assign),a_t) = {}
            & \{ (v\mapsto\assign'(v)) \mid v \in \textsc{wr} \setminus \textsc{del} \} \cup {}\\
            & \{ (v\mapsto\nullv) \mid v \in \textsc{del} \} \cup {}\\
            & \{ (v\mapsto o) \mid v \in \V\setminus\textsc{wr}, (v\mapsto o)\in \dwnetpddl(\mrk,\assign) \} \cup {}\\
            & \{ (p\mapsto\false) \mid p\in\pres{t}\setminus\posts{t} \} \cup {}\\
            & \{ (p\mapsto\true) \mid p\in\posts{t} \} \cup {}\\
            & \{ (p\mapsto b) \mid p\in P\setminus(\pres{t}\cup\posts{t}), (p, b)\in \dwnetpddl(\mrk,\assign) \}
    \end{align*}
    therefore $\gamma(\dwnetpddl(\mrk,\assign),a_t) = \dwnetpddl(\mrk',\assign')$.

\item Let $a_t$ be a ground action of $\alpha_t = (t(z_{v_1}, \ldots, z_{v_k}), \pddlpre(\alpha_t), \pddleff(\alpha_t))$ and $(s, a_t, s')\in \gamma$, then by Equation~\eqref{eq:pddl:transition:template} and Definition~\ref{def:pddl:domain}
\begin{itemize}
\item $s$ satisfies $\pddlqlang{\guardf(t)}$;
\item for all $z_{v}$ where $v\in \textsc{wr}$, $z_{v}\in \writef(t)(v)$ if $v\not\in\textsc{del}$ and $z_{v}=\nullv$ otherwise;
\item $(p\mapsto\true)\in s$ for all $p\in \pres{t}$
\end{itemize}
because $s$ satisfies $\pddlpre(a_t)$, therefore the \ournet state $(\mrk,\assign) = \dwnetpddl(s)$ satisfies
\begin{align*}
    \{ p\in P\mid M(p)>0\}\supseteq \pres{t}\\
    \dmodel, \assign \models \guardf(t)
\end{align*}

Moreover the state $s'$ is such that
    \begin{align*}
        s' = \gamma(s,a_t) = {}
            & \{ (v\mapsto z_{v}) \mid v \in \textsc{wr} \setminus \textsc{del} \} \cup {}\\
            & \{ (v\mapsto\nullv) \mid v \in \textsc{del} \} \cup {}\\
            & \{ (v\mapsto o) \mid v \in \V\setminus\textsc{wr}, (v\mapsto o)\in s \} \cup {}\\
            & \{ (p\mapsto\false) \mid p\in\pres{t}\setminus\posts{t} \} \cup {}\\
            & \{ (p\mapsto\true) \mid p\in\posts{t} \} \cup {}\\
            & \{ (p\mapsto b) \mid p\in P\setminus(\pres{t}\cup\posts{t}), (p\mapsto b)\in s \}
    \end{align*}
so the \ournet state $(\mrk',\assign') = \dwnetpddl(s')$ satisfies
\begin{align*}
    \assign'(v) &=
    \begin{cases}
    z_{v} \in \writef(t)(v) & \text{if $v\in \textsc{wr}\setminus\textsc{del}$}\\
    \text{undefined} & \text{if $v\in \textsc{del}$}\\
    \assign(v)  & \text{otherwise}
  \end{cases}
  \\
    M'(p) &=
    \begin{cases}
      0 & \text{if $p\in \pres{t}\setminus\posts{t}$}\\
      1  & \text{if $p\in \posts{t}$}\\
      M(p) & \text{otherwise}
    \end{cases}
\end{align*}

Therefore $\pddldwnet(s)\fire{t}\pddldwnet(s')$ is a valid firing.
\end{enumerate}
\end{proof}

%% file: additional-encoding-traces.tex
%!TEX root = ./main.tex

% \section{Proofs}\label{apx:proofs}

\section{Technical details of Section \ref{sec:encoding:traces}}\label{apx:encoding:traces}

To relate cases from $W^\tau$ to the original workflow $W$ we introduce a ``projection'' function $\wkfProj_\tau$ that maps elements from cases of the enriched workflow to cases using only elements from the original workflow. To simplify the notation we will use the same name to indicate mappings from states, firings and cases.

\begin{definition}\label{def:trace:workflow:project}
  Let $W = \tuple{\dmodel, \nmodel = \tuple{P,T,F}, \writef, \guardf}$ be a \ournet, $\tau = (e_1,\ldots, e_n)$ -- where $e_i = \tuple{t_i,w_i, w_i^d}$ a trace of $W$, and $W^\tau = \tuple{\dmodel, \nmodel^\tau = \tuple{P^\tau,T^\tau,F^\tau}, \writef^\tau, \guardf^\tau}$ the corresponding trace workflow. The mapping $\wkfProj_\tau$ is defined as following:
  \begin{enumerate}
    \item let $\mrk'$ be a marking of $W^\tau$, then
    $$\wkfProj_\tau(\mrk') = \mrk'\cap P\times\mathbb{N}$$
    is a marking of $W$;
    \item let $(\mrk',\assign')$ be a state of $W^\tau$, then
    $$\wkfProj_\tau((\mrk',\assign')) = (\wkfProj_\tau(\mrk'),\assign')$$
    is a state of $W$;
    \item let $t$ be a transition in $T^\tau$, then
    $$\wkfProj_\tau(t) = \begin{cases}
    t_i & \text{for $t = t_{e_i}$}\\
    t & \text{for $t\in T$}
  \end{cases}$$
    is a transition of $W$;
  \item let $(\mrk,\assign)\fire{t}(\mrk',\assign')$ be a firing in $W^\tau$, then
  $$\wkfProj_\tau((\mrk,\assign)\fire{t}(\mrk',\assign')) = \wkfProj_\tau((\mrk,\assign))\fire{\wkfProj_\tau(t)}\wkfProj_\tau((\mrk',\assign'))$$
    is a firing in $W$;
  \item let $C = f_0,\ldots,f_k$ be a sequence of valid firings in $W^\tau$, then 
  $$\wkfProj_\tau(C) = \wkfProj_\tau(f_0),\ldots,\wkfProj_\tau(f_k)$$
  is a sequence of valid firings in $W$.
  \end{enumerate}
\end{definition}

In the following we consider a \ournet $W = \tuple{\dmodel, \nmodel = \tuple{P,T,F}, \writef, \guardf}$ and a trace $\tau = (e_1,\ldots, e_n)$ of $W$ -- where $e_i = \tuple{t_i,w_i, w_i^d}$. Let $W^\tau = \tuple{\dmodel, \nmodel^\tau = \tuple{P^\tau,T^\tau,F^\tau}, \writef^\tau, \guardf^\tau}$ be the corresponding trace workflow. To simplify the notation, in the following we will use $t_{e_0}$ as a synonymous for $start_t$ and $t_{e_{n+1}}$ as $end_t$; as if they were part of the trace.

\begin{lemma} %{lemmaValidFiringMap}
	\label{lemma:validFiringMap}
  If $C$ is a sequence of valid firings in $W^\tau$ then $\wkfProj_\tau(C)$ is a sequence of valid firings in $W$.
\end{lemma}

\begin{proof}
  Let $C = (\mrk_0,\assign_0)\fire{t_1}(\mrk_1,\assign_1)\ldots (\mrk_{k-1},\assign_{k-1})\fire{t_k}(\mrk_{k},\assign_{k})$, to show that $\wkfProj_\tau(C)$ is a case of $W$ we need to prove that 
 \begin{inparaenum}[(i)]
  \item $\wkfProj_\tau((\mrk_0,\assign_0))$ is an initial state of $W$; and 
  \item the firing $\wkfProj_\tau((\mrk_{i-1},\assign_{i-1})\fire{t_i}(\mrk_{i},\assign_{i}))$ is valid w.r.t.\ $W$ for all $1\leq i\leq n$.
\end{inparaenum}
\begin{compactenum}[i)]
\item By definition $\wkfProj_\tau((\mrk_0,\assign_0)) = (\wkfProj_\tau(\mrk_0),\assign')$ and $\wkfProj_\tau(\mrk_0)\subseteq\mrk_0$. Since the start place is in $P$, then start is the only place with a token in $\wkfProj_\tau(\mrk_0)$.
\item Let consider an arbitrary firing $f_i = (\mrk_{i-1},\assign_{i-1})\fire{t_i}(\mrk_{i},\assign_{i})$ in $C$ (valid by definition), then $\wkfProj_\tau(f_i) = (\wkfProj_\tau(\mrk_{i-1}),\assign_{i-1})\fire{\wkfProj_\tau(t_i)}(\wkfProj_\tau(\mrk_{i}),\assign_{i})$.

  Note that -- by construction -- $\guardf(t_i)$ is a restriction of $\guardf(\wkfProj_\tau(t_i))$, $\posts{\wkfProj_\tau(t_i)} = \posts{t_i}\cap P$, $\pres{\wkfProj_\tau(t_i)} = \pres{t_i}\cap P$, $dom(\writef(t_i)) = dom(\writef(\wkfProj_\tau(t_i)))$ and $\writef(t_i)(v)\subseteq \writef(\wkfProj_\tau(t_i))(v)$ ; therefore
  
    \begin{itemize}
  \item $\set{ p\in P^\tau\mid \mrk_{i-1}>0}\cap P = \{ p\in P\mid \wkfProj_\tau(\mrk_{i-1})>0\}\supseteq \pres{\wkfProj_\tau(t_i)}$ because $\set{ p\in P^\tau\mid \mrk_{i-1}>0}\supseteq \pres{t_i}$;
  \item $\dmodel, \assign \models \guardf(\wkfProj_\tau(t_i))$ because $\dmodel, \assign \models \guardf(t_i)$
  \item for all $p\in P$ $\wkfProj_\tau(\mrk_{j})(p) = \mrk_{j}(p)$, therefore:
  \begin{displaymath}
    \mrk_{i}(p) = \wkfProj_\tau(\mrk_{i})(p) =
    \begin{cases}
      \mrk_{i-1}(p)-1 = \wkfProj_\tau(\mrk_{i-1})(p) - 1 & \text{if $p\in \pres{\wkfProj_\tau(t_i)}\setminus\posts{\wkfProj_\tau(t_i)}$}\\
      \mrk_{i-1}(p)+1 = \wkfProj_\tau(\mrk_{i-1})(p) + 1 & \text{if $p\in \posts{\wkfProj_\tau(t_i)}\setminus\pres{\wkfProj_\tau(t_i)}$}\\
      \mrk_{i-1}(p) = \wkfProj_\tau(\mrk_{i-1})(p) & \text{otherwise}
    \end{cases}
  \end{displaymath}
  because $f_i$ is valid w.r.t.\ $W^\tau$;
  \item the assignment $\assign_i$ satisfies the properties that its domain is $$dom(\assign_i) = dom(\assign_{i-1})\cup\set{v\mid \writef(\wkfProj_\tau(t_i))(v)\neq\emptyset}\setminus\set{v\mid \writef(\wkfProj_\tau(t_i))(v)=\emptyset}$$ and for each $v\in dom(\assign_i)$:
  \begin{displaymath}
    \assign_i(v) =
    \begin{cases}
      d \in \writef(t_i)(v)\subseteq \writef(\wkfProj_\tau(t_i))(v) & \text{if $v\in dom(\writef(t_i)) = dom(\writef(\wkfProj_\tau(t_i)))$}\\
      \assign_{i-1}(v)  & \text{otherwise.}
    \end{cases}
  \end{displaymath}
  because $f_i$ is valid.
\end{itemize}
\end{compactenum}

\end{proof}

Before going into details, we will consider some properties of the ``trace'' workflow.

\begin{lemma}\label{lemma:trace:places}
  Let $W = \tuple{\dmodel, \nmodel = \tuple{P,T,F}, \writef, \guardf}$ be a \ournet and $\tau = (e_1,\ldots, e_n)$ -- where $e_i = \tuple{t_i,w_i, w_i^d}$ -- a trace of $W$. If $C = (\mrk_0,\assign_0)\fire{t_1}(\mrk_1,\assign_1)\ldots (\mrk_{k-1},\assign_{k-1})\fire{t_k}(\mrk_{k},\assign_{k})$ is a sequence of valid firings in $W^\tau$ then for all $0\leq i\leq k$: $$\Sigma_{p\in P^\tau\setminus P} \mrk_i(p)\leq \mrk_0(start)$$
\end{lemma}
\begin{proof}
  By induction on the length of $C$.
  \begin{itemize}
    \item For $k=1$ then the only executable transition is $start_t$, therefore $t_1 = start_t$ which -- by assumption -- has two output places and -- by construction -- $\posts{start_t}\setminus P = \set{p_{e_0}}$. Since the firing is valid, then $\mrk_1(p_{e_0}) = \mrk_0(p_{e_0}) + 1 = 1\leq \mrk_0(start)$.
    \item Let's assume that the property is true for a sequence $C$ of length $n$ and consider $C' = C (\mrk_{n},\assign_{n})\fire{t_{n+1}}(\mrk_{n+1},\assign_{n+1})$. By construction, each $p\in P^\tau\setminus P$ has a single incoming edge and $\set{t\in T^\tau\mid e_i\in \posts{t}} = \set{t_{e_i}}$ and $\set{t\in T^\tau\mid e_i\in \pres{t}} = \set{t_{e_{i+1}}}$. Therefore the only occurrence in which a $p_{e_i}\in P^\tau\setminus P$ can \emph{increase} its value is when $t_{n+1} = t_{e_i}$. Since the transition is valid, then $\mrk_{n+1}(p_{e_i}) = \mrk_{n}(p_{e_i}) + 1$ and $\mrk_{n+1}(p_{e_{i-1}}) = \mrk_{n}(p_{e_{i-1}}) - 1$; therefore $\Sigma_{p\in P^\tau\setminus P} \mrk_i(p) = \Sigma_{p\in P^\tau\setminus P} \mrk_{i-1}(p)\leq \mrk_0(start)$ -- by the inductive hypothesis.
  \end{itemize}
\end{proof}

\begin{lemma}\label{lemma:trace:places:seq}
  Let $W = \tuple{\dmodel, \nmodel = \tuple{P,T,F}, \writef, \guardf}$ be a \ournet and $\tau = (e_1,\ldots, e_n)$ -- where $e_i = \tuple{t_i,w_i, w_i^d}$ -- a trace of $W$, $C = (\mrk_0,\assign_0)\fire{t_1}(\mrk_1,\assign_1)\ldots (\mrk_{k-1},\assign_{k-1})\fire{t_k}(\mrk_{k},\assign_{k})$ a sequence of valid firings in $W^\tau$, and $t_{e_i}$ is a transition of a firing $f_m$ in $C$ with $1\leq i\leq n$, then
    \begin{inparaenum}[(i)]
    \item $t_{e_{i-1}}$ is in a transition of a firing in $C$ that precedes $f_m$, 
    \item and if $\mrk_0(start)=1$ then there is a single occurrence of $t_{e_i}$ in $C$.
    \end{inparaenum}

\end{lemma}
\begin{proof}
  The proof for the first part follows from the structure of the workflow net; because -- by construction -- each $p\in P^\tau\setminus P$ has a single incoming edge and $\set{t\in T^\tau\mid e_i\in \posts{t}} = \set{t_{e_i}}$ and $\set{t\in T^\tau\mid e_i\in \pres{t}} = \set{t_{e_{i+1}}}$.
  Since each firing must be valid -- if $f_m = (\mrk_{m-1},\assign_{m-1})\fire{t_{e_i}}(\mrk_{m},\assign_{m})$ is in $C$, then $\mrk_{m-1}(p_{e_{i-1}})\geq 1$ and this can only be true if there is a firing $f_r = (\mrk_{r-1},\assign_{r-1})\fire{t_{e_{i-1}}}(\mrk_{r},\assign_{r})$ in $C$ s.t.\ $r<m$.
  
  To prove the second part is enough to show that for each $1\leq i\leq n$, if $t_{e_i}$ appears more than once in $C$ then there must be multiple occurrences of $t_{e_{i-1}}$ as well. In fact, if this is the fact, then we can use the previous part to show that there must be multiple occurrences of $t_{e_0} = start$, and this is only possible if $\mrk_0(start)>1$. 
  
  By contradiction let's assume that there are two firings $f_m$ and $f_m'$, with $m<m'$, with the same transition $t_{e_i}$, but there is only a single occurrence of $t_{e_{i-1}}$ in a firing $f_r$. Using the previous part of this lemma we conclude that $r<m<m'$, therefore $\mrk_{m-1}(p_{e_{i-1}}) = 1$ because a token could be transferred into $p_{e_{i-1}}$ only by $t_{e_{i-1}}$, so $\mrk_{m}(p_{e_{i-1}}) = 0$. In the firings between $m$ and $m'$ there are no occurrences of $t_{e_{i-1}}$, so $\mrk_{m'-1}(p_{e_{i-1}}) = \mrk_{m}(p_{e_{i-1}}) = 0$ which is in contradiction with the assumption that $f_m'$ is a valid firing. 
\end{proof}

Now we're ready to show that the ``trace'' workflow characterises all and only the cases compliant wrt\ the given trace. We divide the proof into correctness and completeness.

\begin{lemma}[Correctness]\label{lemma:wkf:trace:encoding:crct} Let $W = \tuple{\dmodel, \nmodel = \tuple{P,T,F}, \writef, \guardf}$ be a \ournet, $\tau = (e_1,\ldots, e_n)$ -- where $e_i = \tuple{t_i,w_i, w_i^d}$ -- a trace of $W$, and $C = (\mrk_0,\assign_0)\fire{t_1}(\mrk_1,\assign_1)\ldots (\mrk_{k-1},\assign_{k-1})\fire{t_k}(\mrk_{k},\assign_{k})$ be a sequence of valid firings in $W^\tau$ s.t.\ $\mrk_0(start) = 1$; given $\ell = max(\set{i\mid t_{e_i}\text{ is in a firing of }C}\cup\set{0})$, then the trace $\tau' = (e_1,\ldots, e_\ell)$ if $\ell > 0$ or $\tau' = \emptyset$ if $\ell = 0$ is compatible with the sequence $\wkfProj_\tau(C)$ in $W$. 
\end{lemma}

\begin{proof}
  By induction on the length of $C$.
  \begin{itemize}
    \item If $C = (\mrk_0,\assign_0)\fire{t_1}(\mrk_1,\assign_1)$ then $t_1 = start_t$ because the firing is valid and the only place with a token in $\mrk_0$ is $start$; therefore $\ell = 0$ and $\tau'$ is the empty trace, and $C$ trivially satisfy the empty trace.
    \item Let $C = (\mrk_0,\assign_0)\fire{t_1}(\mrk_1,\assign_1)\ldots (\mrk_{k-1},\assign_{k-1})\fire{t_k}(\mrk_{k},\assign_{k})$ s.t.\ $\wkfProj_\tau(C)$ is compliant with $\tau'$. Let's consider $C' = C \cdot (\mrk_{k},\assign_{k})\fire{t_{k+1}}(\mrk_{k+1},\assign_{k+1})$: either $t_{k+1}\in T^\tau\setminus T$ or $t_{k+1}\in T$.
    In the first case $t_{k+1} = t_{e_\ell}$ for some $1\leq \ell\leq n$, and -- by using Lemma~\ref{lemma:trace:places:seq} -- in $C$ there are occurrences of all the $t_{e_i}$ for $1\leq i<\ell$ and it's the only occurrence of $t_{e_\ell}$. This means that $\ell = max(\set{i\mid t_{e_i}\text{ is in a firing of }C}\cup\set{0})$ and we can extend $\gamma$ to $\gamma'$ by adding the mapping from $\ell$ to $k+1$. The mapping is well defined because of the single occurrence of $t_{e_\ell}$. By definition of $t_{e_\ell}$, $(\mrk_{k},\assign_{k})\fire{t_{k+1}}(\mrk_{k+1},\assign_{k+1})$ is compliant with ${e_\ell}$ because the additional conditions in $\guardf^{\tau'}(t_{k+1})$ guarantee the proper assignments for variables that are not assigned by the transition (Def.~\ref{def:trace:workflow}). Moreover the mapping $\wkfProj_\tau$ preserve the assignments, therefore $\wkfProj_\tau(\mrk_{k},\assign_{k})\fire{t_{k+1}}(\mrk_{k+1},\assign_{k+1})$ is compliant with ${e_\ell}$ as well. By using the inductive hypnotises we can show that $C'$ is compliant as well.
    In the second case the mapping is not modified, therefore the inductive hypothesis can be used to provide evidence of the first two conditions for trace compliance of Definition~\ref{def:compliance}.
  \end{itemize}
\end{proof}

\begin{lemma}[Completeness]\label{lemma:wkf:trace:encoding:cmpl} Let $W = \tuple{\dmodel, \nmodel = \tuple{P,T,F}, \writef, \guardf}$ be a \ournet, $\tau = (e_1,\ldots, e_n)$ -- where $e_i = \tuple{t_i,w_i, w_i^d}$ -- a trace of $W$, and  $C = (\mrk_0,\assign_0)\fire{t_1}(\mrk_1,\assign_1)\ldots (\mrk_{k-1},\assign_{k-1})\fire{t_k}(\mrk_{k},\assign_{k})$ be a sequence of valid firings in $W$ such that $\tau$ is compliant with it, then there is a sequence of valid firings $C'$ in $W^\tau$ s.t.\ $\wkfProj_\tau(C') = C$.
\end{lemma}
\begin{proof}
  Since $C$ is compliant with $\tau$, then there is a mapping $\gamma$ satisfying the conditions of Definition~\ref{def:compliance}. Let $C' = (\mrk'_0,\assign_0)\fire{t'_1}(\mrk'_1,\assign_1)\ldots (\mrk'_{k-1},\assign_{k-1})\fire{t'_k}(\mrk'_{k},\assign_{k})$ a sequence of firing of $W^\tau$ defined as following:
  \begin{itemize}
    \item $\mrk'_0 = \mrk_0\cup\set{(p_{e_i},0)\mid 0\leq i\leq n}$
    \item $t'_1 = t_1$ and $\mrk'_1 = \mrk_1\cup\set{(p_{e_j},0)\mid 1\leq j\leq n} \cup\set{(p_{e_0},1)}$
    \item for each $(\mrk'_{i-1},\assign_{i-1})\fire{t'_i}(\mrk'_{i},\assign_{i})$, $2\leq i\leq n$:
    \begin{itemize}
    \item if there is $\ell$ s.t.\ $\gamma(\ell) = i$ then $t'_i = t_{e_\ell}$ and $$\mrk'_i = \mrk_i\cup\set{(p_{e_j},0)\mid 0\leq j\leq n, j\neq \ell} \cup\set{(p_{e_\ell},1)}$$
    \item otherwise $t'_i = t_i$ and $$\mrk'_i = \mrk_i\cup(\mrk'_{i-1}\cap (P^\tau\setminus P)\times\mathbb{N})$$
    \end{itemize}
  \end{itemize}
  It's not difficult to realise that by construction $\wkfProj_\tau(C') = C$. 
  
  To conclude the proof we need to show that $C'$ is a sequence of valid firing in $W^\tau$. Clearly $(\mrk'_0,\assign_0)$ is a starting state, so we need to show that all the firings are valid. The conditions involving variables -- guards and update of the assignment -- follows from the fact that the original firings are valid and the newly introduced transitions are restricted according to the trace data.
  
  Conditions on input and output places that are both in $W$ and $W^\tau$ are satisfied because of the validity of the original firing and the additional conditions for the guards of newly introduced transitions are satisfied because of $\tau$ is compliant with $C$. The newly introduced places satisfy the conditions because of the compliance wrt the trace, which guarantees that for each firing with transition $t_{e_\ell}$ there is the preceding firing with transition $t_{e_{\ell-1}}$ that put a token in the $p_{e_{\ell-1}}$ place.
\end{proof}

\thTraceWorkflow*
\begin{proof}\ 

\begin{itemize}
  \item[$\Rightarrow$] If $C$ is a case of $W^\tau$ then it must contain $t_{e_n}$ (the only transition adding a token in the sink), therefore the $\ell$ of Lemma~\ref{lemma:wkf:trace:encoding:crct} is $n$ so $\tau' = \tau$ and $\tau$ is compliant with $\wkfProj_\tau(C)$.
  \item[$\Leftarrow$] If $\tau$ is compliant with $C$ then by Lemma~\ref{lemma:wkf:trace:encoding:cmpl} there is a case $C'$ of $W^\tau$ s.t.\ $\wkfProj_\tau(C') = C$. Morevover, since the last state of $C$ is a final state, so must be the final state of $C'$ because by construction there cannot be any token in the newly introduced places after the last transition $t_{e_n}$.
\end{itemize}
 
\end{proof}

\lemmaKSafe*

\begin{proof}
  We prove the theorem by induction on the length of an arbitrary sequence of valid firings $C = (\mrk_0,\assign_0)\fire{t_1}(\mrk_1,\assign_1)\ldots (\mrk_{k-1},\assign_{k-1})\fire{t_k}(\mrk_{k},\assign_{k})$ in $W^\tau$ . Note that by construction, for any marking $M'$ of $W^\tau$ and $p\in P$, $M'(p) = \wkfProj_\tau(M')(p)$.
  \begin{itemize}
    \item For a case of length 1 the property trivially holds because by definition $\mrk_0(start) \leq k$ and for each $p\in P^\tau$ (different from $start$) $\mrk_0(start) = 0$, and since $(\mrk_0,\assign_0)\fire{t_1}(\mrk_1,\assign_1)$ is valid the only case in which the number of tokens in a place is increased is for $p\in \posts{t_1}\setminus\pres{t_1}$. For any $p$ different from $start$ this becomes $1\leq k$; while since the $start$ place -- by assumption -- doesn't have any incoming arc therefore $\mrk_1(start) = \mrk_0(start) - 1\leq k$.
    \item For the inductive step we assume that each marking $\mrk_0,\ldots \mrk_{m-1}$ is $k$-safe. By contradiction we assume that $\mrk_{m}$ is not $k$-safe; therefore there is a place $p\in P^\tau$ s.t.\ $\mrk_{m}>k$. There are two cases, either $p\in P^\tau\setminus P$ or $p\in P$. In the first case there is a contradiction because, by Lemma~\ref{lemma:trace:places}, $\Sigma_{p\in P^\tau\setminus P} \mrk_i(p)\leq \mrk_0(start) = k$. In the second case, since $\wkfProj_\tau(C)$ is a case of $W$ and $\wkfProj_\tau(\mrk_{m})(p) = \mrk_{m}(p)$, there is a contradiction with the hypothesis that $W$ is $k$-safe.
  \end{itemize}
\end{proof}